%% file: main.tex
\documentclass{article}
\pdfoutput=1
\usepackage{microtype}
\usepackage{graphicx}
\usepackage{booktabs} 

\usepackage{hyperref}
\usepackage{sty/arxiv}
\usepackage{times}
\usepackage{natbib}

\usepackage{amsmath}
\usepackage{amssymb}
\usepackage{mathtools}
\usepackage{amsthm}

\usepackage[capitalize,noabbrev]{cleveref}

\usepackage{times}
\usepackage{url}
\usepackage{amsfonts}       
\usepackage{thmtools,thm-restate}
\usepackage{nicefrac}       
\usepackage{xcolor}         
\usepackage{bm}             
\usepackage{caption}
\usepackage{subfig}
\usepackage{bbm}

\newcommand*\laplace{\mathop{}\!\Delta}
\DeclareMathOperator{\diag}{diag}
\DeclareMathOperator{\trace}{tr}

\DeclareMathOperator{\vect}{vec}
\DeclareMathOperator{\lip}{Lip}

\theoremstyle{plain}
\newtheorem{theorem}{Theorem}
\newtheorem{proposition}[theorem]{Proposition}

\theoremstyle{definition}

\theoremstyle{remark}

\title{On the Lipschitz Constant of Deep Networks and Double Descent}

\author{
  Matteo Gamba \\
  KTH \\
  Sweden \\
  \texttt{mgamba@kth.se} \\
  \And
  Hossein Azizpour \\
  KTH \\
  Sweden \\
  \texttt{azizpour@kth.se} \\
  \And
  M\r{a}rten Bj\"{o}rkman \\
  KTH \\
  Sweden \\
  \texttt{celle@kth.se} \\
}

\begin{document}
\maketitle

\begin{abstract}
\input{./sections/abstract}
\end{abstract}

\section{Introduction}
\label{sec:introduction}
\input{./sections/introduction}

\input{./sections/experimental_setup}

\section{Input-Smoothness Follows Double Descent}
\label{sec:findings}
\input{./sections/findings}

\section{Implications for Implict Regularization}
\label{sec:implications}
\input{./sections/implications}

\section{Related Work and Discussion}
\label{sec:rel_work}
\input{./sections/related_work}

\section{Conclusions}
\label{sec:conclusions}
\input{./sections/conclusions}

\subsubsection*{Broader Impact Statement}
\input{./sections/ethical_statement}

\subsubsection*{Acknowledgments}
\input{./sections/acknowledgments}

\bibliographystyle{sty/tmlr}  
\bibliography{references}  

\newpage
\appendix
\input{sections/appendix}

\end{document}

%% file: sections/abstract.tex
Existing bounds on the generalization error of deep networks assume some form of smooth or bounded dependence on the input variable, falling short of investigating the mechanisms controlling such factors in practice. In this work, we present an extensive experimental study of the empirical Lipschitz constant of deep networks undergoing double descent, and highlight non-monotonic trends strongly correlating with the test error. Building a connection between parameter-space and input-space gradients for SGD around a critical point, we isolate two important factors -- namely loss landscape curvature and distance of parameters from initialization -- respectively controlling optimization dynamics around a critical point and bounding model function complexity, even beyond the training data. Our study presents novel insights on implicit regularization via overparameterization, and effective model complexity for networks trained in practice.~\footnote{Source code available at \url{https://github.com/magamba/overparameterization}}\footnote{This is the uncompressed version of the final paper presented at BMVC 2023.}

%% file: sections/introduction.tex
A longstanding question towards understanding the remarkable generalization ability of deep networks is characterizing the hypothesis class of models \textit{trained in practice}, thus isolating properties of the networks' model function that capture generalization~\citep{hanin2019deep,neyshabur2015search}. Chiefly, a central problem is understanding the role played by overparameterization~\citep{arora18optimization,neyshabur2018role,zhang2018understanding} -- a key design choice of state of the art models -- in promoting regularization of the model function. 

Modern overparameterized networks can achieve good generalization while perfectly interpolating the training set~\citep{nakkiran2019deep}. This phenomenon is described by the \textit{double descent} curve of the test error~\citep{belkin2019reconciling,geiger2019jamming}: as model size increases, the error follows the classical bias-variance trade-off curve~\citep{geman1992biasvariance}, peaks when a model is large enough to interpolate the training data, and then decreases again as model size grows further~\citep{nakkiran2019deep,belkin2019reconciling}. 
Thus, a promising direction for capturing model complexity is studying regularity w.r.t.\ smooth interpolation of training data. 

Interestingly, many existing bounds on the generalization error of deep networks \textit{postulate} bounded dependence of the model function on the input variable, via bounded Lipschitz constant~\citep{kawaguchi2022robustness,ma2021linear,wei2019data,nagarajan2018deterministic,bartlett2017spectrally}, falling short of investigating the mechanisms controlling Lipschitz continuity and input-space smoothness in relation to interpolation.

Thus, a natural question to ask is \textit{whether bounded Lipschitzness -- a key theoretical assumption for representing well-behaved model functions for fixed-size architectures -- provides a faithful representation of the hypothesis class of networks trained in practice, when model size varies in the overparameterized setting}. Specifically, any notion of regularity of model functions capturing generalization should mirror the non-monotonic trend of the test error. 

\paragraph{Contributions} In this work, (1) we present an empirical investigation of input-space smoothness of deep networks through a lower bound on their Lipschitz constant, capturing smoothness of interpolation of the training data, as model size varies; (2) we observe non-monotonic trends for the empirical Lipschitz lower bound, showing strong correlation with double descent; (3) we provide an upper bound on the true Lipschitz constant, also mirroring double descent; (4) we establish a theoretical connection between the observed trends and parameter-space dynamics of SGD in terms of sharpness of the loss landscape; (5) we present several correlates of double descent, providing insights on the hypothesis class of networks trained in practice and their effective complexity.

%% file: sections/experimental_setup.tex
\paragraph{Experimental Setup} We study deep networks under double descent, when model size is controlled by network width.
We reproduce the double descent curves of the test error~\citep{belkin2019reconciling} by training a family of ConvNets and ResNet18s~\citep{he2015delving} on the CIFAR datasets~\citep{krizhevsky2009learning} with up to $20\%$ training labels randomly perturbed. Following \citet{nakkiran2019deep}, we control model size by increasing the number $\omega$ of learned feature maps of each convolutional stage in both model families, following the progression ${[}\omega, 2\omega, 4\omega, 8\omega {]}$, for $\omega = 1, \ldots 64$.  To isolate the role of overparameterization, we remove potential confounders from the optimization process by training all networks with crossentropy loss and SGD with momentum and fixed learning rate, without any explicit regularization (e.g.\ batch norm, weight decay. Full details in appendix~\ref{sec:appendix:setup}).

Figure~\ref{fig:findings:lipschitz} (top) shows the double descent curve for the test error for our experimental setting, with the test error showing the classic U-shaped curve for small models, and a second descent as the degree of parameterization grows further.

Hereafter, we denote with \textit{interpolation threshold} the smallest model width perfectly classifying the training data. Furthermore, we refer to the Lipschitz lower bound (introduced in section~\ref{sec:findings}) as the \textit{empirical Lipschitz constant}. We emphasize that our study focuses on the trends presented by Lipschitz smoothness on the training data, rather than on precisely estimating the true Lipschitz constant of deep networks (which is NP-hard~\citep{jordan2020exactly,virmaux2018lipschitz}).

\paragraph{Outline of the Paper} Section~\ref{sec:findings} presents our main results, connecting input-smoothness with parameter-space curvature of the loss landscape and model function. Section~\ref{sec:implications} discusses broader implications of our results. Finally, section~\ref{sec:rel_work} discusses related works.

%% file: sections/findings.tex

\begin{figure*}[t]
    \centering
    \includegraphics[width=0.22\linewidth, trim={0cm 0cm 12cm 12cm}, clip]{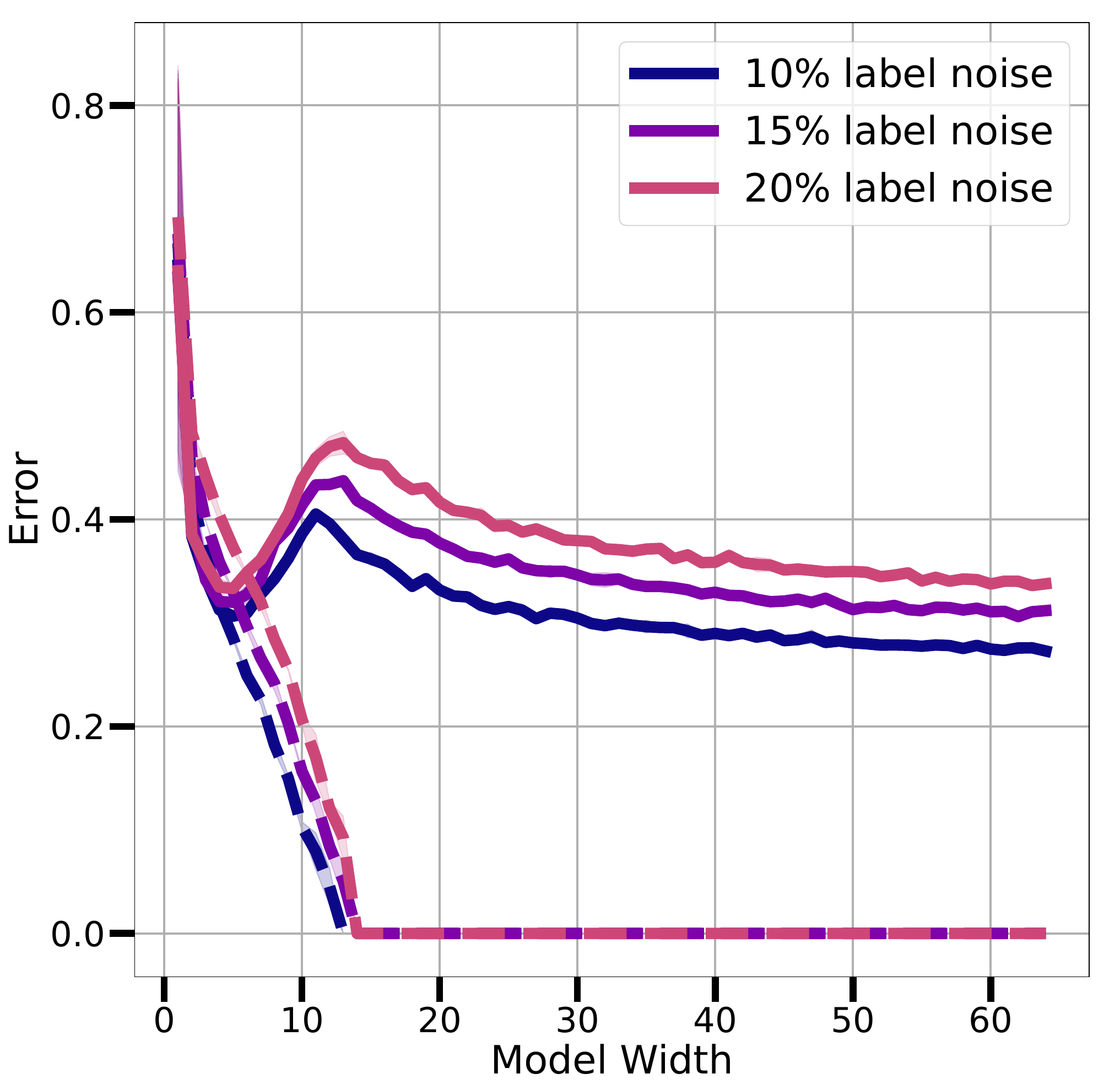}~
    \includegraphics[width=0.22\linewidth, trim={0cm 0cm 12cm 12cm}, clip]{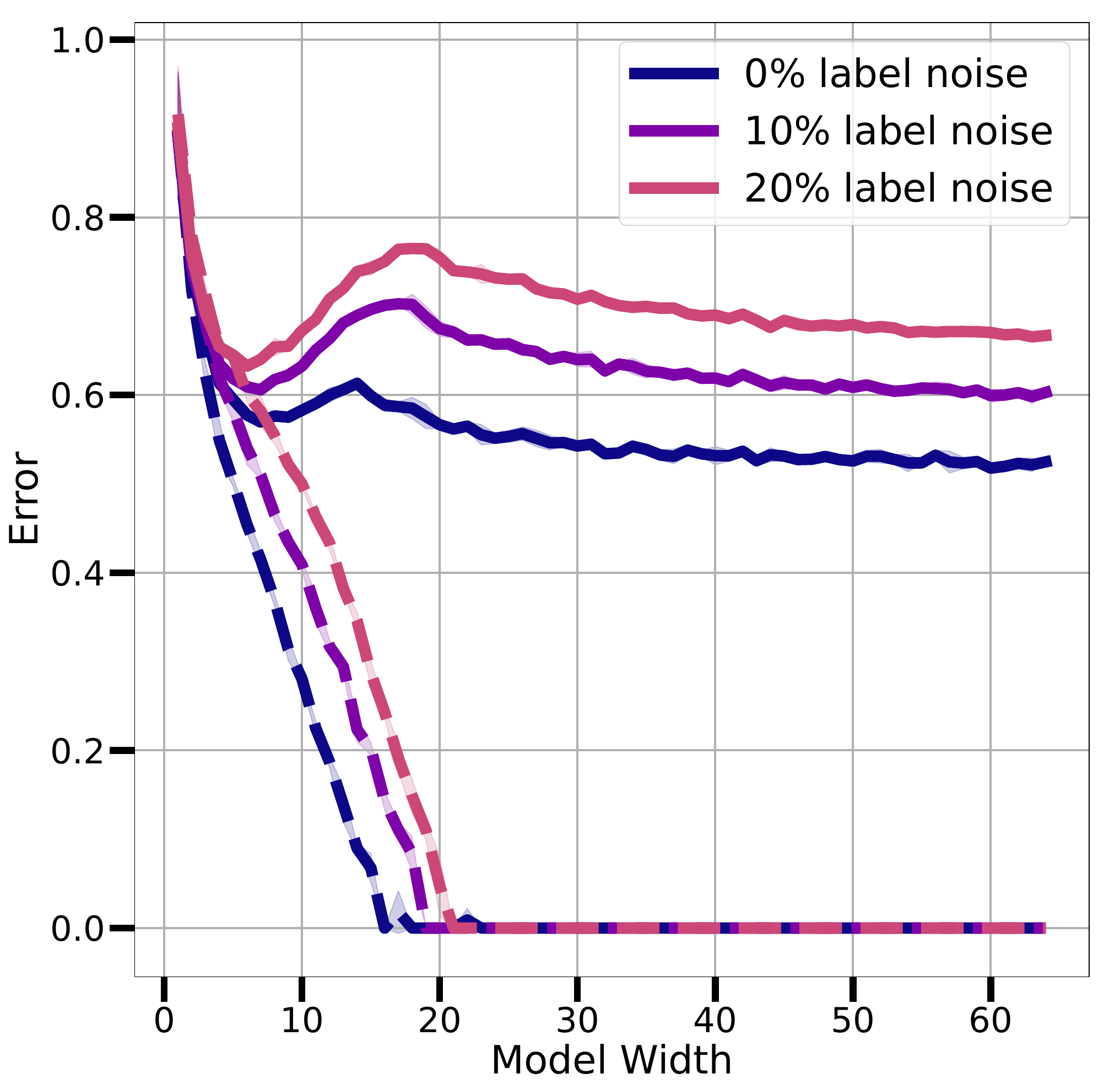}~
    \includegraphics[width=0.22\linewidth, trim={0cm 0cm 12cm 12cm}, clip]{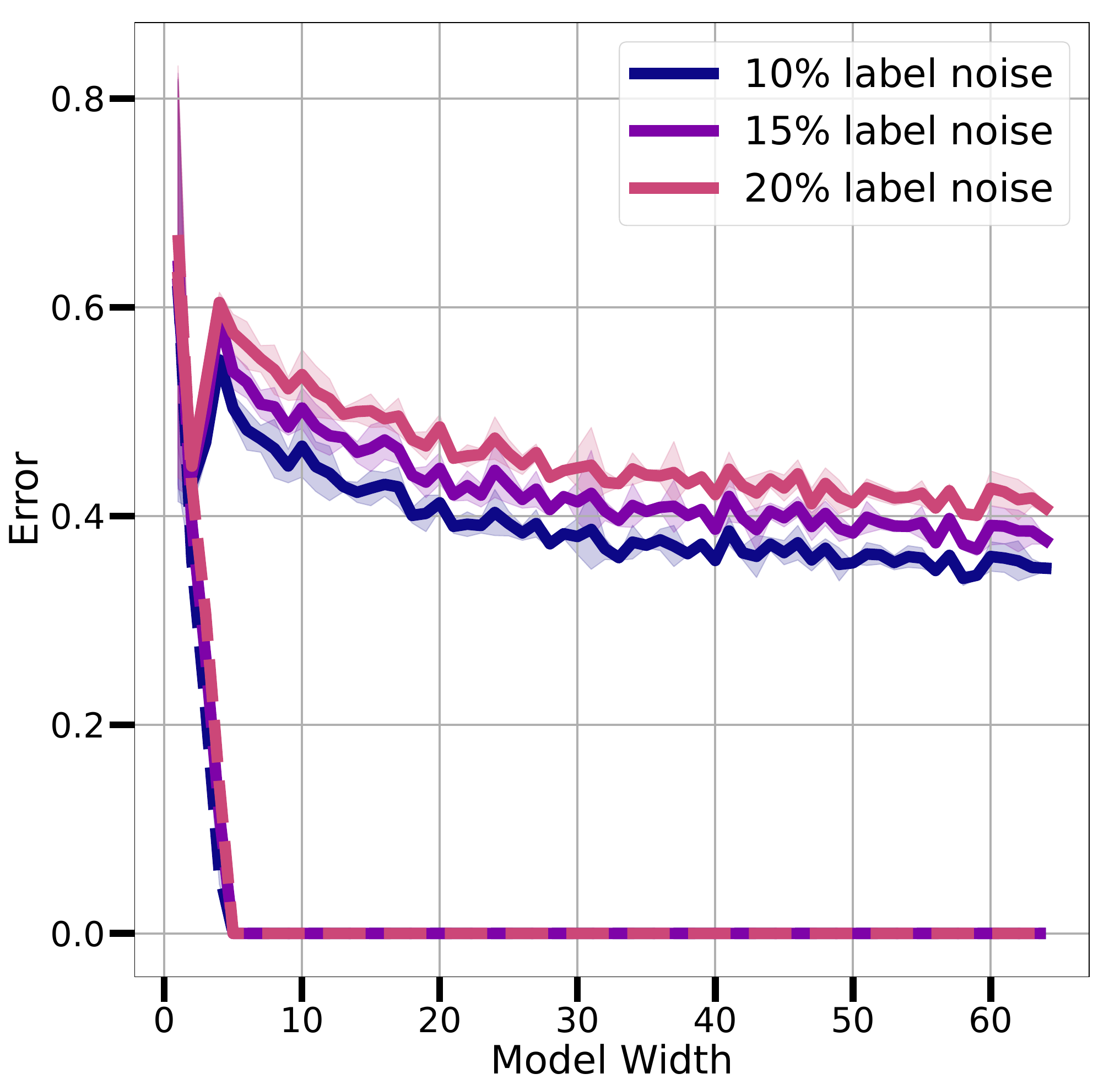}
    \includegraphics[width=0.22\linewidth, trim={0cm 0cm 12cm 12cm}, clip]{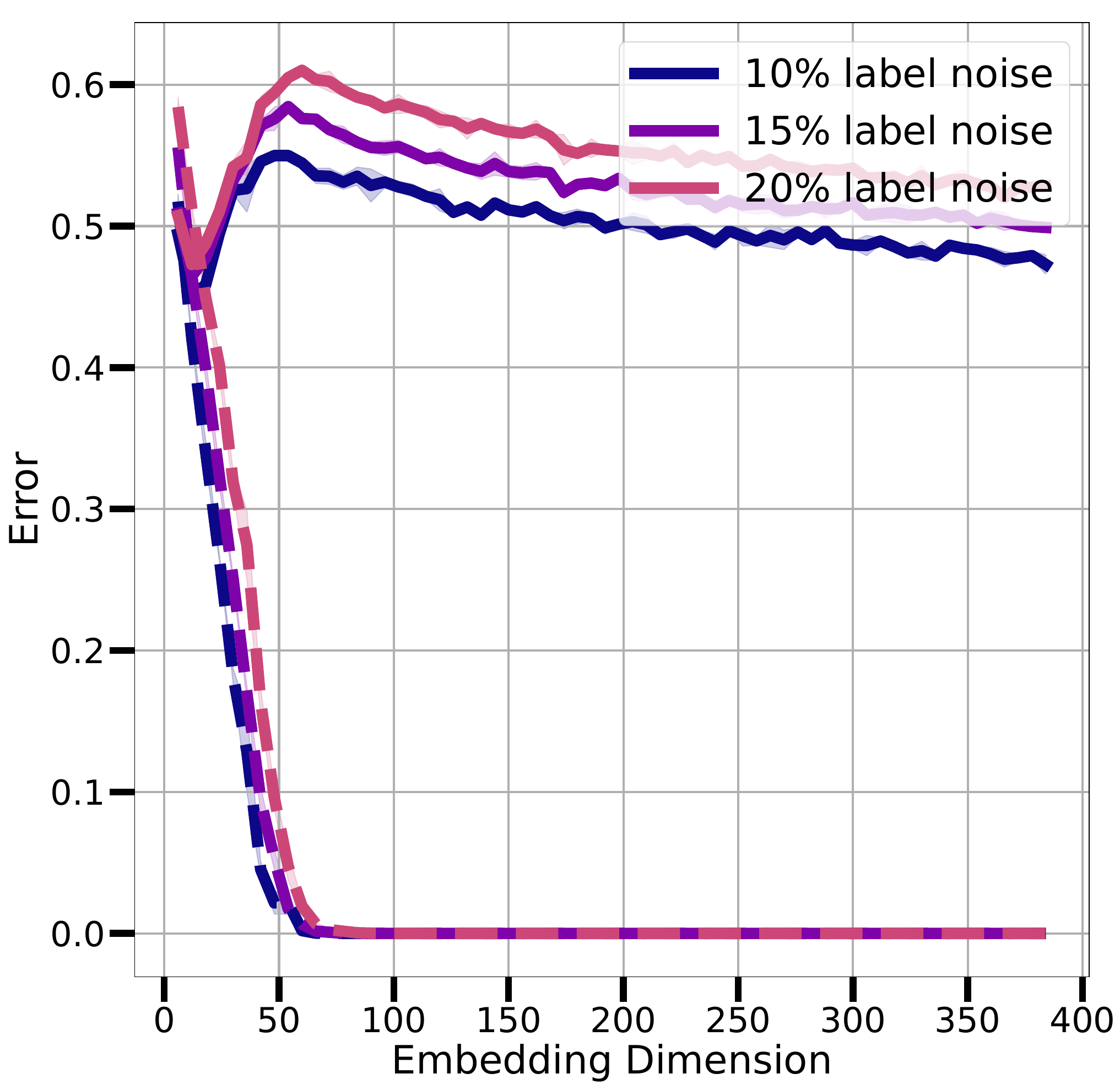}\\
    \includegraphics[width=0.22\linewidth, trim={0cm 0cm 12cm 12cm}, clip]{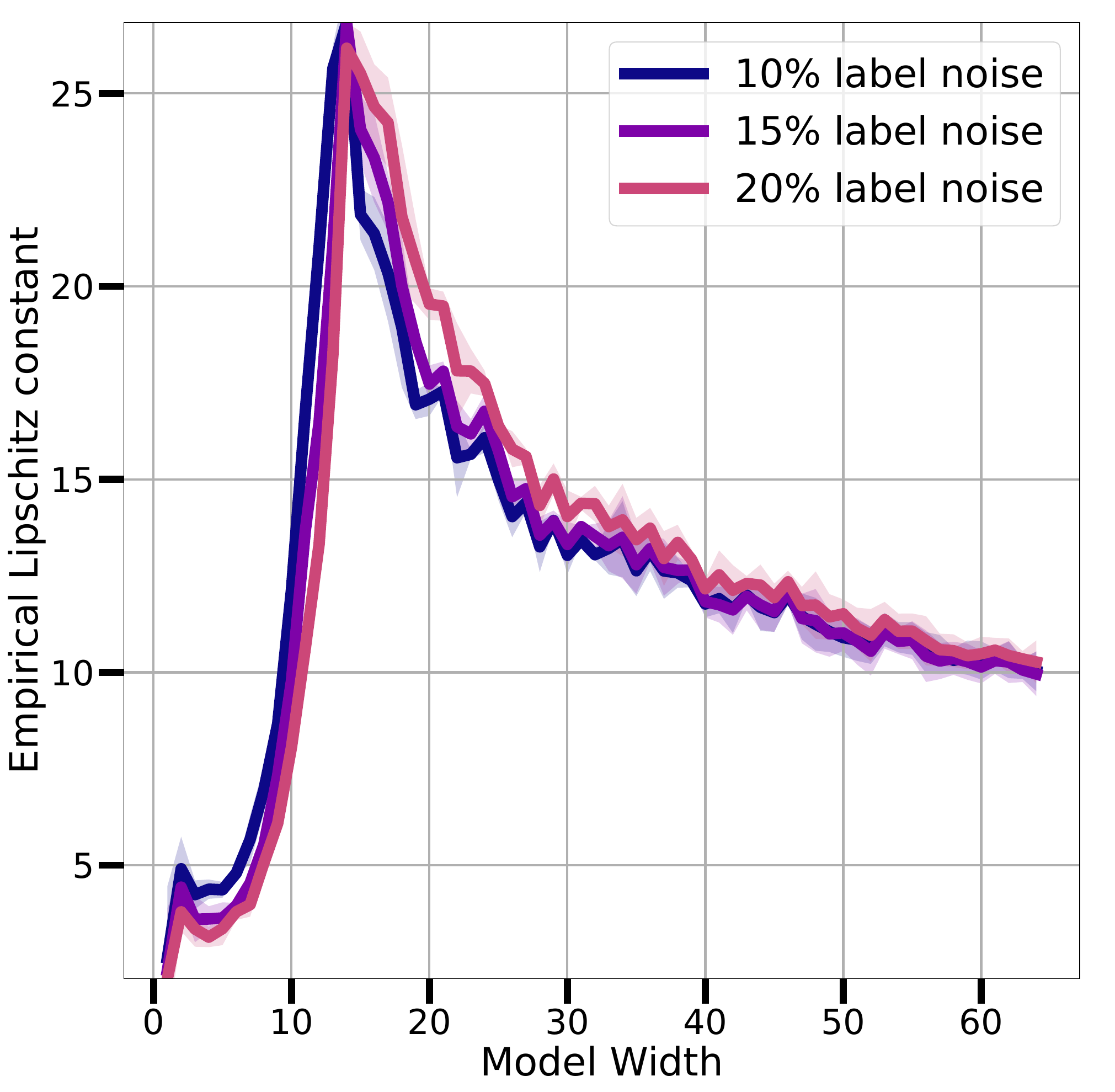}~
    \includegraphics[width=0.22\linewidth, trim={0cm 0cm 12cm 12cm}, clip]{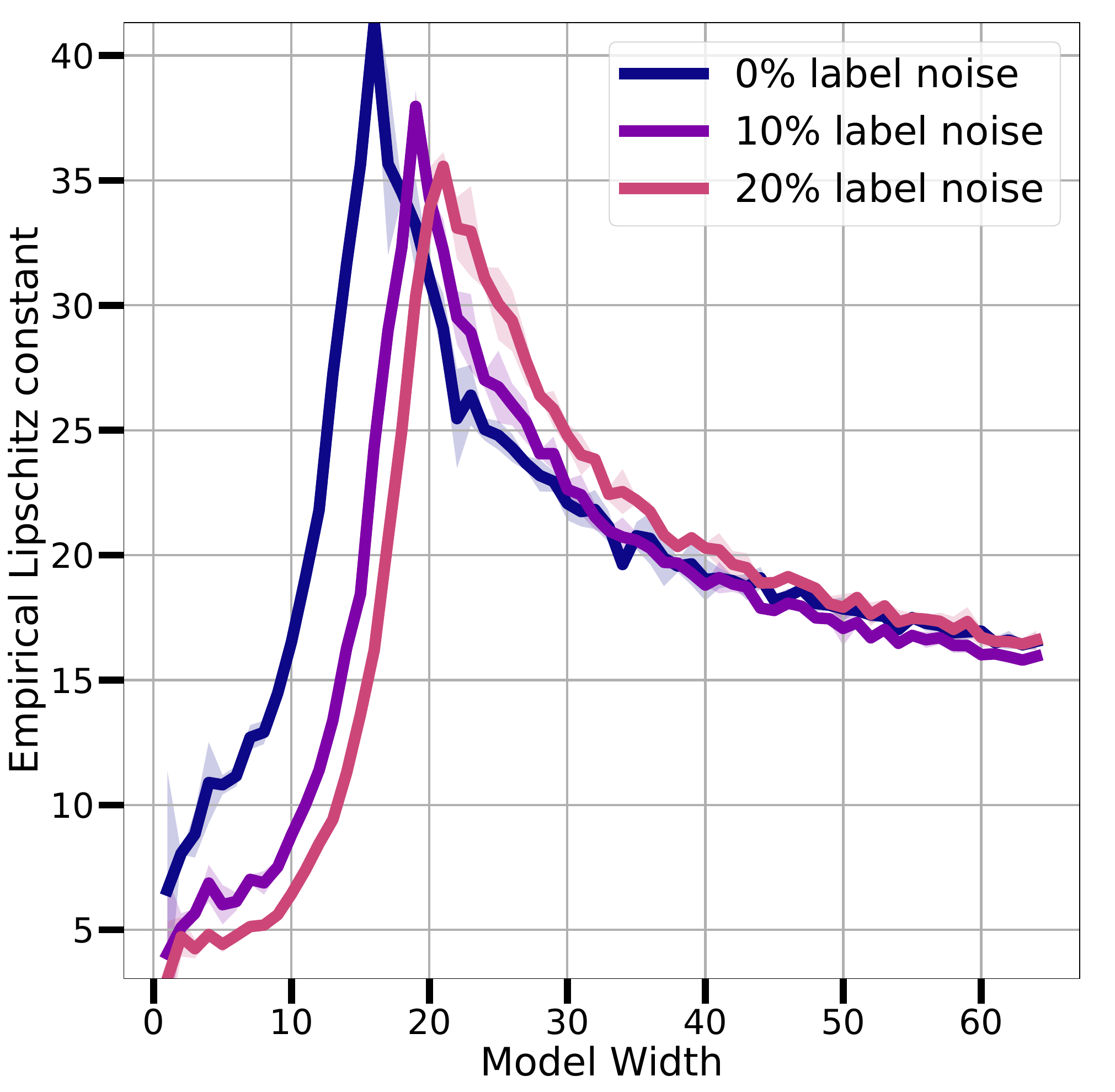}~
    \includegraphics[width=0.22\linewidth, trim={0cm 0cm 12cm 12cm}, clip]{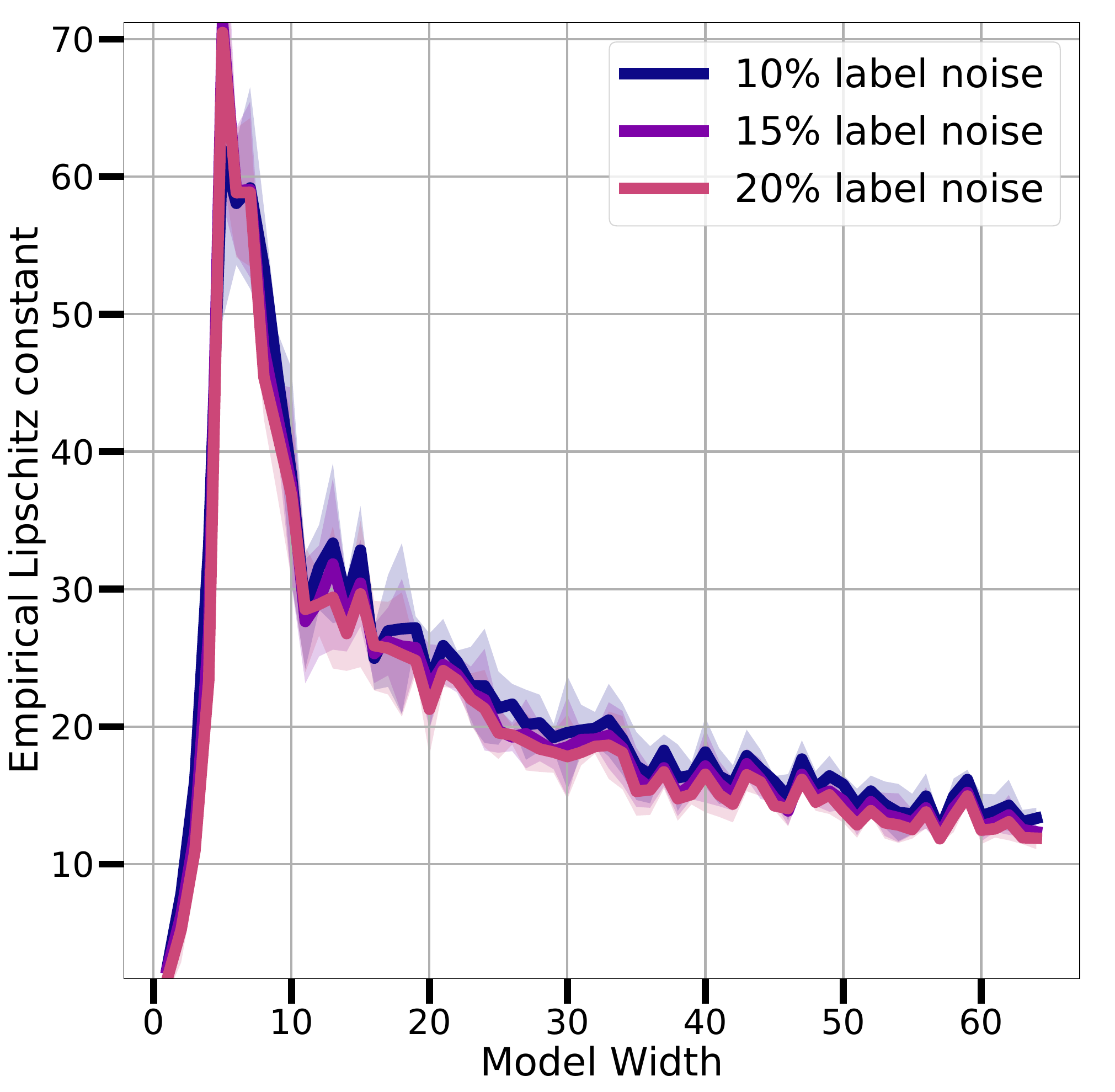}~
    \includegraphics[width=0.22\linewidth, trim={0cm 0cm 12cm 12cm}, clip]{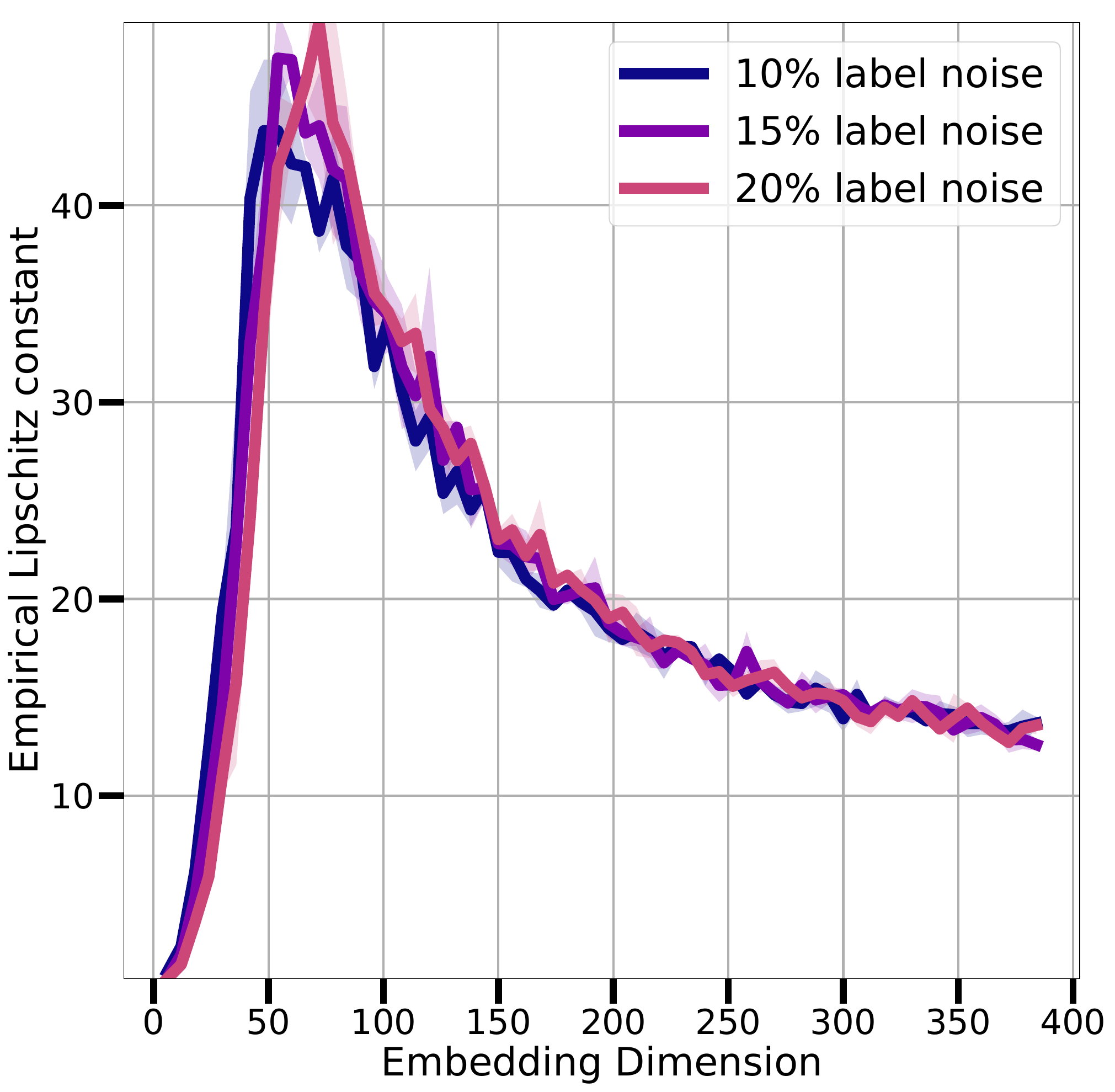}
    \caption{(Top) \textbf{Train error} (dashed) and \textbf{test error} (solid) for our experimental setting, with the test error undergoing double descent as model size increases. (Left to right) ConvNets trained on CIFAR-10 (left) and CIFAR-100 (mid-left), ResNets trained on CIFAR-10 (mid-right) and Vision Transformers on CIFAR-10 (right). (Bottom) \textbf{Empirical Lipschitz constant} for the same models. The Lipschitz lower bound depends non-monotonically on model size, strongly correlating with double descent, showing that overparameterization promotes regularization of the learned model functions via increased local Lipschitz continuity.}
    \label{fig:findings:lipschitz}
\end{figure*}

We begin by introducing the empirical Lipschitz constant for piece-wise linear networks. We consider feed-forward networks $\mathbf{f}(\mathbf{x},\bm{\theta}) : \Omega \times \mathbb{R}^p \to \mathbb{R}^K$, composing $L$ affine layers with the continuous piece-wise linear activation ReLU $\phi(x) = \max\{0, x\}$, interpreted as functions $$\mathbf{f}(\mathbf{x}, \bm{\theta}) = \bm{\theta}^L\phi(\bm{\theta}^{L-1}\phi(\cdots\phi(\bm{\theta}^1\mathbf{x} + \mathbf{b}^1)) + \mathbf{b}^{L-1}) + \mathbf{b}^L,$$ with $\bm{\theta} = (\vect(\bm{\theta}^1),\mathbf{b}^1, \ldots, \vect(\bm{\theta}^L),\mathbf{b}^L )$ representing the vectorized model parameter, and $\mathbf{x} \in \Omega \subseteq \mathbb{R}^d$ the input to the network~\footnote{typically $\Omega$ is a bounded domain, e.g.\ RGB pixels.}.

For each fixed value of $\bm{\theta}$, $\mathbf{f}_{\bm{\theta}}: \mathbb{R}^d \to \mathbb{R}^K$ corresponds to a fixed hypothesis in the space $\mathcal{H}$ of all functions expressible by the network architecture. Each model function $\mathbf{f}_{\bm{\theta}}$ is itself continuous piece-wise linear, and partitions its input space  $\Omega$ into disjoint convex polytopes $P_\epsilon$ known as activation regions~\citep{raghu2017expressive,montufar2014number}, on each of which a linear function is computed, with $\cup_\epsilon P_\epsilon = \Omega$. By piece-wise linearity, one can write \mbox{$\mathbf{f}_{\bm{\theta}}(\mathbf{x}) = \sum_\epsilon \mathbbm{1}_{P_\epsilon(\mathbf{x})} \big{[}\bm{\theta}_\epsilon \mathbf{x} + \mathbf{b}_\epsilon \big{]}$,} 
where the indicator function selects the activation region $P_\epsilon$ according to $\mathbf{x}$, and $\bm{\theta}_\epsilon$ represents conditioning the factorization $\bm{\theta}_\epsilon := \prod_{\ell=1}^L \diag(S^\ell_\mathbf{x})\bm{\theta}^\ell$ by the binary activation pattern $(S_\mathbf{x}^1, \ldots, S^L_{\mathbf{x}})$ associated with $P_\epsilon$ according to each ReLU activation, dependent on the input $\mathbf{x}$ to the network~\footnote{A similar conditioning is applied to compute the bias term $\mathbf{b}_\epsilon$.}. Formally, ${(S^\ell_\mathbf{x})}_i = \mathbbm{1}{[}\bm{\theta}^\ell_i \mathbf{x}^{\ell -1} + \mathbf{b}^\ell_i > 0{]}$, where $\bm{\theta}^\ell_i$ denotes the $i$:th row of $\bm{\theta}^\ell$, and $\mathbf{x}^{\ell -1}$ is the input to layer $\ell$.

Particularly, for any input $\overline{\mathbf{x}} \in \Omega$, evaluating the Jacobian $\nabla_\mathbf{x}\mathbf{f}_{\bm{\theta}}$ at $\overline{\mathbf{x}}$ yields $\bm{\theta}_\epsilon$, i.e.\ the linear function computed by $\mathbf{f}_{\bm{\theta}}$ on the activation region $\epsilon$ containing $\overline{\mathbf{x}}$. 
Hence, given a dataset $\mathcal{D} = \{ (\mathbf{x}_n, y_n)\}_{n=1}^N$, and denoting $\epsilon_n := \epsilon(\mathbf{x}_n)$, the empirical Lipschitz constant of $\mathbf{f}_{\bm{\theta}}$ on $\mathcal{D}$ can be estimated by computing the expected operator norm
\begin{equation}
\label{eq:findings:operator}
\left(\mathbb{E}_\mathcal{D}\|\nabla_{\mathbf{x}}\mathbf{f}_{\bm{\theta}} \|_2^2\right)^{\frac{1}{2}} := \left(\frac{1}{N} \sum\limits_{n=1}^N \sup\limits_{\mathbf{x} : \|\mathbf{x}\| \ne 0} \frac{\|\bm{\theta}_{\epsilon_n}\mathbf{x}\|^2_2}{\|\mathbf{x}\|^2_2}\right)^{\frac{1}{2}}
\end{equation}
representing the expected largest change propagated by the function on activation regions covering $\mathcal{D}$, and can be thought of as a measure of scale of $\mathbf{f}_{\bm{\theta}}$.
Appendix~\ref{sec:appendix:power_method} outlines a procedure for estimating the operator norm in practice via a power method.

\subsection{Input Smoothness of Piece-wise Linear Networks}
\label{sec:findings:pwl}

The empirical Lipschitz constant measures sensitivity of the model function around each training point. In the interpolating regime, it captures smoothness of interpolation. 

In Figure~\ref{fig:findings:lipschitz} (bottom), we compute Equation~\ref{eq:findings:operator} for deep networks trained in practice and present our main result: \textit{the empirical Lipschitz constant of deep networks is non-monotonic in model size, increasing until the interpolation threshold, and then decreasing afterward, strongly correlating with the test error. The trend is consistent across all architectures, datasets, and noise settings considered}.

Figure~\ref{fig:findings:lipschitz} (bottom right) extends the finding beyond piece-wise linear networks to Vision Transformers~\citep{dosovitskiy2021image, vaswani2017attention} trained on CIFAR-10, whereupon model size is controlled by changing the embedding dimension, as well as the width of MLP layers (see  appendix~\ref{sec:appendix:setup} for details). 

This finding sheds light on the effective complexity of \textit{trained networks} in relation to model size, complementing existing notions of Lipschitz continuity assumed in many theoretical works~\citep{kawaguchi2022robustness,ma2021linear,wei2019data,nagarajan2018deterministic,bartlett2017spectrally} which miss the observed non-monotonicity, and extending to the double descent setting the relevance of local Lipschitz continuity for generalization. The observed trends highlight a strong correlation between increased relative smoothness of $\mathbf{f}_{\bm{\theta}}$ and its generalization ability, as well as dependency of the phenomenon on model size. 


With the main message of this work established, in the following sections we draw formal connections between the empirical Lipschitz constant and parameter-space regularity (Section~\ref{sec:findings:parameters}); we discuss implications for the true Lipschitz constant (Figure~\ref{fig:findings:lipschitz_upper_bound}); finally, we present further experiments that offer broader insights on model complexity and double descent (Section~\ref{sec:implications}).

\subsection{Connection to Parameter-Space Dynamics}
\label{sec:findings:parameters}

In this section, we connect the empirical Lipschitz constant to parameter-space dynamics of SGD, by studying the relationship between input-space and parameter-space gradients of $\mathbf{f}_{\bm{\theta}}$. We defer all proofs to appendix~\ref{sec:appendix:proofs}. Let $\mathbf{x}^\ell := \phi(\bm{\theta}^\ell\mathbf{x}^{\ell-1} + \mathbf{b}^\ell)$ denote the output of the $\ell$-th layer, for $\ell = 1, \ldots, L$, with $\mathbf{x}^{0} := \mathbf{x} \in \Omega$. We begin by noting that linear layers, when composed hierarchically, enjoy a duality between their input and parameters, which ties parameter-space gradients at each layer to gradients w.r.t.\ its input. 

Formally, during backpropagation, computing the gradient $\frac{\partial \mathbf{f}(\mathbf{x}, \bm{\theta})}{\partial \bm{\theta}^\ell} = \frac{\partial\mathbf{f}(\mathbf{x}, \bm{\theta})}{\partial (\bm{\theta}^\ell\mathbf{x}^{\ell -1} + \mathbf{b}^\ell)}^T {\mathbf{x}^{\ell-1}}^T$ entails calculating the upstream gradient $\frac{\partial\mathbf{f}(\mathbf{x}, \bm{\theta})}{\partial (\bm{\theta}^\ell\mathbf{x}^{\ell -1} + \mathbf{b}^\ell)}$, which also appears in the computation of the partial derivative w.r.t.\ the $\ell$:th layer's input $\frac{\partial \mathbf{f}(\mathbf{x}, \bm{\theta})}{\partial \mathbf{x}^{\ell -1}} = \frac{\partial\mathbf{f}(\mathbf{x}, \bm{\theta})}{\partial (\bm{\theta}^\ell\mathbf{x}^{\ell -1} + \mathbf{b}^\ell)} {\bm{\theta}^\ell}$. The relationship ties the two gradients, providing the following statement.


\begin{restatable}{theorem}{sobolev}
\label{thm:findings:sobolev} Let $\mathbf{f}$ denote a neural network with a least one hidden layer, with $\|\bm{\theta}^1\| > 0$ and arbitrary weights $\bm{\theta}^2, \ldots, \bm{\theta}^L$. Let $x_{\min} := \min\limits_{\mathbf{x}_n \in \mathcal{D}}\|\mathbf{x}_n\|_2$. Then, parameter-space gradients bound input-space gradients of $\mathbf{f}$ from above:
\begin{equation}
    \label{eq:findings:sobolev}
    \begin{aligned}
        \frac{x_{\min}^2}{\|\bm{\theta}^1\|_2^2} \mathbb{E}_\mathcal{D}\|\nabla_\mathbf{x}\mathbf{f}\|^2_2 \le \mathbb{E}_\mathcal{D}\|\nabla_{\bm{\theta}}\mathbf{f} \|^2_2\,.
    \end{aligned}
\end{equation}
\end{restatable}
Crucially, the bound highlights an implicit regularization mechanism arising from hierarchical representations, whereby parameter-space gradients control input-space sensitivity by bounding the expected norm of input-space gradients $\mathbb{E}_\mathcal{D}\|\nabla_\mathbf{x}\mathbf{f}\|$, thus regularizing the empirical Lipschitz constant. We note that, while an analogous bound was first observed by~\citet{ma2021linear} (Theorem 3) for the first layer's preactivation, the authors propose a uniform bound $\mathbb{E}_\mathcal{D}\|\nabla_{\bm{\theta}}\mathbf{f}\| \le \alpha p$ that linearly increases with the number of model parameters $p$, with constant $\alpha$ depending on learning rate and batch size. In contrast, we generalize the bound to any layer beyond the first, and study it in connection to double descent, as $p$ varies with network width. Specifically, in section~\ref{sec:findings:curvature} we provide an upper bound to Theorem~\ref{thm:findings:sobolev} that captures double descent in practical settings.

Interestingly, by recalling that $\nabla_\mathbf{x} \mathbf{f}_{\bm{\theta}} = \prod_{\ell=1}^L \diag(S_\mathbf{x}^\ell)\bm{\theta}^\ell$ for $\mathbf{x} \in \Omega$, we note that the empirical Lipschitz constant is intimately tied to the model's parameters, and thus the bound in Theorem~\ref{thm:findings:sobolev} controls the expected growth of all layers. 
Additionally, by noting that the operator norm $\|\diag(S^\ell_{\mathbf{x}}) \|_2 = 1$, for $\ell = 1, \ldots, L$, the factorization $\nabla_\mathbf{x} \mathbf{f}_{\bm{\theta}} = \prod_{\ell=1}^L \diag(S_\mathbf{x}^\ell)\bm{\theta}^\ell$ allows to derive an upper bound on the \textit{true Lipschitz constant} $\lip(\mathbf{f})$ of $\mathbf{f}_{\bm{\theta}}$ on the whole domain $\Omega$.
\begin{equation}
    \label{eq:findings:lipschitz_upper_bound}
    \lip(\mathbf{f}) := \sup\limits_{\mathbf{x} \in \Omega}\|\nabla_\mathbf{x} \mathbf{f}_{\bm{\theta}}\| \le \sup\limits_{\mathbf{x} \in \Omega} \prod_{\ell=1}^L \|\diag(S_\mathbf{x}^\ell)\bm{\theta}^\ell\| \le \sup\limits_{\mathbf{x} \in \Omega} \prod_{\ell=1}^L \|\bm{\theta}^\ell\| = \prod_{\ell=1}^L \|\bm{\theta}^\ell\|_2
\end{equation}
Figure~\ref{fig:findings:lipschitz_upper_bound} presents the upper bound on the true Lipschitz constant for ConvNets trained on CIFAR-10, CIFAR-100, and ResNets trained on CIFAR-10. Similarly to the empirical Lipschitz lower bound, the upper bound closely follows double descent for the test error, peaking near the interpolation threshold. We note that, since the upper bound is independent of the binary activation pattern of ReLU, it bounds global worst-case sensitivity of the network on the whole domain $\Omega$ of $\mathbf{f}$, suggesting that the non-monotonic dependency of Lipschitz continuity on model size holds also beyond the training set $\mathcal{D}$. This observation is substantiated experimentally in section~\ref{sec:implications}.


We conclude this section by extending Theorem~\ref{thm:findings:sobolev} to exponential losses $\mathcal{L}(\bm{\theta}, \mathbf{x}, \mathbf{y})$, such as crossentropy and Mean Squared Error (MSE).

\begin{figure}[t]
    \centering
    \includegraphics[width=0.295\linewidth, trim={0cm 0cm 12cm 10cm}, clip]{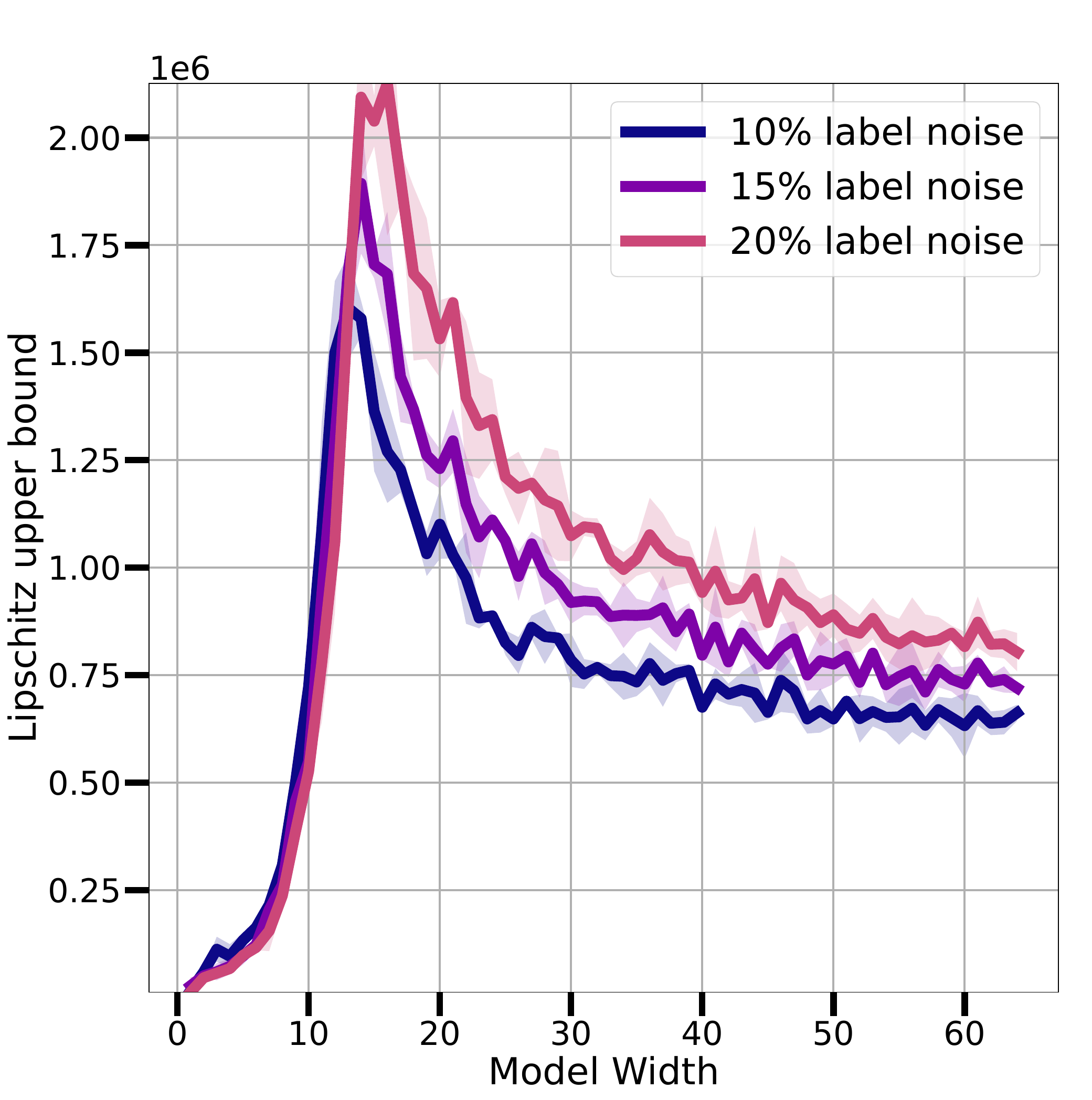}~
    \includegraphics[width=0.28\linewidth, trim={0cm 0cm 12cm 10cm}, clip]{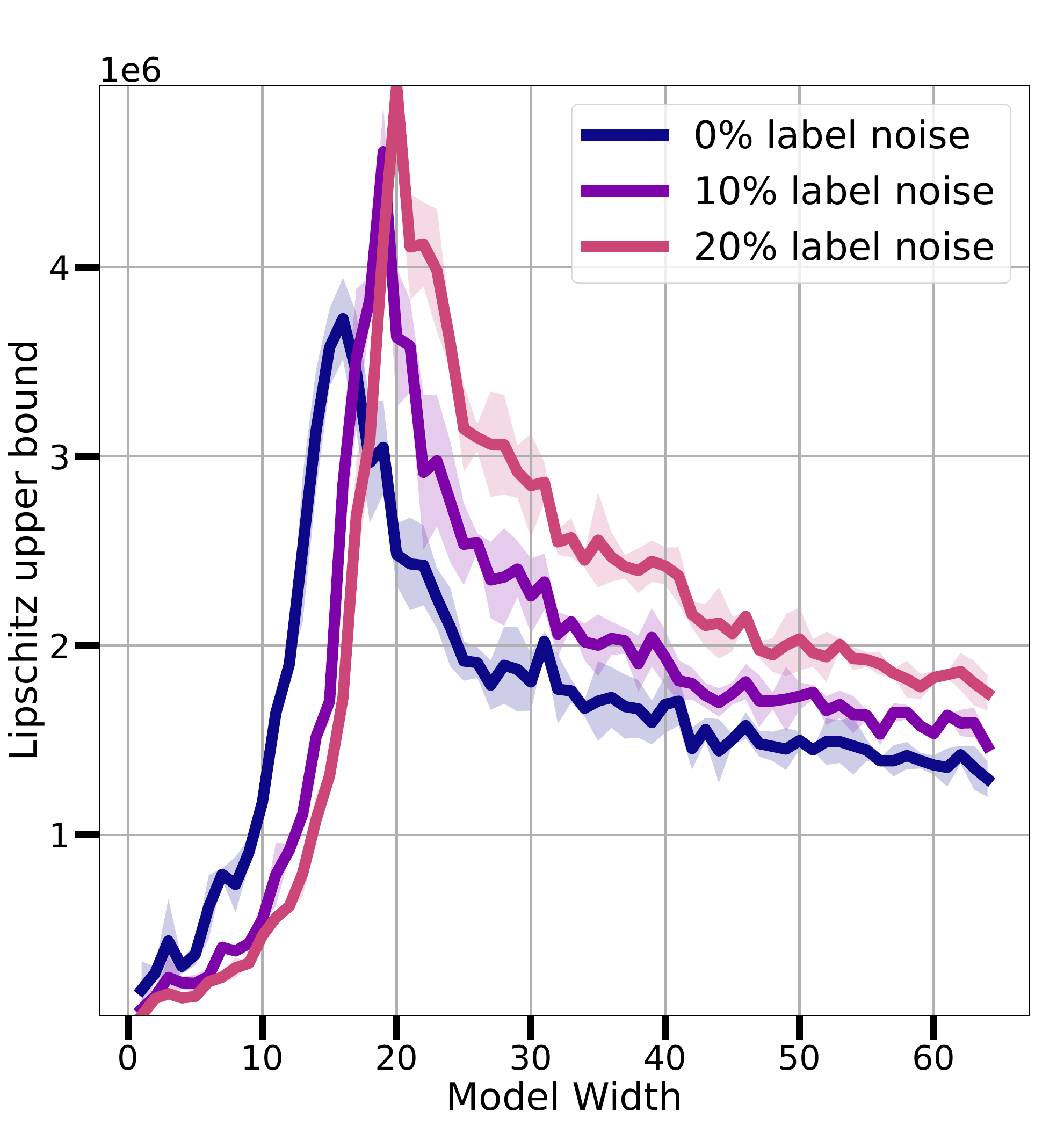}~ 
    \includegraphics[width=0.29\linewidth, trim={0cm 0cm 12cm 12cm}, clip]{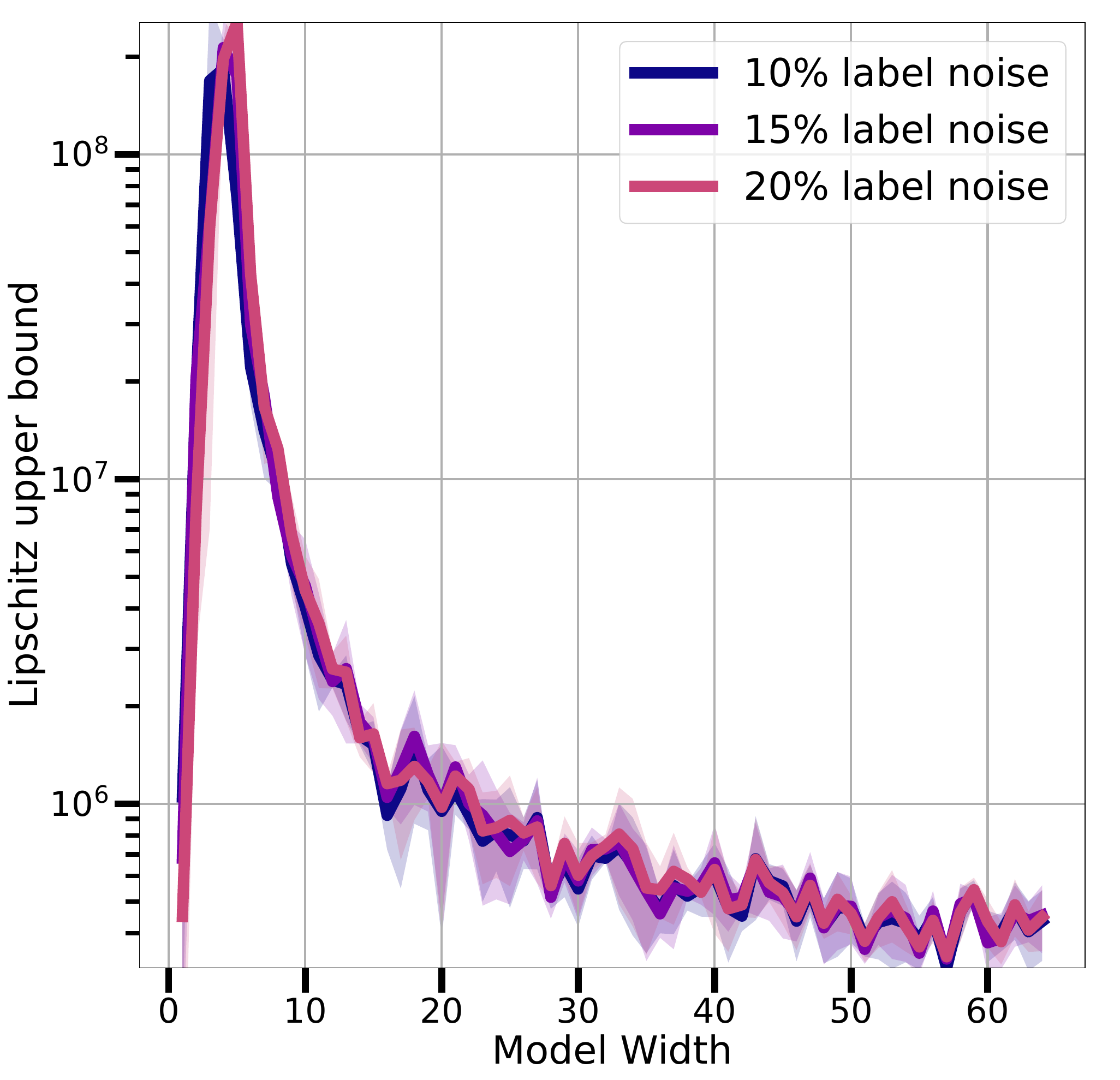}
    \caption{\textbf{Upper bound on the true Lipschitz constant}, undergoing double descent as model size increases. From left to right: ConvNets trained on CIFAR-10 (left), CIFAR-100 (middle) and ResNets trained on CIFAR-10 (right).}
    \label{fig:findings:lipschitz_upper_bound}
\end{figure}

\begin{restatable}{corollary}{sobolevloss}
\label{thm:findings:sobolev_loss} Consider the composition of a loss function $\mathcal{L}$ with a neural network $\mathbf{f}$ with a least one hidden layer, with $\|\bm{\theta}^1\| > 0$ and arbitrary weights $\bm{\theta}^2, \ldots, \bm{\theta}^L$. Then,
\begin{equation}
    \label{eq:findings:sobolev_loss}
    \begin{aligned}
        \frac{x_{\min}^2}{\|\bm{\theta}^1\|_2^2}\mathbb{E}_\mathcal{D}\|\nabla_\mathbf{x}\mathcal{L}\|^2_2 \le \mathbb{E}_\mathcal{D}\|\nabla_{\bm{\theta}}\mathcal{L} \|^2_2\,.
    \end{aligned}
\end{equation}
\end{restatable}

Importantly, the proof of Corollary~\ref{thm:findings:sobolev_loss} is based on the factorizations $\nabla_\mathbf{x} \mathcal{L}$ = $\frac{\partial \mathcal{L}}{\partial \mathbf{f}} \nabla_{\mathbf{x}} \mathbf{f}_{\bm{\theta}}$, and $\nabla_{\bm{\theta}} \mathcal{L}$ = $\frac{\partial \mathcal{L}}{\partial \mathbf{f}} \nabla_{\bm{\theta}} \mathbf{f}_{\bm{\theta}}$, with the common factor $\frac{\partial \mathcal{L}}{\partial \mathbf{f}}$ contributing a bounded monotonic rescaling of the model function gradients (visualized in Figure~\ref{fig:appendix:confidence} in the appendix). Thus, the double descent trend of the empirical Lipschitz constant can be observed also in the loss landscape, by tracking the input-space loss Jacobian $\nabla_{\mathbf{x}} \mathcal{L}$. Figure~\ref{fig:findings:hessian} presents non-monotonic trends for input-space loss Jacobian norm, as model size increases.

In the following, building on Corollary~\ref{thm:findings:sobolev_loss}, we draw an explicit connection between $\|\nabla_\mathbf{x}\mathcal{L} \|_2$ and the parameter-space geometry of the loss landscape.

\subsection{Connection to Parameter-Space Curvature}
\label{sec:findings:curvature}

\begin{figure*}[t]
    \centering
    \includegraphics[width=0.25\linewidth, trim={0cm 0cm 17cm 12cm}, clip]{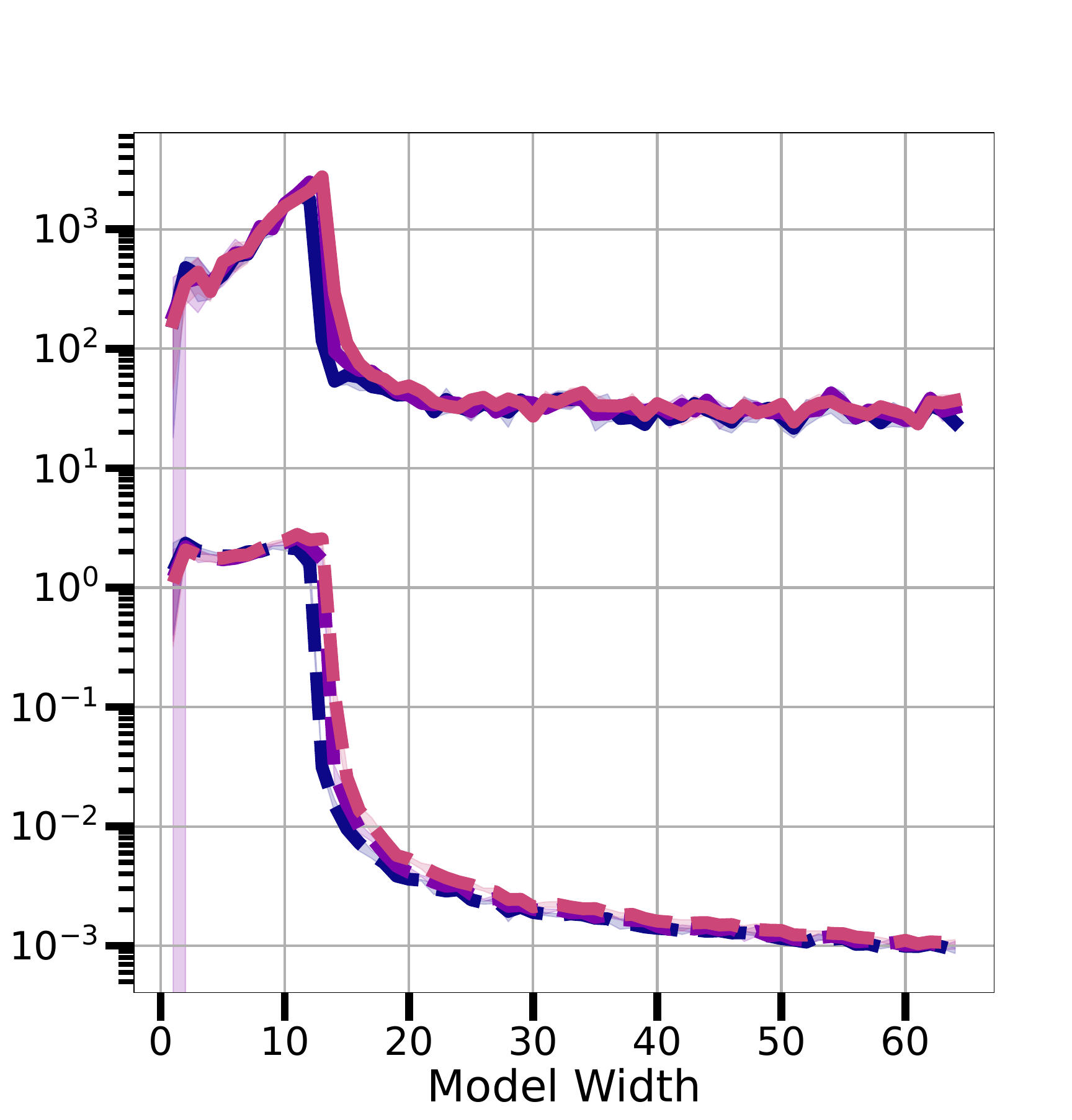}~
    \includegraphics[width=0.25\linewidth, trim={0cm 0cm 17cm 12cm}, clip]{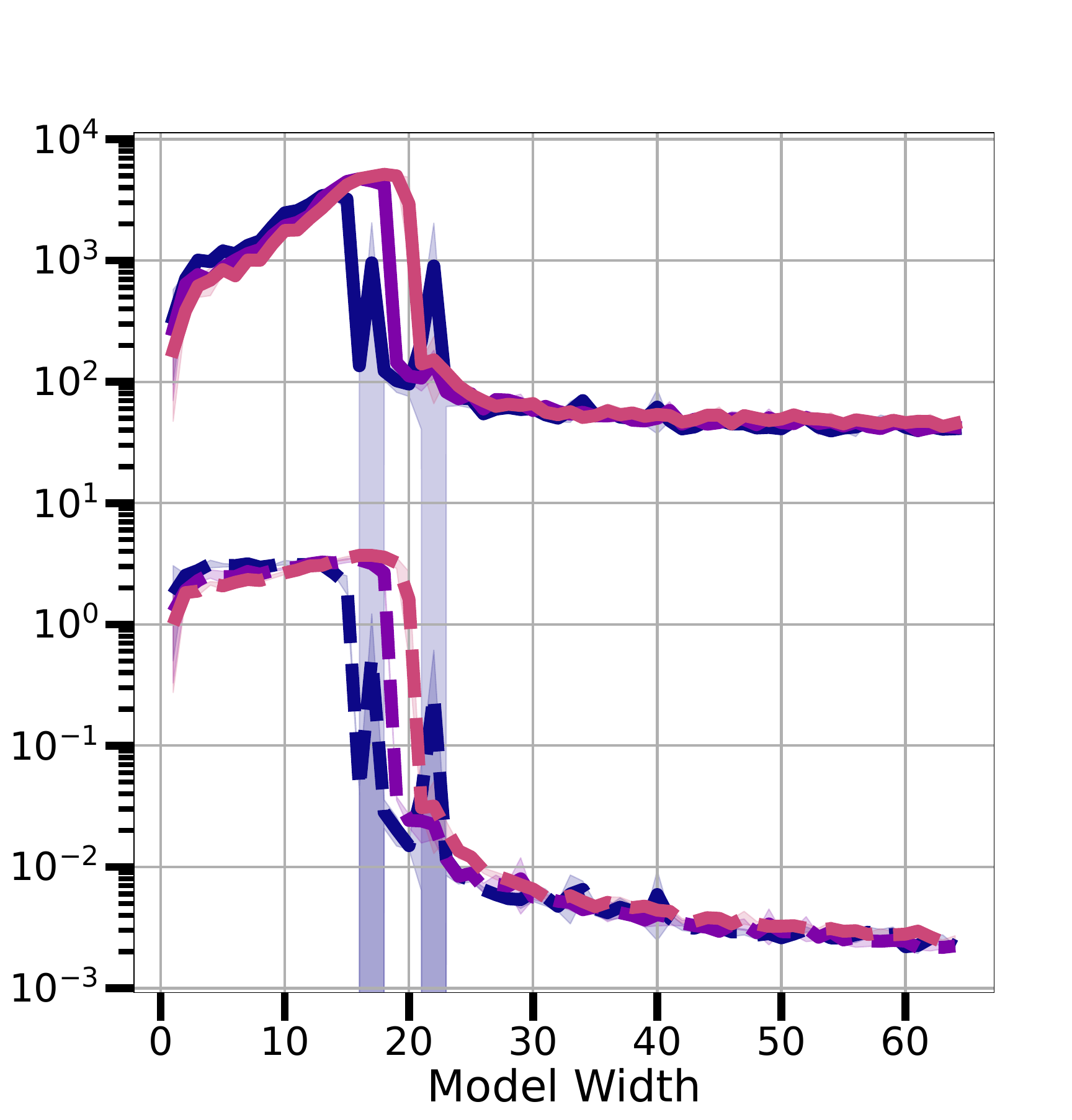}~
    \includegraphics[width=0.39\linewidth, trim={0cm 0cm 0cm 12cm}, clip]{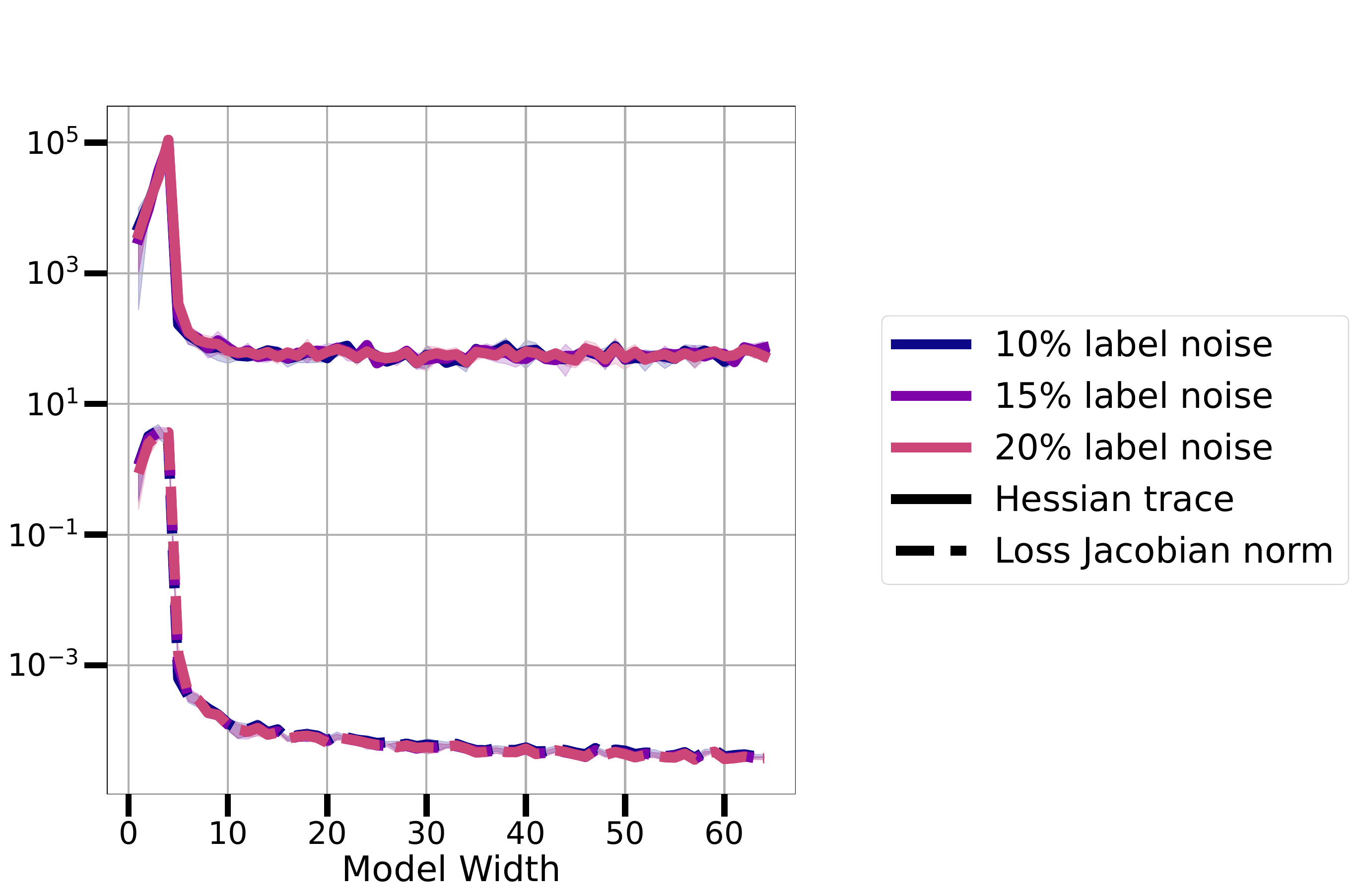} \\
    \caption{(Top) \textbf{Loss Jacobian norm} $\|\nabla_{\mathbf{x}}\mathcal{L} \|$ in \textit{input space} and \textbf{Mean loss curvature} (Hessian trace) in \textit{parameter space}. (From left to right) ConvNets trained on CIFAR-10 (left), CIFAR-100 (middle) and ResNets trained on CIFAR-10 (right). In all settings, mean parameter-space curvature strongly correlate with double descent, peaking at the interpolation threshold, and highlighting a nonlinear dependence on model size.} 
    \label{fig:findings:hessian}
\end{figure*}

To study the implications of Theorem~\ref{thm:findings:sobolev} and Equation~\ref{eq:findings:sobolev_loss}, we consider the dynamics of SGD in proximity of a minimum $\bm{\theta}^* \in \mathbb{R}^p$ of the loss $\mathcal{L}$. We adopt a linear stability perspective~\citep{hosoe2022second,wu2018sgd}, and approximate the loss in a neighbourhood of $\bm{\theta}^*$ via a second-order Taylor expansion in $\bm{\theta}$
\begin{equation}
    \label{eq:findings:taylor}
    \mathbb{E}_\mathcal{D}\mathcal{L}(\bm{\theta},\mathbf{x},y) = \frac{1}{2}(\bm{\theta} - \bm{\theta}^*)^TH(\bm{\theta} - \bm{\theta}^*) + o(\|\bm{\theta} - \bm{\theta}^*\|^2)
\end{equation}
where the first order term vanishes at the critical point $\bm{\theta}^*$, as does the zeroth order term for interpolating models, and $H$ represents the expected Hessian of the training loss.

In the next result, we derive upper bounds on \textit{input-space} smoothness of the loss (Equation~\ref{eq:findings:sobolev_loss}), in connection to \textit{parameter-space} geometry, focusing on the mean squared error \mbox{$\mathcal{L} = \frac{1}{2N}\sum\limits_{n=1}^N (f_{\bm{\theta}}(\mathbf{x}_n) - y_n)^2$}. In the following, let $\mathcal{L}(\bm{\theta}) := \mathbb{E}_\mathcal{D}{[}\mathcal{L}(\bm{\theta}, x, y){]}$.

\begin{restatable}{theorem}{curvature}
\label{thm:findings:curvature}
Let $\bm{\theta}^*$ be a critical point for the loss $\mathcal{L}(\bm{\theta}, \mathbf{x}, y)$ on $\mathcal{D}$. Let $\mathbf{f}_{\bm{\theta}}$ denote a neural network with at least one hidden layer, with $\|\bm{\theta}^1\| > 0$. Then,
\begin{equation}
    \label{eq:findings:curvature}
    \frac{x_{\min}^2}{\|\bm{\theta}^1\|_2^2}\mathbb{E}_\mathcal{D}\|\nabla_\mathbf{x}\mathcal{L}\|_2^2 \le 2\mathcal{L}_{\max}(\bm{\theta})\laplace{(\mathcal{L}(\bm{\theta}))} + o(\mathcal{L}(\bm{\theta}))
\end{equation}
with $\laplace{(\mathcal{L}(\bm{\theta}))} := \trace{(H)}$ denoting the Laplace operator, $H := \mathbb{E}_\mathcal{D}{[}\frac{\partial^2\mathcal{L}}{\partial\bm{\theta}\partial\bm{\theta}^T}{]}$ denoting the expected parameter-space Hessian of $\mathcal{L}$, and $$\mathcal{L}_{\max}(\bm{\theta}) := \max\limits_{(\mathbf{x}_n,y_n) \in \mathcal{D}}\mathcal{L}(\bm{\theta},\mathbf{x}_n,y_n)$$
\end{restatable}

Theorem~\ref{thm:findings:curvature} links input-space smoothness of the loss to the geometry of the loss landscape in parameter space, via mean curvature $\laplace{(\mathcal{L}(\bm{\theta}))}$ in a neighbourhood of $\bm{\theta}^*$.

Figure~\ref{fig:findings:hessian} shows mean curvature of the loss landscape (solid line) in parameter space for our experimental setup (see appendix~\ref{sec:appendix:power_method} for algorithmic details), as well as input-space smoothness of the loss (dashed line). Mean curvature mirrors the \textit{input-space} loss Jacobian as model width increases, peaking near the interpolation threshold and decreasing afterward. This substantiates the bound in Equation~\ref{eq:findings:curvature}, and provides a characterization of the empirical Lipschitz constant in the loss landscape in terms of fundamental quantities in parameter space (Hessian trace). Figure~\ref{fig:appendix:hessian} complements our observations, by tracking the largest and smallest non-zero Hessian eigenvalues in parameter space, and showing that both quantities track double descent.

\paragraph{Connection to Stochastic Noise}~~~SGD is known to fluctuate around critical points due to stochastic noise arising from the discretization of the dynamic (finite learning rate)~\citep{mori2022power}, as well as the use of mini-batches to estimate model gradients~\citep{ziyin2022strength,thomas2020interplay}. At iteration $t$, an estimate of the noise $\bm{\epsilon}_t$ is given by $\bm{\epsilon}_t = \frac{1}{B}\sum_{b=1}^B\nabla_{\bm{\theta}}\mathcal{L}(\bm{\theta}_t, \mathbf{x}_{\xi_b}, y_{\xi_b}) -\mathbb{E}_{\bm{\xi}}\nabla_{\bm{\theta}}\mathcal{L}(\bm{\theta}_t)$, where $B$ denotes the batch size, and the indices $\bm{\xi} = (\xi_1, \ldots, \xi_B)$ represent sampling of mini-batches.

For models trained without weight decay, the parameter-space gradient covariance $C = \mathbb{E}_{\bm{\xi}}{[}\bm{\epsilon_t}\bm{\epsilon}_t^T{]}$ is closely related to the mean Hessian $H$~\citep{ziyin2022strength}, producing the following corollary.

\begin{restatable}{corollary}{covariance}
\label{thm:findings:covariance}
Let $\bm{\theta}^*$ be a critical point for the loss $\mathcal{L}(\bm{\theta}, \mathbf{x}, y)$ on $\mathcal{D}$. Let $\mathbf{f}_{\bm{\theta}}$ denote a neural network with at least one hidden layer, with $\|\bm{\theta}^1\| > 0$. Then,
\begin{equation}
    \label{eq:findings:covariance}
    \frac{x_{\min}^2}{\|\bm{\theta}^1\|_2^2} \mathbb{E}_\mathcal{D}\|\nabla_\mathbf{x}\mathcal{L}\|_2^2 \le \trace{(S)} + o(\mathcal{L}(\bm{\theta}))
\end{equation}
with $S = C + \frac{1}{B}\nabla_{\bm{\theta}}\mathcal{L}(\bm{\theta})^T\nabla_{\bm{\theta}}\mathcal{L}(\bm{\theta})$ denoting the gradient uncentered covariance.
\end{restatable}

Figure~\ref{fig:appendix:covariance} tracks the largest principal component of $C$ for increasing model size, showing that stochastic noise peaks near the interpolation threshold, and then decreases in the overparameterized regime. In section~\ref{sec:appendix:covariance} we conclude our theoretical analysis by discussing the role of hyperparameters in controlling mean curvature and, in turn, input-space smoothness.

\begin{figure*}[t]
    \begin{minipage}{0.8\textwidth}
        \vspace*{\fill}
        \centering
        \includegraphics[width=0.45\linewidth,trim={0 0 40cm 11cm}, clip]{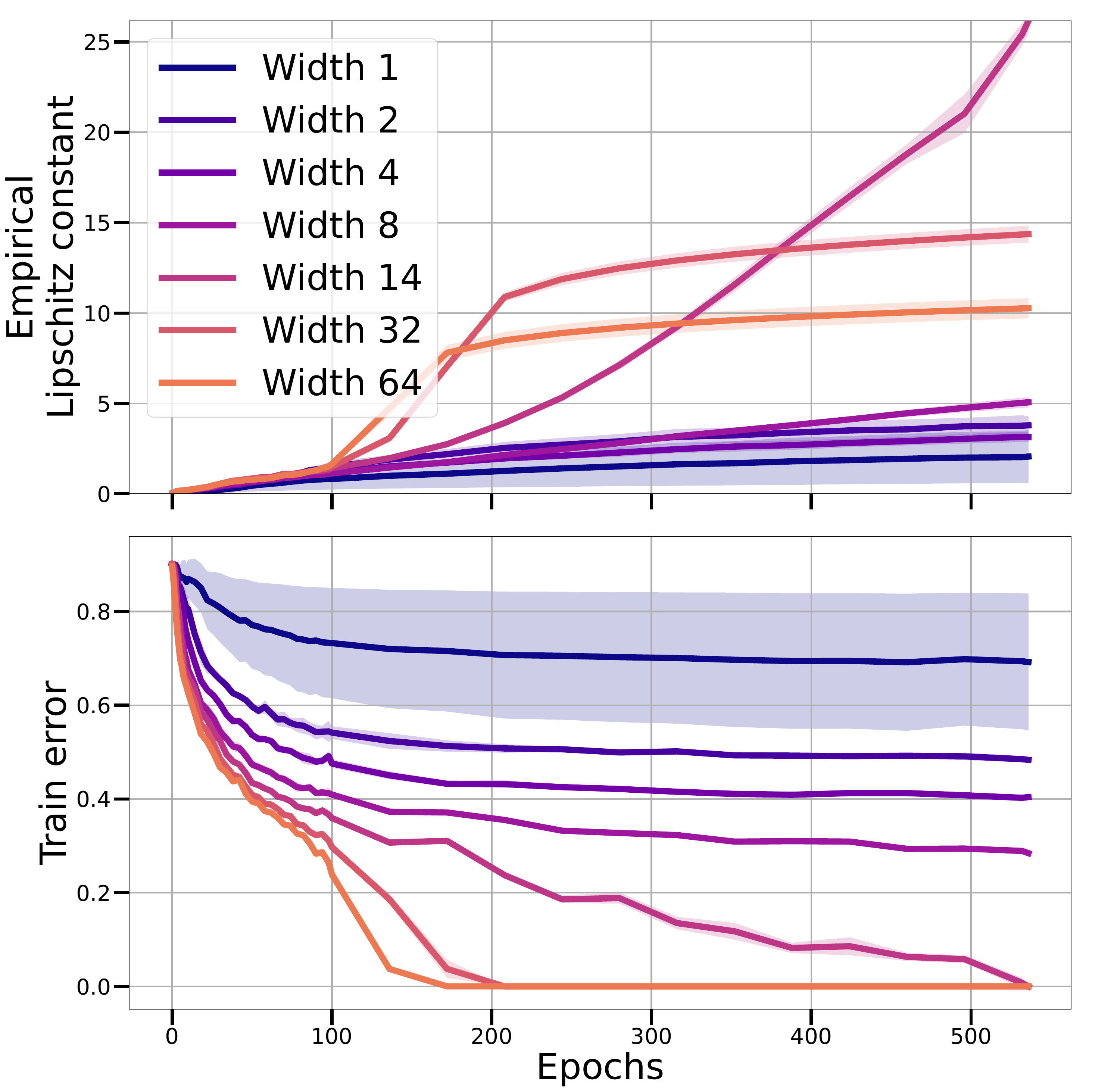} \
        \includegraphics[width=0.45\linewidth,trim={0 0 40cm 12cm}, clip]{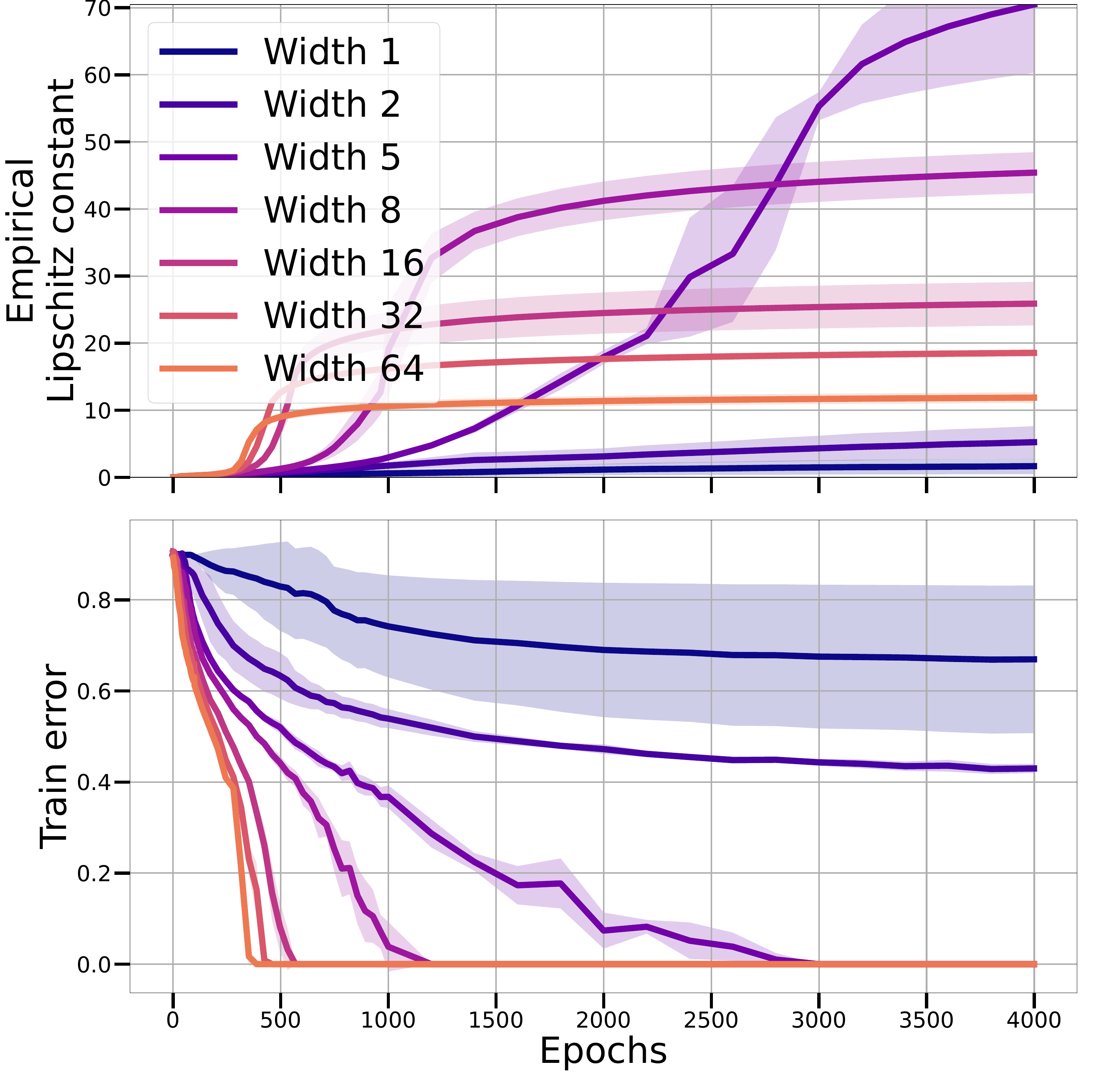}
    \end{minipage}\hspace*{\fill}%
    \begin{minipage}{0.2\textwidth}
        \vspace*{\fill}
        \centering
        \includegraphics[width=\linewidth,trim={0 0 10cm 10cm}, clip]{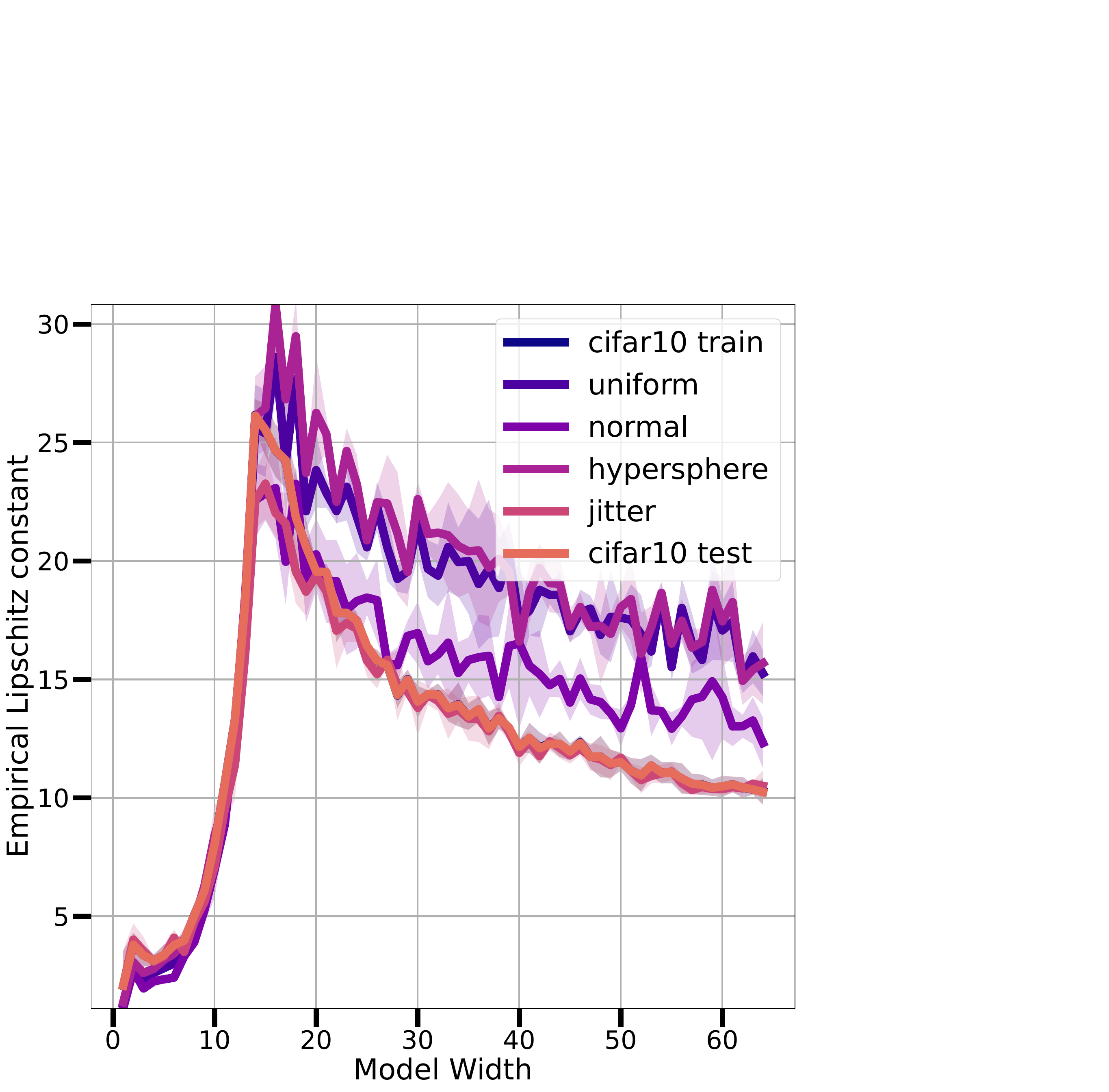}\par\vfill
        \includegraphics[width=\linewidth,trim={0 0 10cm 10cm}, clip]{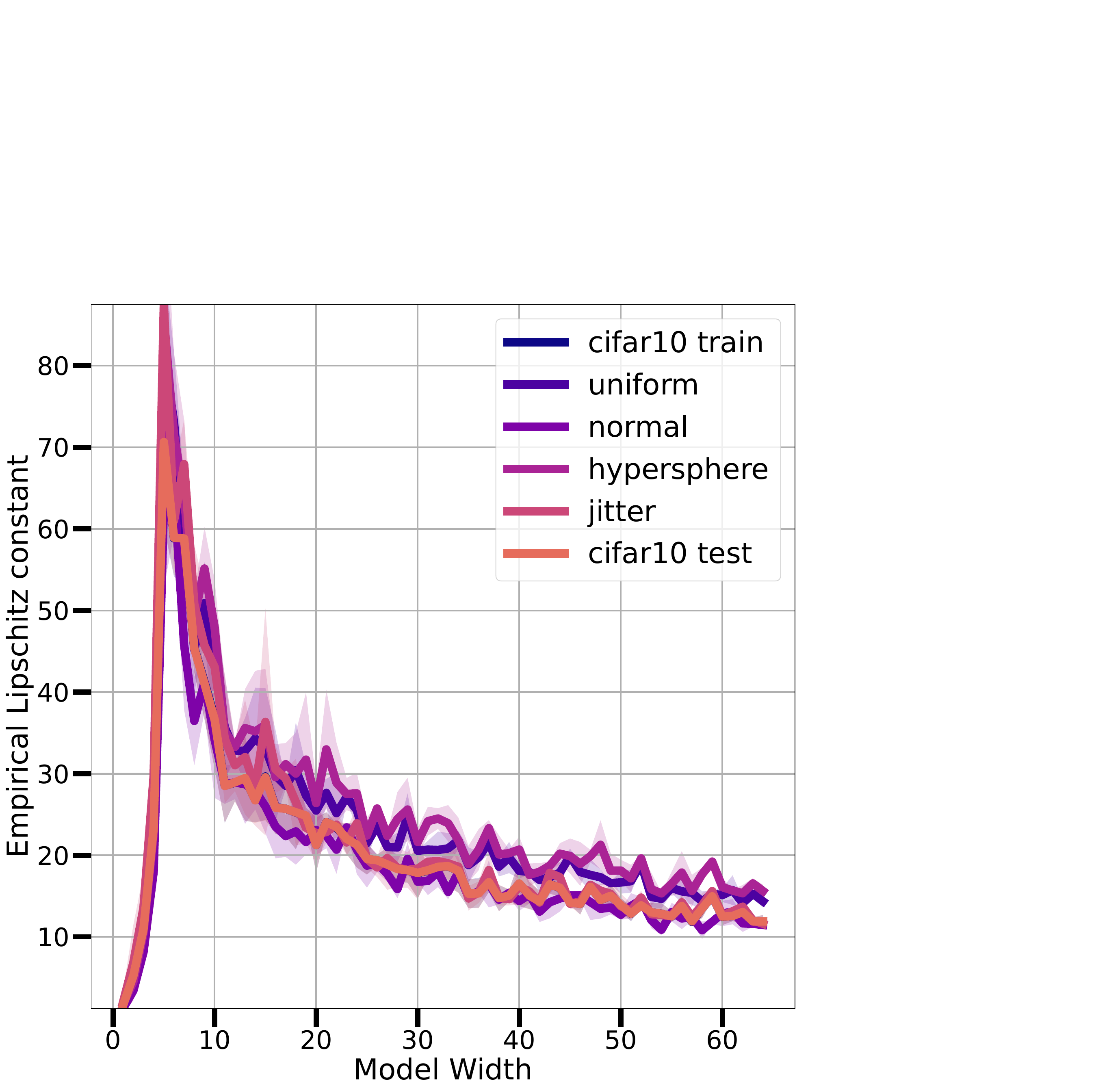}
    \end{minipage}
    \caption{(Left and middle panels) \textbf{Empirical Lipschitz constant over epochs} (top) and \textbf{Train error} (bottom) for ConvNets (left) and ResNets (middle) trained on CIFAR-10. (Right) Empirical Lipschitz constant for the same models on test data and random noise, for ConvNets (right, top panel) and ResNets (right, bottom panel).}
    \label{fig:implications:epochwise}
\end{figure*}

\paragraph{Summary}~~~By relating the empirical Lipschitz constant to the mean curvature of parameter space, we highlight a mechanism by which optimization implicitly controls sensitivity of $\mathbf{f}_{\bm{\theta}}$ on the training data $\mathcal{D}$ in proximity of a critical point $\bm{\theta}^*$.




%% file: sections/implications.tex
We conclude our study by exploring broader implications of the trends observed in section~\ref{sec:findings}. We begin by studying the empirical Lipschitz constant throughout epochs, then explore implications of our findings for understanding effective complexity of trained networks.

\paragraph{Overparameterization Accelerates Interpolation} 
Figure~\ref{fig:implications:epochwise} shows the empirical Lipschitz constant (top) of ConvNets (left) and ResNet18s (right) trained on CIFAR-10 with $20\%$ noisy training labels, for representative model widths, together with the respective training error (bottom). Heatmaps for all model widths are presented in appendix~\ref{sec:appendix:epochwise}, connecting to the test error in Figure~\ref{fig:appendix:epochwise}. We recall that the model-wise interpolation threshold~\citep{belkin2019reconciling} denotes the smallest model width $\omega_0$ that perfectly classifies the training set, in our experiment corresponding to $\omega_0 = 14$ for ConvNets, and $\omega_0 = 5$ for ResNets.

During training, we observe three distinct behaviours. Small models ($\omega \ll \omega_0$) are unable to interpolate the entire training set, and their training error as well as empirical Lipschitz constant quickly plateau, remaining stable therefrom. Increasing size among small models reduces their training error, and correspondingly increases the empirical Lipschitz constant. 

At the same time, models near the interpolation threshold $\omega_0$ -- peaking in test error and empirical Lipschitz constant (cfr.\ Figure~\ref{fig:findings:lipschitz}) -- are able to achieve interpolation, \textit{only when given considerable training budget}. Correspondingly, the empirical Lipschitz constant monotonically increases over training as the training error is reduced, resulting in models achieving worst sensitivity and worst test error. In contrast, consistently with the double descent trends reported in section~\ref{sec:findings}, large models ($\omega \gg \omega_0$) are able to quickly interpolate the training set, with the largest models requiring fewer epochs to achieve interpolation.

The seemingly unbounded empirical Lipschitz constant of models near $\omega_0$ suggests that the observations reported in~\citet{hardt2016train} -- for which prolonged training budgets may hurt generalization performance -- are pertaining only to models near the threshold. In fact, larger models can be trained for considerably long without a comparable increase in complexity.

\begin{figure*}[t]
    \centering
    \includegraphics[width=0.32\linewidth, trim={0cm 0cm 0cm 0cm}, clip]{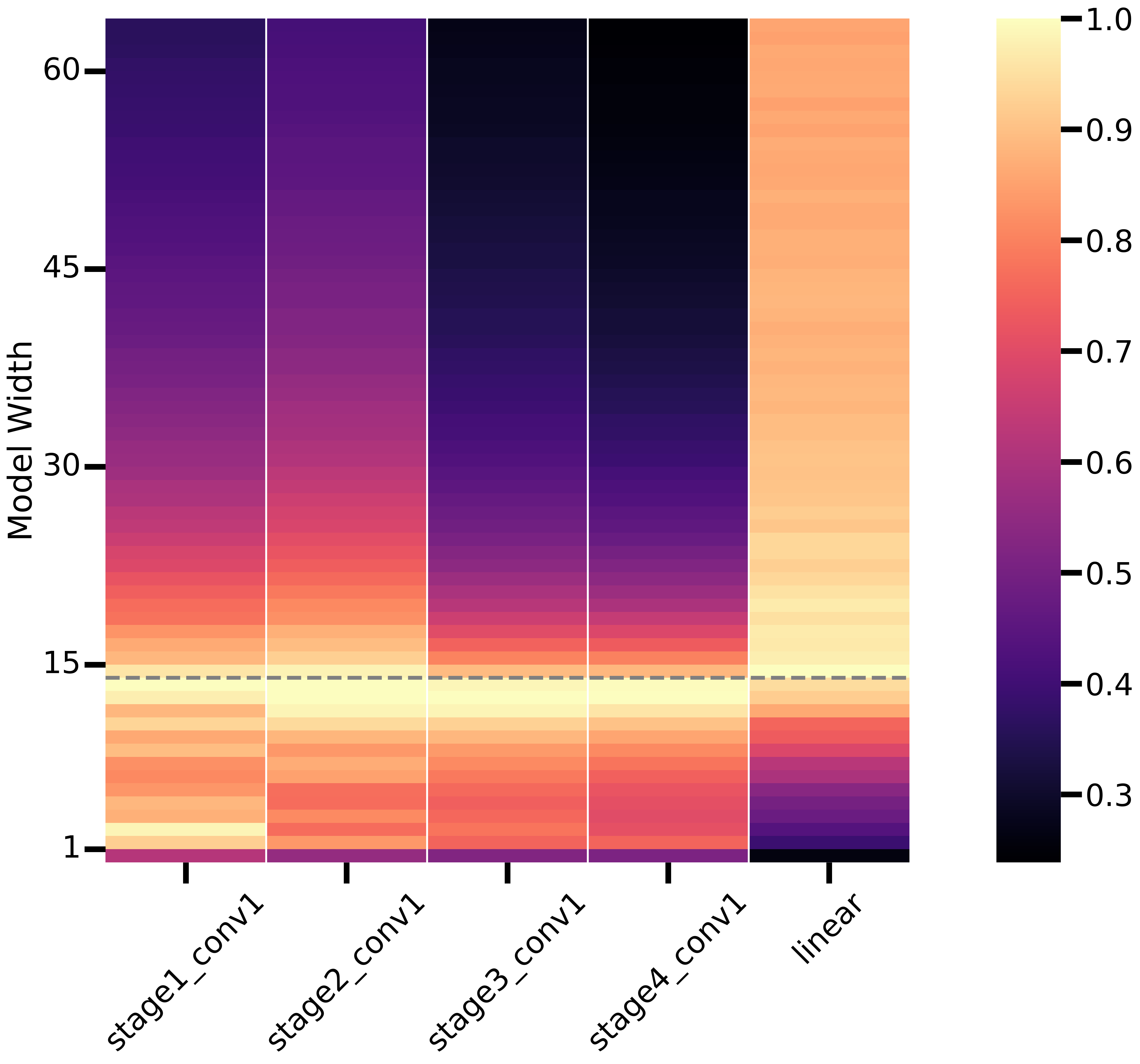}~
    \includegraphics[width=0.32\linewidth, trim={0cm 0cm 0cm 0cm}, clip]{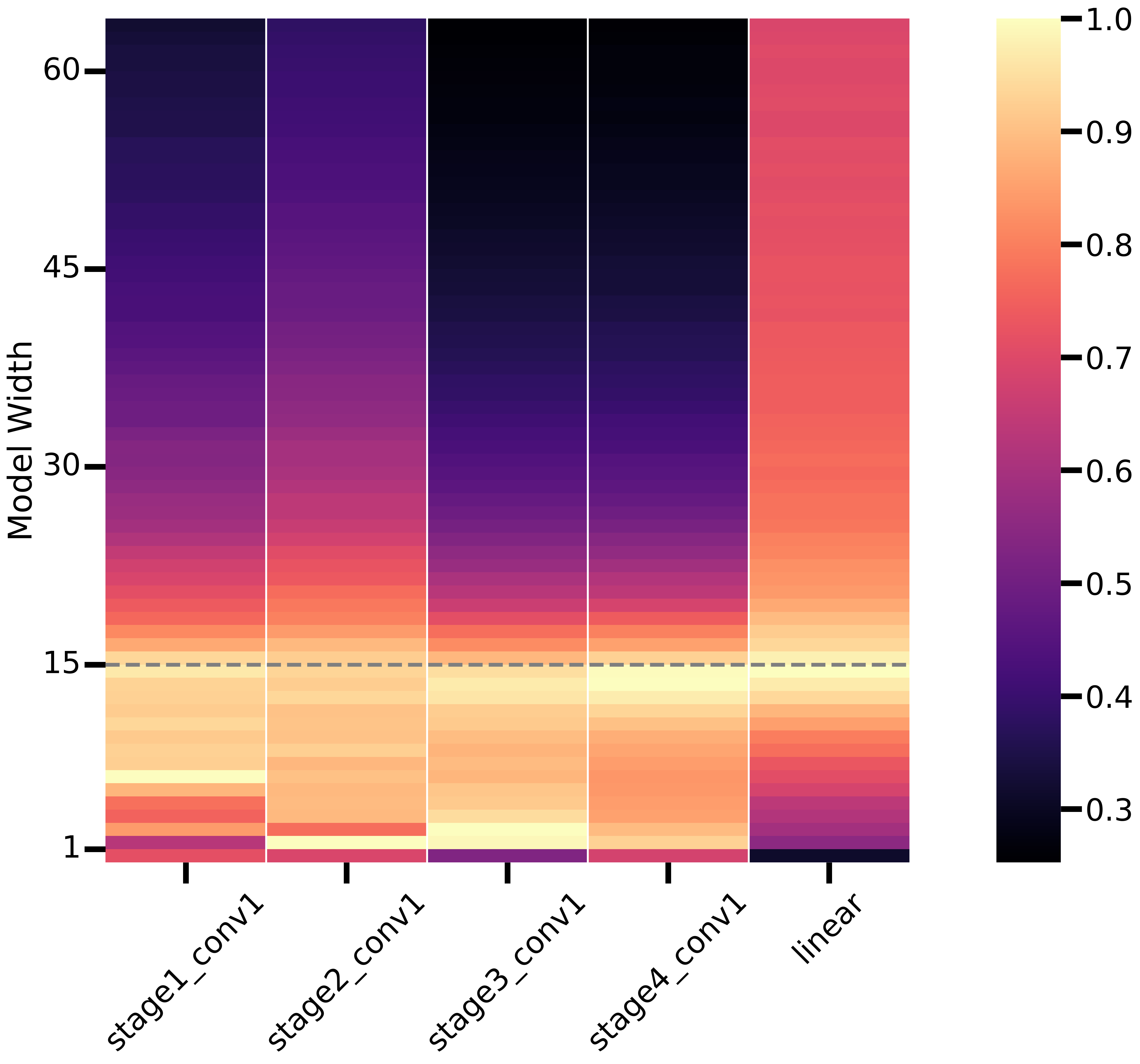}~
    \includegraphics[width=0.305\linewidth]{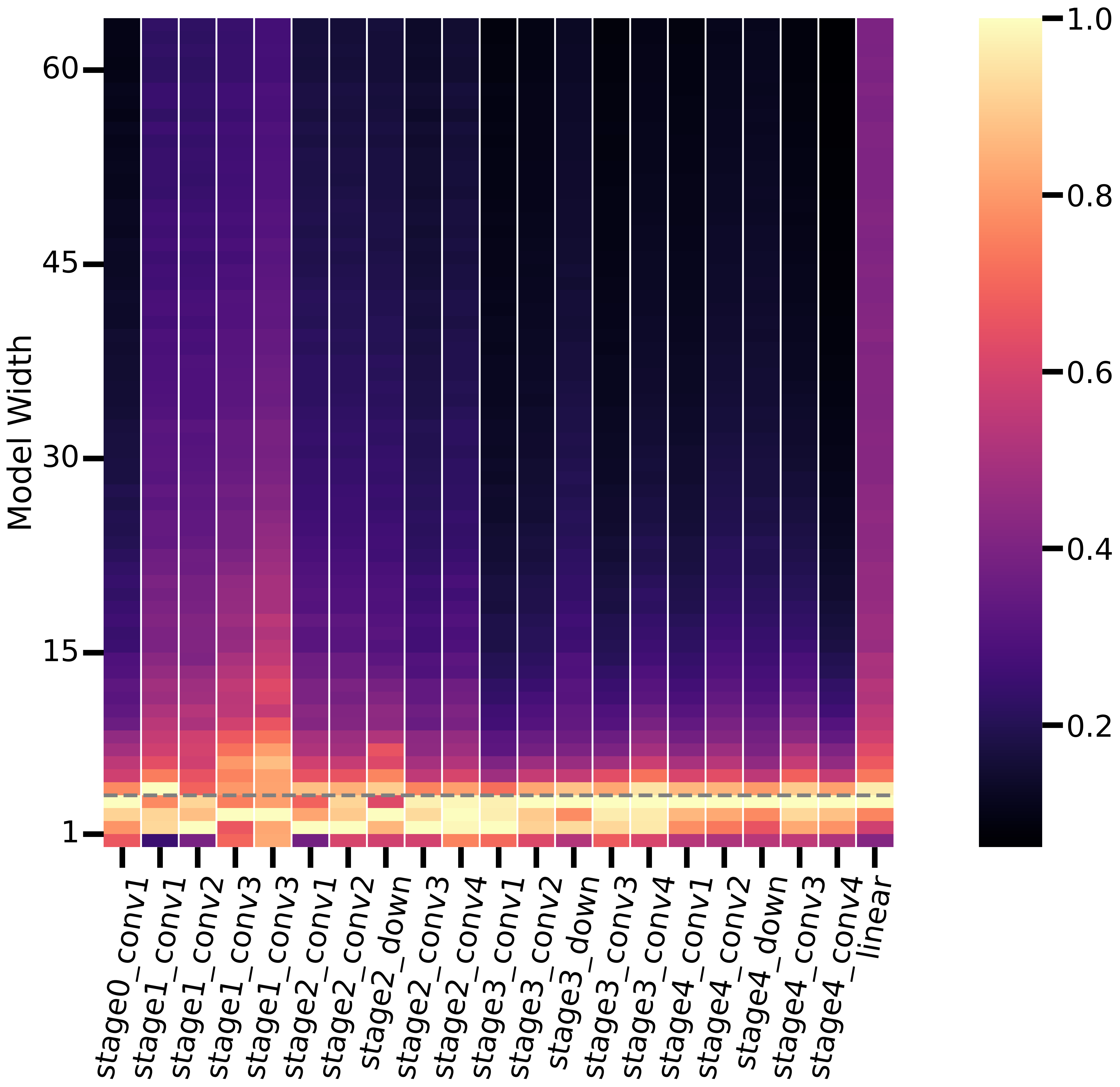}
    \caption{\textbf{Distance from initialization for each layer} of ConvNets trained on noisy CIFAR-10 (left), CIFAR-100 (middle), and ResNet18s trained on noisy CIFAR-10 (right). For each ConvNet and most ResNet layers, distance from initialization follows double descent, peaking at the interpolation threshold (dashed), suggesting global boundedness of the model function beyond training data for large models.}
    \label{fig:implications:init_distance}
\end{figure*}

\paragraph{Overparameterization Constrains Complexity} Referring again to Figure~\ref{fig:implications:epochwise}, we now consider implications for effective complexity of trained networks. First, since model weights are typically initialized to small values around zero~\citep{he2015delving,glorot2010understanding}, the empirical Lipschitz constant of all models is close to zero at the beginning of training. This corresponds to each model expressing a very simple function (low empirical Lipschitz constant), albeit with low generalization performance (typically close to random chance). Second, during training, fitting the dataset requires all models' Lipschitz constant to grow, with corresponding increase in model complexity (as measured by Equation~\ref{eq:findings:sobolev}). When zero error is reached ($\omega \ge \omega_0$), the empirical Lipschitz constant approximately plateaus, thereafter only slowly increasing over epochs. Recalling that large models interpolate faster, this finding suggest that large models may achieve interpolation via least meaningful deviation from initialization, realizing an overall smooth function even beyond the training set. 

To assess our hypothesis, we study the normalized distance from initialization $\frac{\|\bm{\theta}^\ell_T - \bm{\theta}^\ell_0\|_F}{\|\bm{\theta}^\ell_0\|_F}$ of each layer $\ell$, with $\bm{\theta}^\ell_0$ and $\bm{\theta}^\ell_T$ respectively denoting the layers' weights at initialization and convergence. Figure~\ref{fig:implications:init_distance} presents distance from initialization (colour) as model width increases ($y$-axis), for each layer ($x$-axis), for ConvNets (left) and ResNets (right) trained on CIFAR-10 with $20\%$ label noise, and ConvNets on CIFAR-100 with no label noise (middle).

For almost all layers, the quantity follows double descent as model width increases, peaking near the interpolation threshold (dashed line), and matching the epoch-wise trend reported in Figure~\ref{fig:implications:epochwise}.

This exciting finding supports our interpretation that faster interpolation, as promoted by overparameterization, results in model functions which are overall low-complexity, due to least (but meaningful) deviation from initialization. 

Our findings extend \citet{neyshabur2018role}, who initially reported that distance from initialization decreases for overparameterized models. Importantly, we show that the statistic is non-monotonic in model size, and that it strongly correlates with double descent for the test error. Together with the observed low mean curvature of large models shown in section~\ref{sec:findings:curvature}, this finding shares potential connection to the linear mode connectivity phenomenon~\citep{garipov2018loss}, by which low-loss paths that connect solutions obtained by optimization of the same model and task have been found in practice. Indeed, deeper layers of large models can be rewound to their value at initialization without considerably affecting model performance~\citep{chatterji2020intriguing,zhang2019all}, supporting our observations.

\begin{figure*}[t]
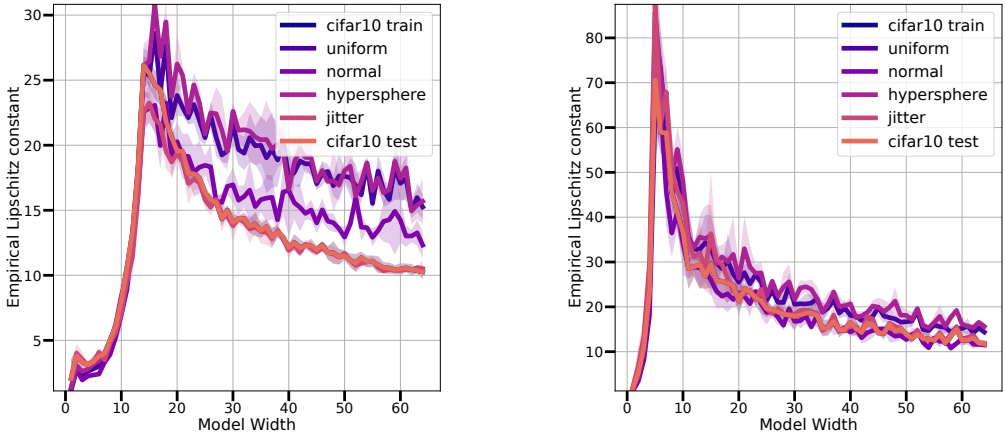

    \centering
    \hspace{\fill} \includegraphics[width=.4\linewidth,trim={0 0 10cm 10cm}, clip]{./img/cnn/cifar10/jacobian_operator_norm-ablation.pdf} \hspace{\fill}
    \includegraphics[width=.4\linewidth,trim={0 0 10cm 10cm}, clip]{./img/resnet18/cifar10/jacobian_operator_norm-ablation.pdf} \hspace{\fill}
    \caption{Empirical Lipschitz constant for ConvNets (left) and ResNet18s (right) test data and random noise.}
    \label{fig:implications:ablation}
\end{figure*}

\paragraph{Bounded Complexity Beyond Training Data} To conclude, in Figure~\ref{fig:implications:ablation} we estimate the empirical Lipschitz constant of ConvNets (left) and ResNets (right) trained on CIFAR-10, probing the networks by computing Equation~\ref{eq:findings:operator} on unseen test data as well as on random noise lying far from the support of the data distribution (experimental details in appendix~\ref{sec:appendix:random}). Intriguingly, the empirical Lipschitz constant remains bounded even far from $\mathcal{D}$, and the model-wise trend follows double descent, peaking at the interpolation threshold. This finding further strengthens the view that reduced distance from initialization via acceleration may essentially control complexity of the \textit{whole} model function.

%% file: sections/related_work.tex
Deep networks are able to express a rich family of functions as their model size increases~\citep{zhang2018understanding,telgarsky16benefits,cybenko1989approximation}. However, the complexity of generalizing models appears to be constrained in practice~\citep{neyshabur2018role,zhang2019all,neyshabur2015search}. Developing a formal characterization of the phenomenon is still a challenging open problem. Theoretical studies hinge upon finding a parameterization of the hypothesis class of trained networks that meaningfully constrains their expressivity. Importantly, several works rely on uniform bounds on the Lipschitz constant to constrain model function variation~\citep{kawaguchi2022robustness,ma2021linear,wei2019data,nagarajan2018deterministic,bartlett2017spectrally}. Moreover, in practical settings, explicitly regularizing the Lipschitz constant yields improved performance~\citep{gouk2021regularisation,moosavi2019robustness,novak2018sensitivity}.

Recently, the study of the Lipschitz constant has received renewed attention, with \citet{bubeck2021universal} prescribing overparameterization as a necessary condition for smooth interpolation, for a generic class of learners. Our work corroborates their findings, by also presenting an upper bound on the constant in relation to the geometry of the loss landscape.

While tightly estimating the Lipschitz constant is NP-hard for deep networks~\citep{jordan2020exactly,virmaux2018lipschitz}, we focus on complexity w.r.t.\ training data. Crucially, our Theorem~\ref{thm:findings:sobolev} extends a uniform bound on input-space sensitivity~\citep{ma2021linear} with a novel one that experimentally captures double descent.

Interestingly, a concurrent work uses Sobolev seminorms of ReLU networks on the training set to propose a complexity measure~\cite{dherin2022neural}. In line with our findings, their proposed measure mirrors the test error. Our works differ in that they focus on studying regularization of the metric, while instead we build a theoretical connection to several fundamental quantities capturing double descent in connection to the geometry of the loss landscape.

%% file: sections/conclusions.tex
We carry out an extensive study of the empirical Lipschitz constant of deep networks undergoing double descent, presenting implications for Lipschitz continuity and its implicit regularization via overparameterization. By building a theoretical connection with the loss landscape geometry, we present several correlates of double descent in terms of fundamental notions, that we hope will inspire further theoretical studies. We isolate two important quantities -- namely loss landscape curvature and distance of parameters from initialization -- respectively controlling optimization dynamics around a critical point and bounding model function complexity beyond training data. We believe understanding the structure and singularity of the overparameterized mapping from parameters to model functions is a fundamental open problem, which might reveal the true latent factors driving generalization.

%% file: sections/ethical_statement.tex
Our work aims at improving understanding of the impact of overparameterization in promoting regularisation in deep learning. We hope that the insights presented in our work and the extensive experimental verification of our findings will help guide and inspire theoretical works in understanding generalisation in deep learning, and eventually help guide the design of improved learning algorithms. While we do not see a direct societal impact of our work, our extensive empirical verification required extensive computation, amounting to several GPU-years of computation, and thus carrying a non-negligible carbon footprint.

%% file: sections/acknowledgments.tex
The authors thank David Lopez-Paz for many fruitful discussions on smoothness priors of deep networks and Kevin Scaman for technical feedback on an earlier draft of the paper.
The work was partially funded by Swedish Research Council project 2017-04609. Scientific computation was enabled by the supercomputing resource Berzelius provided by National Supercomputer Centre at Link\"{o}ping University and the Knut and Alice Wallenberg foundation, as well as by resources provided by the National Academic Infrastructure for Supercomputing in Sweden (NAISS) at Alvis partially funded by the Swedish Research Council through grant agreement no. 2022-06725.

%% file: sections/appendix.tex
\section{Organization of the Appendix}

\begin{itemize}
    \item Section~\ref{sec:appendix:setup} fully details our experimental setup, as well as the hardware infrastructure used for our experiments.
    \item Section~\ref{sec:appendix:power_method} presents the algorithms used for estimating the empirical Lipschitz constant and to measure parameter space curvature.
    \item Section~\ref{sec:appendix:extra} presents additional figures supporting the experiments in sections~\ref{sec:findings} and \ref{sec:implications}.
    \item[] \begin{itemize}
        \item Section~\ref{sec:appendix:upper_bound} discusses an upper bound on the true Lipschitz constant of piece-wise linear networks, undergoing double descent as mode size increases.
        \item Section~\ref{sec:appendix:hessian} presents additional results on parameter-space curvature of the loss landscape under double descent.
        \item Section~\ref{sec:appendix:covariance} discusses our results in relationship to training hyperparameters.
        \item Section~\ref{sec:appendix:transformers} extends our results to Transformers trained on machine translation tasks.
        \item Section~\ref{sec:appendix:epochwise} extends the epoch-wise trends reported for selected models in Figure~\ref{fig:implications:epochwise} to all model widths considered in our study.
        \item Section~\ref{sec:appendix:correlation} validates Theorem~\ref{thm:findings:curvature} in the interpolating regime for the models considered.
    \end{itemize}
    \item Section~\ref{sec:appendix:random} describes the distributions used for generating random validation data for Figure~\ref{fig:implications:epochwise}.
    \item Finally, section~\ref{sec:appendix:proofs} presents proofs of the formal statements appearing in section~\ref{sec:findings}.
\end{itemize}

\section{Experimental Setup}
\label{sec:appendix:setup}

We train a family of ConvNets composed of $4$ convolutional stages -- each corresponding to a \texttt{{[}Conv, ReLU{]}} block followed by maxpooling with stride $2$ -- and $1$ dense classification layer. We also train a family of ResNet18s~\citep{he2015delving} without batch normalization layers. Both network architectures are composed of $4$ convolutional stages, in which each spatial dimension is reduced by factor of $2$ and the number of learned feature maps doubles. More precisely, the convolutional stages respectively follow the progression ${[}\omega, 2\omega, 4\omega, 8\omega{]}$, where $\omega$ is the base width of the network, i.e.\ the number of feature maps learned at the first layer. 

In our experiments, following~\citet{nakkiran2019deep}, we vary the base width in the range $\omega = 1, \ldots, 64$. By controlling the network size through the network width, we produce a range of models presenting model-wise double descent in the test error, which captures the essence of the \textit{benign overfitting} phenomenon\citep{bartlett2020benign} observed for large interpolating networks, while also presenting \textit{malign overfitting} for models near the interpolation threshold. Furthermore, controlling model size through base width allows us to keep the network depth fixed, and focus our study on effective complexity of fixed-depth networks, for two network architecture families (ConvNets and ResNets).

To tune hyperparameters, we take a random validation split of size $1000$ from each CIFAR training set. We train all networks with SGD with momentum $0.9$, batch size $128$, and fixed learning rate, set at $\eta = 5\rm{e}-3$ for the ConvNets and $\eta = 1\rm{e}-4$ for the ResNets. We train the ConvNets for $500$ epochs, and the ResNets for $4000$ epochs. To stabilize prolonged training, we use learning rate warmup over the first $5$ epochs of training, starting from a learning rate $\eta_0 = 10^{-1} \times \eta$.

\noindent\textbf{Transformers on Machine Translation tasks}~~~We also train multi-head attention-based Transformers~\citep{vaswani2017attention} for neural machine translation tasks. We vary model size by controlling the embedding dimension $d_e$, as well as the width $h$ of all fully connected layers, which we set to $h = 4d_e$ following the architecture described in~\citet{vaswani2017attention}. We train the transformer networks on the WMT'14 En-Fr task~\citep{machavcek2014results}, as well as ISWLT'14 De-En~\citep{cettolo2012wit3}. The training set of WMT'14 is reduced by randomly sampling $200$k sentences, fixed for all models. The networks are trained for $80$k gradient steps, to optimize per-token perplexity, with $10\%$ label smoothing, and no dropout, gradient clipping or weight decay.

\noindent\textbf{Hardware specifications}~~~Our codebase is implemented in Pytorch version 1.11, running on a local cluster equipped with NVIDIA A100 GPUs with 40GB onboard memory. Our experiments involve training $64$ ConvNets and ResNets (each corresponding to a base width $\omega$) for up to $4000$ epochs, producing $72$ model checkpoints per network. We use $3$ random seeds for the ConvNets and $5$ for the ResNets, controlling network initialization and the shuffling and sampling of mini-batches from the training set. We use a dedicated random seed for generating the validation split used for hyperparameter tuning, fixed for all networks, as well as a fixed seed for corrupting the CIFAR training labels. The empirical Lipschitz constant is estimated and averaged on every training point for each of the reported configurations.

\noindent\textbf{Number of model parameters}~~~Our main empirical finding is that, while network size increases -- causing uniform upper bounds like \citet{ma2021linear} (Theorem 3) to monotonically increase -- the empirical Lipschitz constant of the models decreases past the interpolation threshold. To better frame our observations, we report in Table~\ref{tab:appendix:params} the number of parameters for a few representative models in our experiments.

\begin{table}[t]
    \caption{Number of model parameters $p$ for representative widths $\omega$ on CIFAR-10. Models near the interpolation threshold are marked in bold.}
    \label{tab:appendix:params}
    \centering
    \begin{tabular}{@{}ccc@{}}
    \toprule
    $\omega$ & ConvNet     & ResNet18     \\ \midrule
    1        & $510$       & $2,902$      \\
    2        & $1,766$     & $11,242$     \\
    4        & $6,546$     & $\mathbf{44,266}$     \\
    8        & $25,178$    & $175,690$    \\
    16       & $\mathbf{98,730}$    & $700,042$    \\
    32       & $390,986$   & $2,794,762$  \\
    64       & $1,556,106$ & $11,168,266$ \\ \bottomrule
    \end{tabular}
\end{table}

\section{Operator Norm Estimation}
\label{sec:appendix:power_method}

For linear operators $\mathbf{A}: (\mathbb{R}^d, \|\cdot\|_p) \to (\mathbb{R}^K, \|\cdot\|_q)$, the operator norm is defined as 
\begin{equation}
\label{eq:methodology:operator_norm}
\| A \|_{\text{op}} := \sup_{\mathbf{x} : \|\mathbf{x}\|_p \ne 0} \frac{\|A\mathbf{x}\|_q}{\|\mathbf{x}\|_p},
\end{equation}
where the norms $\|\cdot\|_p$ and $\|\cdot\|_q$ are respectively taken in input and logit space. Crucially, if $p = q = 2$, then the operator norm can be estimated by computing the largest singular value of $\mathbf{A}$. For any data point $\overline{\mathbf{x}} \in \mathbb{R}^d$, evaluating the Jacobian at $\overline{\mathbf{x}}$ yields $\nabla_\mathbf{x}\mathbf{f}(\mathbf{x}, \bm{\theta})|_{\mathbf{x} = \overline{\mathbf{x}}} = \bm{\theta}_\epsilon$, i.e. the linear function computed by $\mathbf{f}$ on the activation region $\epsilon$ of $\overline{\mathbf{x}}$. Hence, at each point, $\|\bm{\theta}_\epsilon\|_{\text{op}}$, provides an estimate of worst-case sensitivity of the corresponding linear ``piece'' of $\mathbf{f}$. We note that, while the supremum $\|\mathbf{A}\|_{\text{op}}$ may not be attained within the activation region of $\overline{\mathbf{x}}$, the operator norm upper bounds worst-case sensitivity within the region. Furthermore, activation regions neighbouring training data tend to compute approximately the same linear function~\citep{gamba2022all, roth2020adversarial}.

\paragraph{Computing the operator norm} Computing the operator norm of $\bm{\theta}_\epsilon \in \mathbb{R}^{K \times d}$ entails two steps. First, computing the gradient $\nabla_\mathbf{x} \mathbf{f}|_{ \mathbf{x} = \overline{\mathbf{x}}} = \bm{\theta}_\epsilon$  (via automatic differentiation), and then estimating its largest singular value. To perform the latter, we use a standard power method. Starting at iteration $t=0$ with randomly initialized vectors $\tilde{\mathbf{u}}_0 \in \mathbb{R}^K$, $\tilde{\mathbf{v}}_0 \in \mathbb{R}^d$, and corresponding normalized vectors $\mathbf{u}_0 = \frac{\tilde{\mathbf{u}}_0}{\| \tilde{\mathbf{u}}_0\|_q}$, \mbox{$\mathbf{v}_0 = \frac{\tilde{\mathbf{v}}_0}{ \| \tilde{\mathbf{v}}_0\|_p}$}, at step $t$ we compute 
\begin{equation}
\begin{aligned}
    \tilde{\mathbf{u}}_t   &\gets \nabla_\mathbf{x}\mathbf{f}~\mathbf{v}_{t-1} \\
    \tilde{\mathbf{v}}_{t} &\gets \mathbf{u}_t^T \nabla_\mathbf{x}\mathbf{f} \\
    \sigma_t               &\gets \mathbf{u}_t^T \nabla_\mathbf{x}\mathbf{f}~\mathbf{v}_t
\end{aligned}
\end{equation}
with $\sigma_t$ storing the largest singular value at convergence, defined based on a relative tolerance $1\rm{e}-6$ on the size of the increments of $\sigma_t$.

In our experiments, we estimate the Lipschitz constant of the network by its empirical constant, $\mathbb{E}_{\mathcal{D}}\| \nabla_\mathbf{x}\mathbf{f}|_{\mathbf{x} = \mathbf{x}_n}\|_{\text{op}}$, for all training points $\mathbf{x}_n \in \mathcal{D}$. We extend the empirical Lipschitz constant estimation to validation data in Figure~\ref{fig:implications:epochwise}. 

\subsection{Hessian Eigenvalue Estimation}
\label{sec:appendix:hessian_estimation}

The power method detailed in section~\ref{sec:appendix:power_method} can be used to estimate the largest eigenvalue of the parameter-space loss Hessian (Figure~\ref{fig:appendix:hessian}), as well as the first principal component of the gradient noise covariance (Figure~\ref{fig:appendix:covariance}). Importantly, for large networks, direct computation of any of the two matrices is infeasible due to the large number of parameters. Instead, we use efficient Jacobian-vector products for estimating the noise covariance (which entails accumulating the true gradient $\mathbb{E}_{\bm{\xi}}\nabla_{\bm{\theta}}\mathcal{L}$ at each iteration of the algorithm. For the Hessian matrix, Jacobian-vector products can be turned into Hessian-vector products using Pearlmutter's trick~\citep{pearlmutter1994fast}.

\subsection{Hessian Trace Estimation}
\label{sec:appendix:trace}

To estimate the Hessian trace in Figure~\ref{fig:findings:hessian}, we use Hutchinson's algorithm~\citep{hutchinson1990stochastic}, which provides an unbiased estimator of the trace. At each iteration $t$, the algorithm generates a set of $V$ random test vectors, $\mathbf{v}_n \in \mathbb{R}^p$ with zero mean $\mathbb{E}{\mathbf{v}_n} = \frac{1}{p}\sum\limits_{i = 1}^p v_n^i = 0$ and variance $\mathbb{E}{[}\mathbf{v}_n\mathbf{v}_n^T{]} = I_p$, by sampling each $\mathbf{v}_n$ from the Rademacher distribution. At iteration $t$, the algorithm computes $\trace_t = \frac{1}{V}\sum\limits_{n = 1}^V\mathbf{v}_n^TH\mathbf{v}_n$, where $H$ is the expected loss Hessian. Notably, the trace is obtained by computing $\frac{1}{V}\sum\limits_{i=1}^p\mathbf{v}_n^TH\mathbf{v}_n = \trace{(\mathbf{v}_n^TH\mathbf{v}_n)}$, where the Hessian is never instantiated and is implicitly computed via Hessian-vector products~\citep{pearlmutter1994fast}. In our work, we estimate the trace using $V=100$ test vectors. 

\section{Additional Experiments}
\label{sec:appendix:extra}

\subsection{Upper Bounding the Lipschitz Constant}
\label{sec:appendix:upper_bound}

We complement our analysis of the empirical Lipschitz lower bound of Equation~\ref{eq:findings:operator} by studying an upper bound on the true Lipschitz constant $\lip(\mathbf{f})$, described by Equation~\ref{eq:findings:lipschitz_upper_bound}. Figure~\ref{fig:findings:lipschitz_upper_bound} presents the upper bound for ConvNets trained on CIFAR-10, CIFAR-100, and ResNets trained on CIFAR-10. Similarly to the empirical Lipschitz lower bound, the upper bound closely follows double descent for the test error, peaking near the interpolation threshold. We note that, since the upper bound is independent of the binary activation pattern of ReLU, it captures global worst-case sensitivity of the network on the whole domain $\Omega$ of $\mathbf{f}$, suggesting that the non-monotonic dependency of Lipschitz continuity on model size holds also beyond the training set $\mathcal{D}$. This observation is substantiated by experimentally extending the lower Lipschitz bound in Equation~\ref{eq:findings:operator} to validation as well as random data in Figure~\ref{fig:implications:ablation}, as well as by observing that distance from initialization of trained weights also undergoes double descent (Figure~\ref{fig:implications:init_distance}). Together, with Theorem~\ref{thm:findings:curvature}, these observations suggest that the main factor controlling double descent when the number of model parameters varies is the loss landscape curvature, and which in turn controls input-space sensitivity on the training set through the empirical Lipschitz lower bound. We explore parameter-space curvature in more detail in the next section.

\begin{figure*}[t]
    \centering
    \includegraphics[width=0.27\linewidth, trim={0cm 0cm 17cm 10cm}, clip]{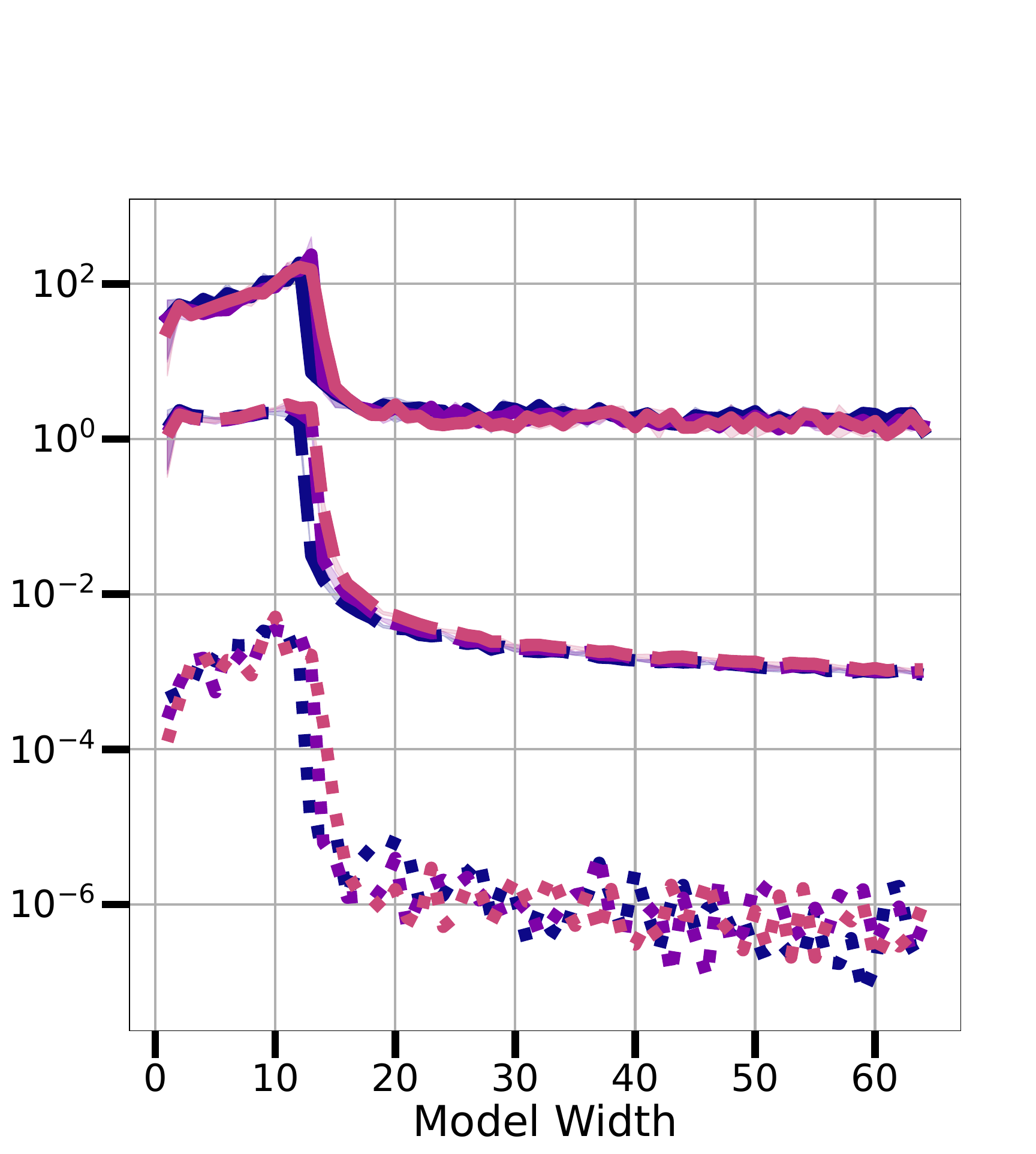}~
    \includegraphics[width=0.27\linewidth, trim={0cm 0cm 17cm 10cm}, clip]{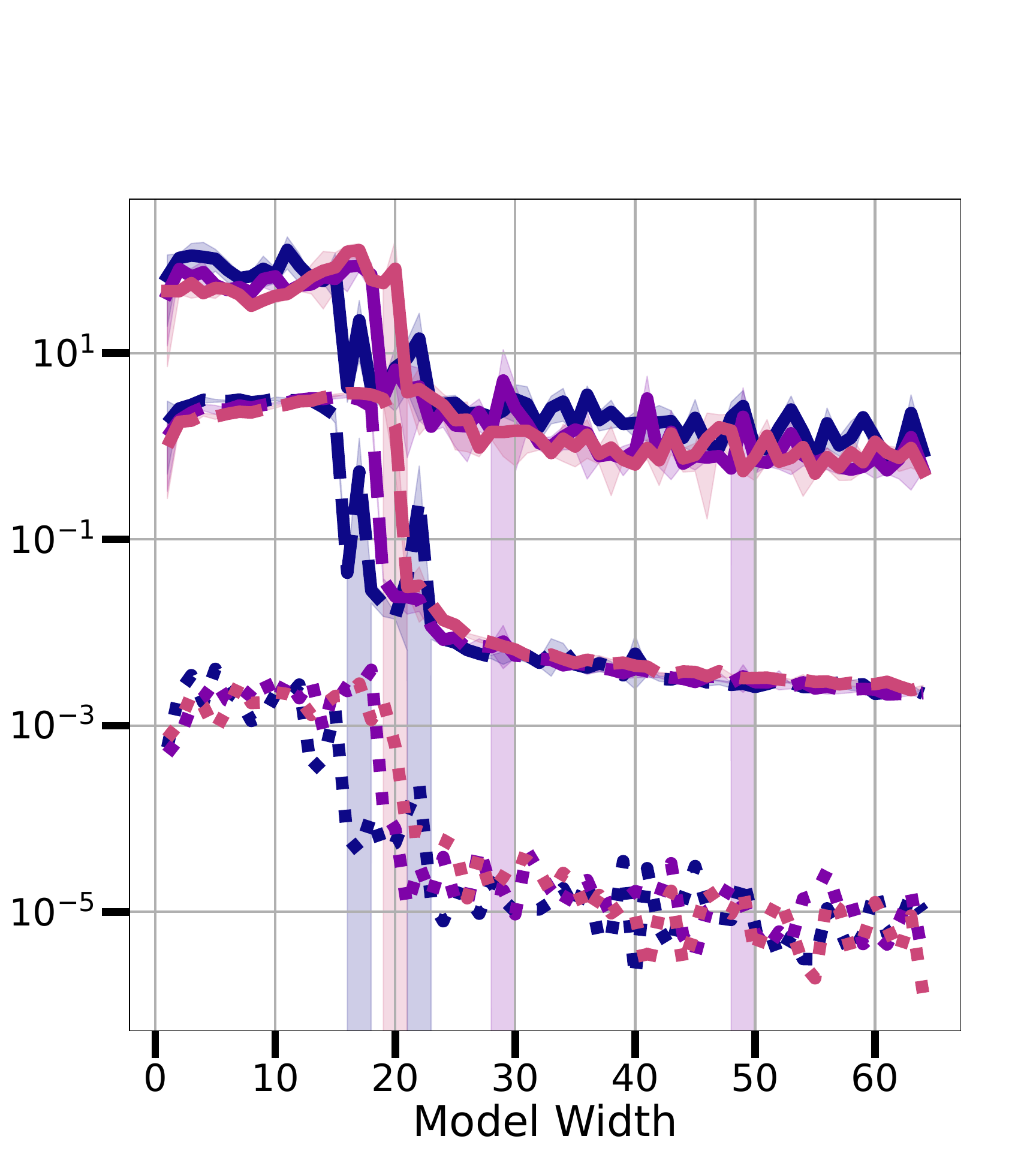}~
    \includegraphics[width=0.43\linewidth, trim={0cm 0cm 0cm 10cm}, clip]{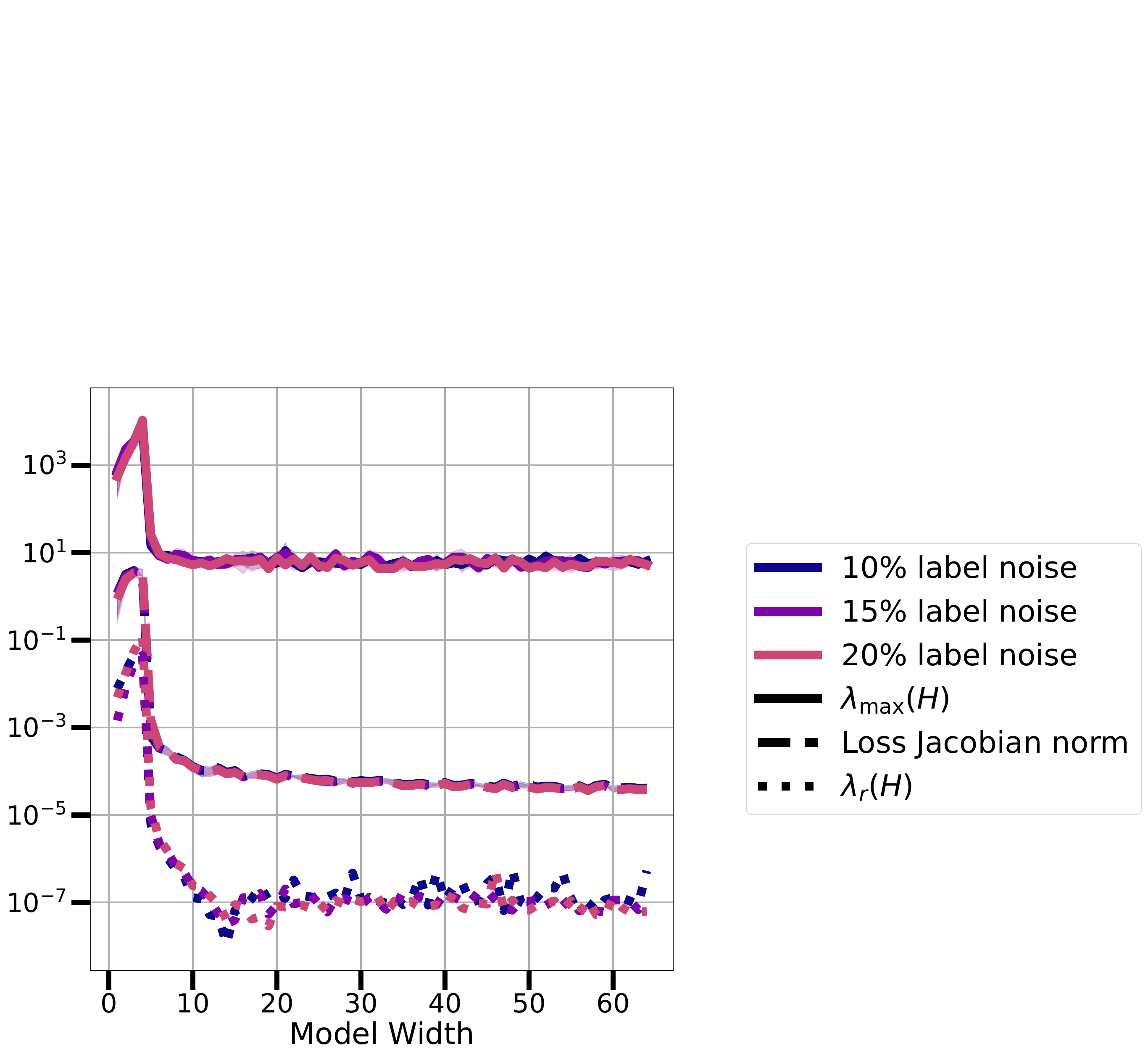}
    \caption{\textbf{Maximum and minimum curvature} for the loss in parameter space, and \textbf{input-space loss Jacobian norm}. From left to right: ConvNets trained on CIFAR-10 (left), CIFAR-100 (middle) and ResNets trained on CIFAR-10 (right). In all settings, minimum and maximum parameter-space curvature strongly correlate with double descent, peaking at the interpolation threshold, and highlighting a nonlinear dependence on model size.}
    \label{fig:appendix:hessian}
\end{figure*}

\subsection{Parameter Space Curvature}
\label{sec:appendix:hessian}

Theorem~\ref{thm:findings:curvature} provides a bound on input-space sensitivity via mean curvature of the loss in parameter space, connecting parameter-space dynamics to input-space sensitivity under double descent. In Figure~\ref{fig:appendix:hessian}, we explore parameter-space curvature in more detail, by plotting the largest and smallest non-zero Hessian eigenvalues, together with the input-space loss Jacobian norm studied in Corollary~\ref{thm:findings:sobolev_loss}. We observe that maximum and minimum \textit{parameter space} curvature mirror \textit{input space} sensitivity, as measured by the loss Jacobian norm, peaking near the interpolation threshold, and then decreasing. Our observations support the hypothesis that overparameterization non-monotonically controls flatness of the parameter space, which in turn controls sensitivity of the model function.

\subsection{Mean Curvature, Stochastic Noise and Linear Stability}
\label{sec:appendix:covariance}

In this section, we study Theorem~\ref{thm:findings:curvature} in relation to the training dynamics in proximity of a critical point $\bm{\theta}^*$. Finally, we discuss the influence of training hyperparameters on curvature. 

First, we draw a connection between the mean loss Hessian $H$ and gradient noise covariance $C$, as defined in Corollary~\ref{thm:findings:covariance}. Then, we study reachability of the critical point $\bm{\theta}^*$ by SGD, in relation to training hyperparameters. In turn, this allows us to draw a connection between hyperparameters, their influence on mean curvature, and input-space sensitivity.

At iteration $t$, the update rule of SGD with batch size $B$, and learning rate $\eta$, is given by
\begin{equation}
\label{eq:appendix:sgd}
\bm{\theta}_{t+1} = \bm{\theta}_t - \frac{\eta}{B}\sum\limits_{b=1}^B \nabla_{\bm{\theta}}\mathcal{L}(\bm{\theta}_t, \mathbf{x}_{\xi_b}, y_{\xi_b})
\end{equation}
with random variables $\bm{\xi} = (\xi_1, \ldots, \xi_B)$ representing sampling of mini-batches. At step $t$, the stochastic noise $\bm{\epsilon}_t$ of SGD is given by 
\begin{equation}
\label{eq:appendix:noise}
\bm{\epsilon}_t = \frac{1}{B}\sum\limits_{b=1}^B\nabla_{\bm{\theta}}\mathcal{L}(\bm{\theta}_t, \mathbf{x}_{\xi_b}, y_{\xi_b}) -\mathbb{E}_{\bm{\xi}}\nabla_{\bm{\theta}}\mathcal{L}(\bm{\theta}_t)
\end{equation}
dependent both on the current parameter $\bm{\theta}_t$ and $\bm{\xi}_t$~\citep{ziyin2022strength,mori2022power}. Importantly, the noise covariance $C = \mathbb{E}_{\bm{\xi}}{[}\bm{\epsilon_t}\bm{\epsilon}_t^T{]}$ accounts for fluctuations of the training dynamics around $\bm{\theta}^*$.

For loss functions without Tikhonov regularization terms such as weight decay, the noise covariance matrix has been shown by several works to be equivalent to the mean Hessian~\citep{ziyin2022strength, mori2022power}. Hence, the bound in Theorem~\ref{thm:findings:curvature} can be expressed in terms of fluctuations of the parameter gradients around $\bm{\theta}^*$, providing Corollary~\ref{thm:findings:covariance}, restated below. A proof of the statement is given in section~\ref{sec:appendix:proofs}.

\covariance*

\begin{figure*}[t]
    \centering
    \includegraphics[width=0.32\linewidth, trim={0cm 0cm 13cm 12cm}, clip]{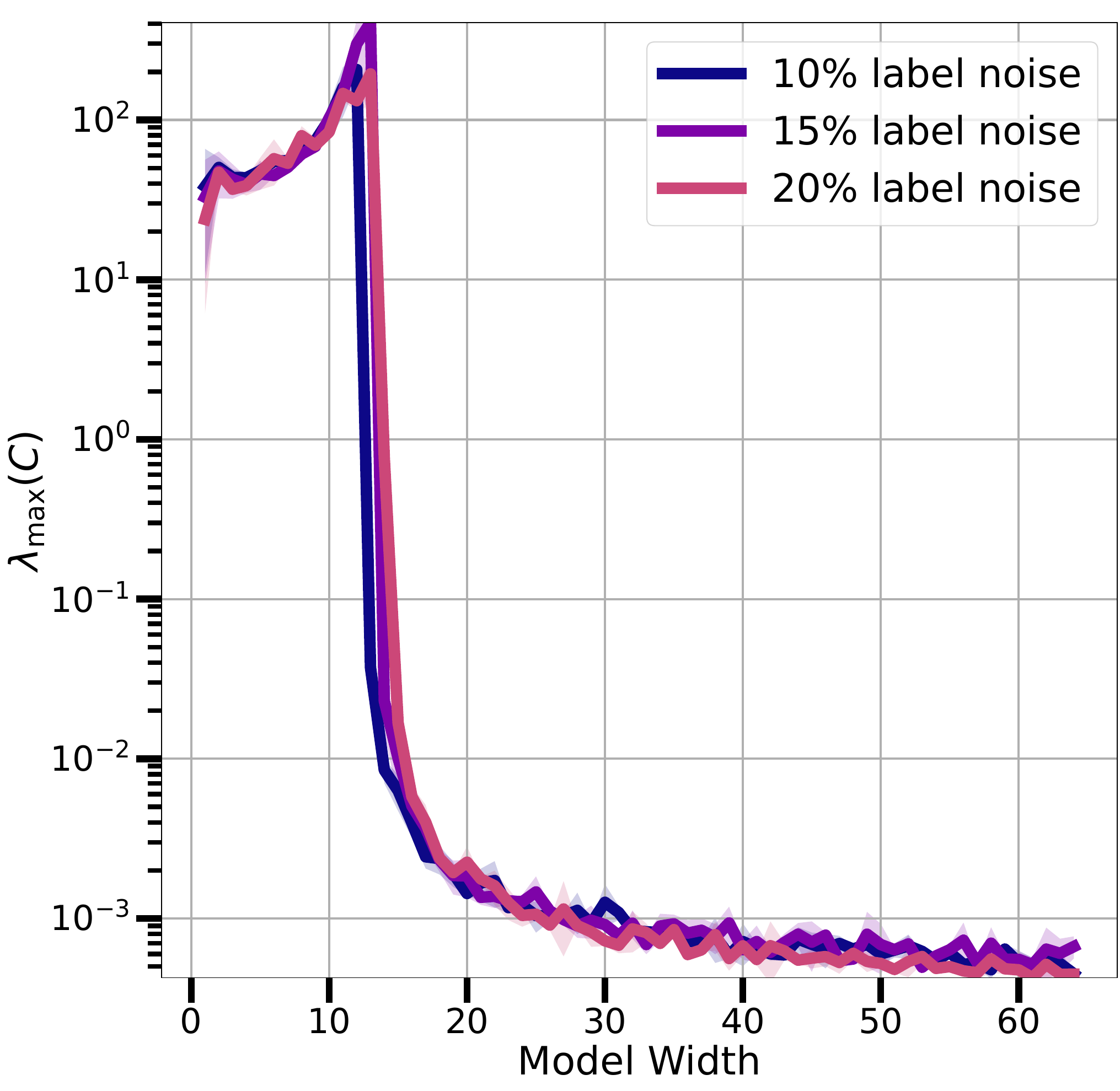}~
    \includegraphics[width=0.315\linewidth, trim={0cm 0cm 13cm 12cm}, clip]{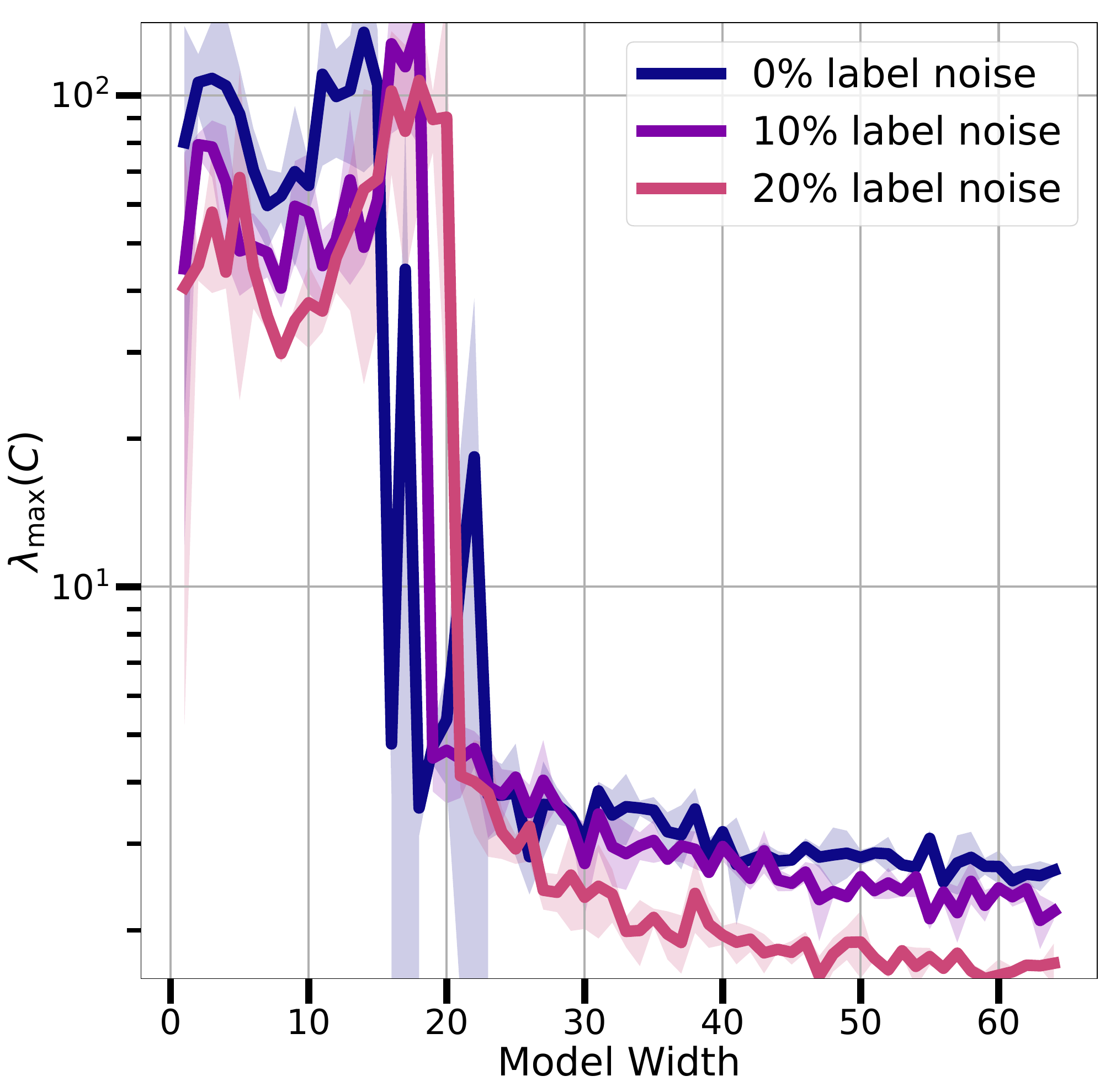}~
    \includegraphics[width=0.32\linewidth, trim={0cm 0cm 13cm 12cm}, clip]{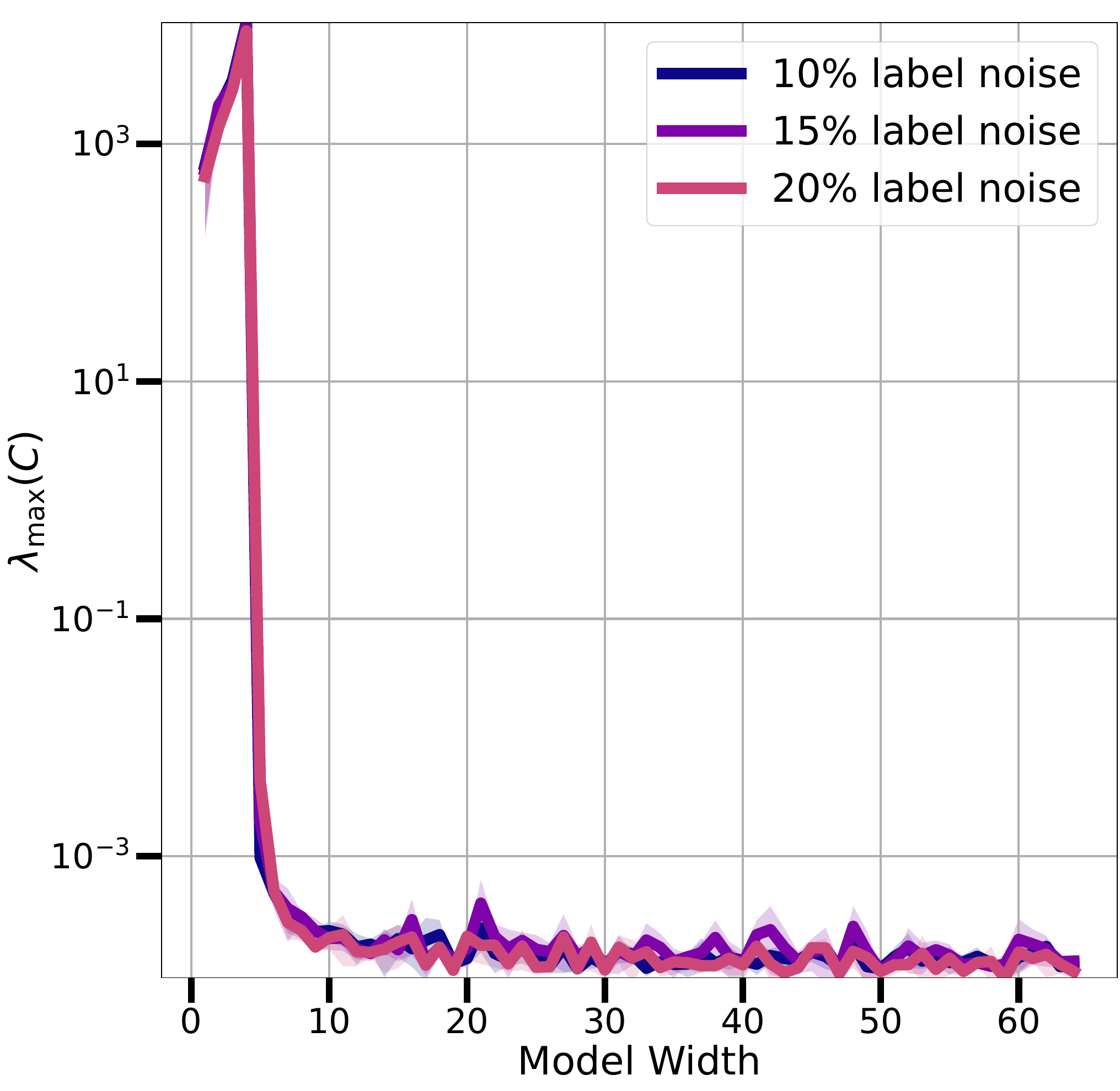}
    \caption{\textbf{Dominant noise-covariance eigenvalue.} (Top) From left to right: ConvNets trained on CIFAR-10 (left), CIFAR-100 (middle) and ResNets trained on CIFAR-10 (right). In all settings, the magnitude of stochastic noise strongly correlates with double descent, peaking at the interpolation threshold, and highlighting a nonlinear dependence on network width.}
    \label{fig:appendix:covariance}
\end{figure*}

Figure~\ref{fig:appendix:covariance} shows the largest principal component $\lambda_{\max}(C)$, as model size increases. Similarly to the mean curvature, stochastic noise strongly correlates with the empirical Lipschitz constant, decreasing considerably in the interpolation regime, and showing that overparameterization non-monotonically affects the dynamics of training.

After having established a clearer connection between training dynamics in proximity of $\bm{\theta}^*$ and our main bound, we discuss the role of training hyperparameters in affecting mean curvature.

\paragraph{Linear Stability of SGD} 

In proximity of a critical point $\bm{\theta}^*$, it is possible to derive stability conditions under which the point is attainable by SGD~\citep{wu2018sgd}. Essentially, under the quadratic approximation of Equation~\ref{eq:findings:taylor}, the dynamics of SGD are said to be linearly stable in a neighbourhood of $\bm{\theta}^*$ if 
$\exists~\gamma > 0$ for which $\mathbb{E}_{\mathcal{D}} \|\bm{\theta}_{t}\|^2 \le \gamma~\mathbb{E}_{\mathcal{D}}\|\bm{\theta}_0 \|^2$, for all $t$~\citep{hosoe2022second}. \citet{wu2018sgd} provide linear stability conditions for SGD in the following proposition.

\begin{proposition}{(\citet{wu2018sgd}, Theorem 1.)}
\label{thm:appendix:linear_stability}
    A critical point $\bm{\theta}^*$ is linearly stable for SGD with learning rate $\eta$ and batch size $B$ if 
    $$ \lambda_{\max} \Big( (I_p - \eta H)^2 + \frac{\eta^2(N - B)}{B (N -1)}\Sigma \Big) \le 1$$
    with $N = |\mathcal{D}|$ and $\Sigma = \mathbb{E}_{\mathcal{D}}(H^2) -  (\mathbb{E}_{\mathcal{D}}H)^2$.
\end{proposition}

Additionally, \citet{wu2018sgd}, provide a necessary condition for Proposition~\ref{thm:appendix:linear_stability} to hold, by requiring $\lambda_{\max}(I_p - \eta H) \le 1$ and $\lambda_{\max}(\frac{\eta^2(N - B)}{B (N -1)}\Sigma) \le 1$ to hold separately, providing the conditions 
\begin{equation}
\label{eq:appendix:stability_necessary}
    \begin{cases}
        0 \le &\lambda_{\max}(H) \le \frac{2}{\eta}\\
        0 \le &\lambda_{\max}(\Sigma) \le \frac{1}{\eta} \sqrt{\frac{B(N-1)}{n -B}}
    \end{cases}
\end{equation}

The term $\lambda_{\max}(\Sigma)$, called non-uniformity, measures the mean squared deviation of curvature under sampling of mini-batches from $\mathcal{D}$. 

Thus, the choice of $\eta$ and $B$ affects reachability of critical points $\bm{\theta}^*$ under the dynamics of SGD. Particularly, the conditions in Equation~\ref{eq:appendix:stability_necessary} imply that large learning rates $\eta$ and small batch sizes $B$ will select critical points respectively with low curvature $\lambda_{\max}(H)$ and low non-uniformity $\lambda_{\max}(\Sigma)$. Hence, $\eta$ and $B$ control parameter space curvature around critical points attainable by the training dynamics and, via Theorem~\ref{thm:findings:curvature}, input-sensitivity.

In the next sections, we extend the empirical findings of section~\ref{sec:implications}.

\begin{figure*}[t]
    \centering
    \includegraphics[width=0.3\linewidth, trim={0cm 0cm 12cm 12cm}, clip]{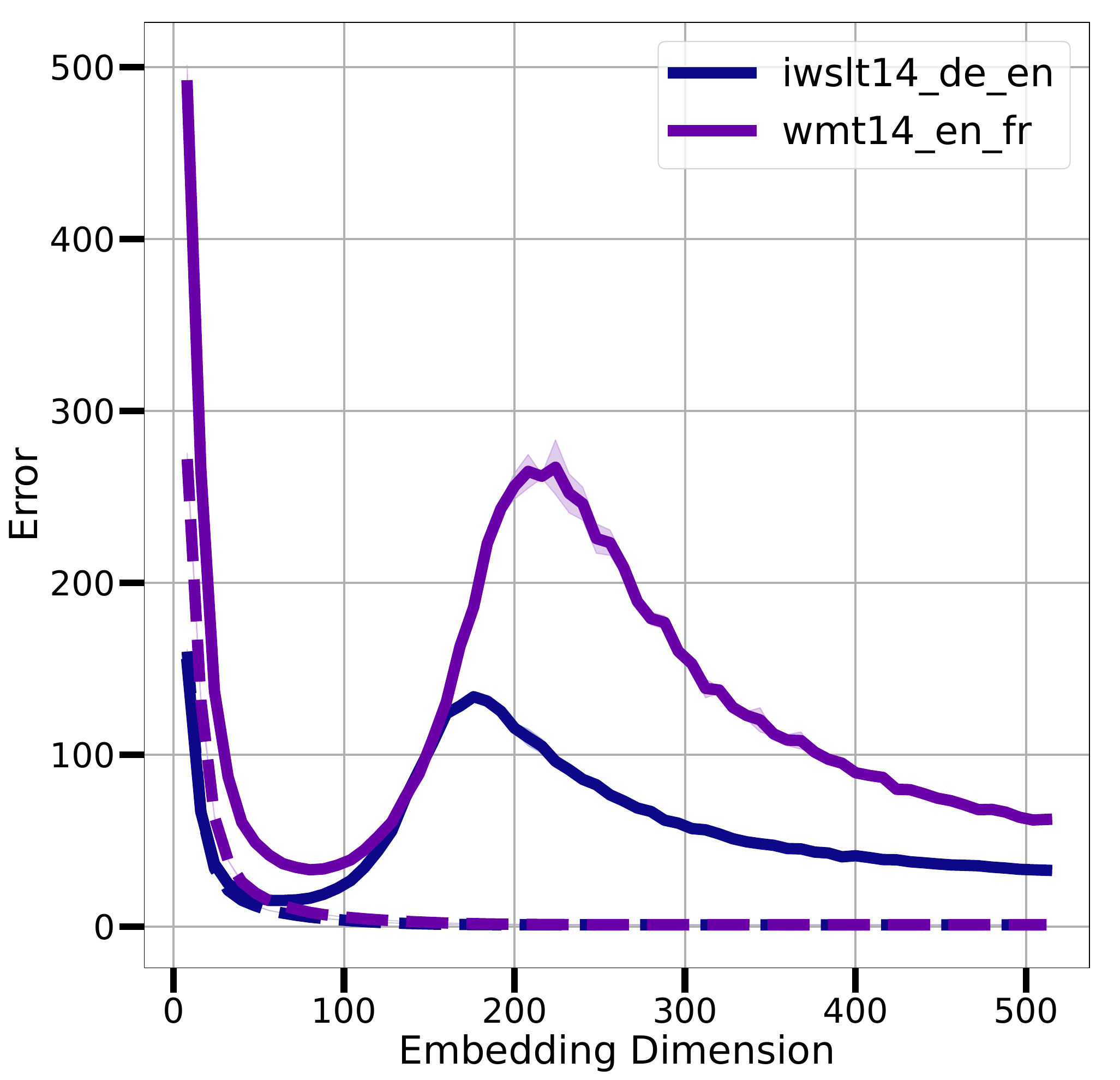}~
    \includegraphics[width=0.3\linewidth, trim={0cm 0cm 13cm 12cm}, clip]{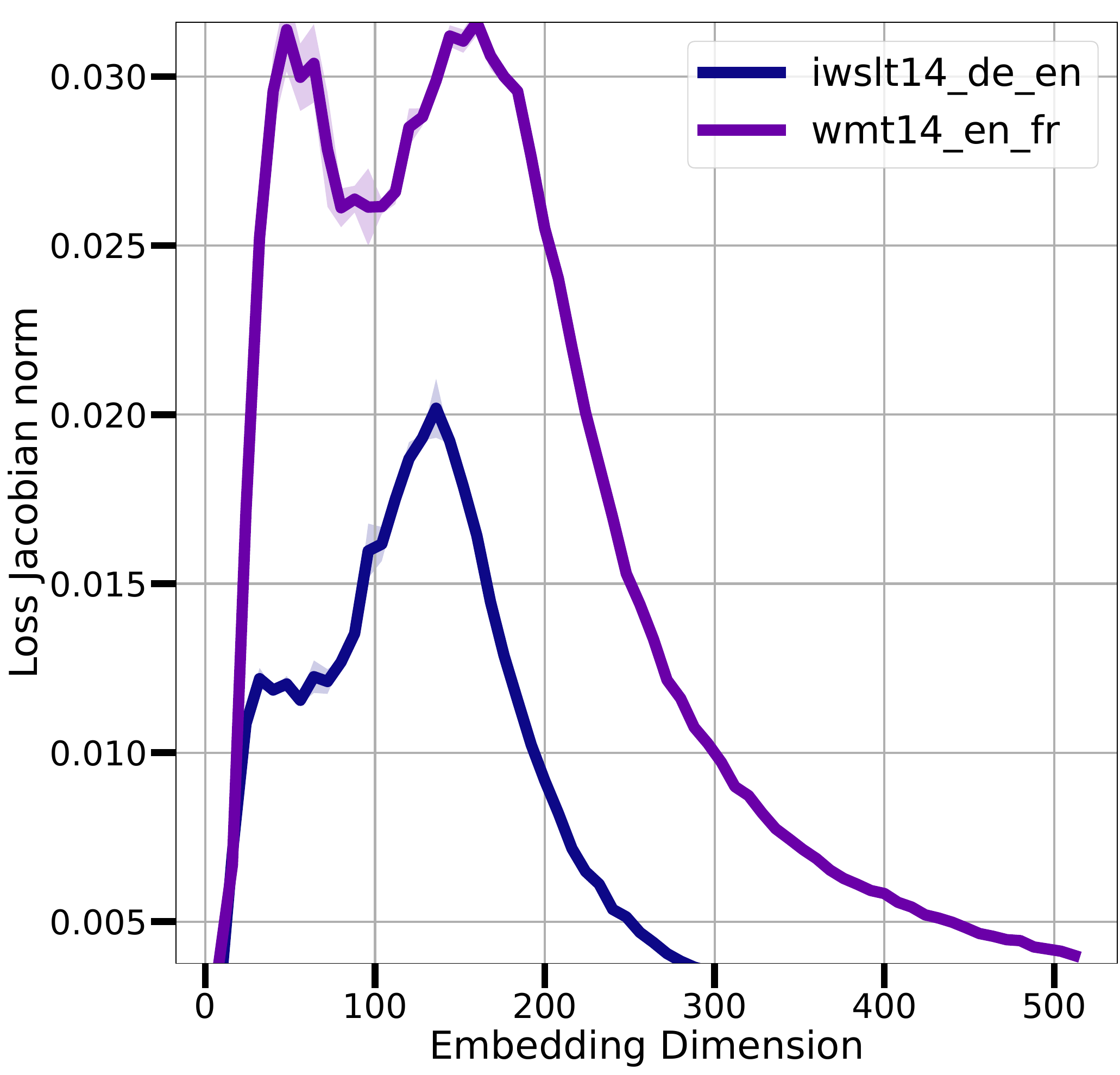}~
    \caption{\textbf{Input-space smoothness of Transformers}, as the embedding dimension and model width vary. Train error (dashed) and double descent for the test error for Transformers trained machine translation tasks (left) and input-space loss Jacobian norm (right).}
    \label{fig:appendix:transformers}
\end{figure*}

\subsection{Beyond Piece-wise Linear Networks}
\label{sec:appendix:transformers}

In this section, we extend our main finding beyond vision architectures and focus on natural language processing tasks. Specifically, we consider transformer architectures and train $8$-layer multi-head attention transformers~\citep{vaswani2017attention} on machine translation tasks, controlling the embedding dimension, as well as the width of hidden fully connected layers $\omega = 4h$. We report the test error in Figure~\ref{fig:appendix:transformers} (left). We compute Equation~\ref{eq:findings:operator} on $\nabla_\mathbf{x}\mathcal{L}$, where $\mathcal{L}$ is the per-token perplexity. We note that Equation~\ref{eq:findings:operator} can still be applied to the Jacobian $\nabla_\mathbf{x}\mathcal{L}$ -- which linearly approximates $\mathcal{L}$ at each point $\mathbf{x}$ -- and the expected operator norm should be intended as the Sobolev seminorm $\|\mathcal{L}\|_{\mathcal{D},1,2}$ of $\mathcal{L}$ on $\mathcal{D}$. Figure~\ref{fig:appendix:transformers} (right panel) extends our main finding, showing that $\nabla_{\mathbf{x}}\mathcal{L}$ depends non-monotonically on model size, peaking near the interpolation threshold, and extending our main result beyond vision architectures.

\subsection{Empirical Lipschitz Throughout Training}
\label{sec:appendix:epochwise}

In this section, we complement the results shown for selected model widths in Figure~\ref{fig:implications:epochwise}, by plotting the development of the empirical Lipschitz constant throughout training for all model sizes, and discuss its relationship to the test error. Extending our finding to additional model widths, Figure~\ref{fig:appendix:epochwise} shows that small models maintain a small empirical Lipschitz constant throughout training, while models near the interpolation threshold accumulate a large empirical Lipschitz constant after prolonged training. Finally, large models maintain a relatively low empirical Lipschitz constant, plateauing earlier as model size increases past the interpolating threshold.

\begin{figure}[t]
    \centering
    \includegraphics[width=0.3\linewidth]{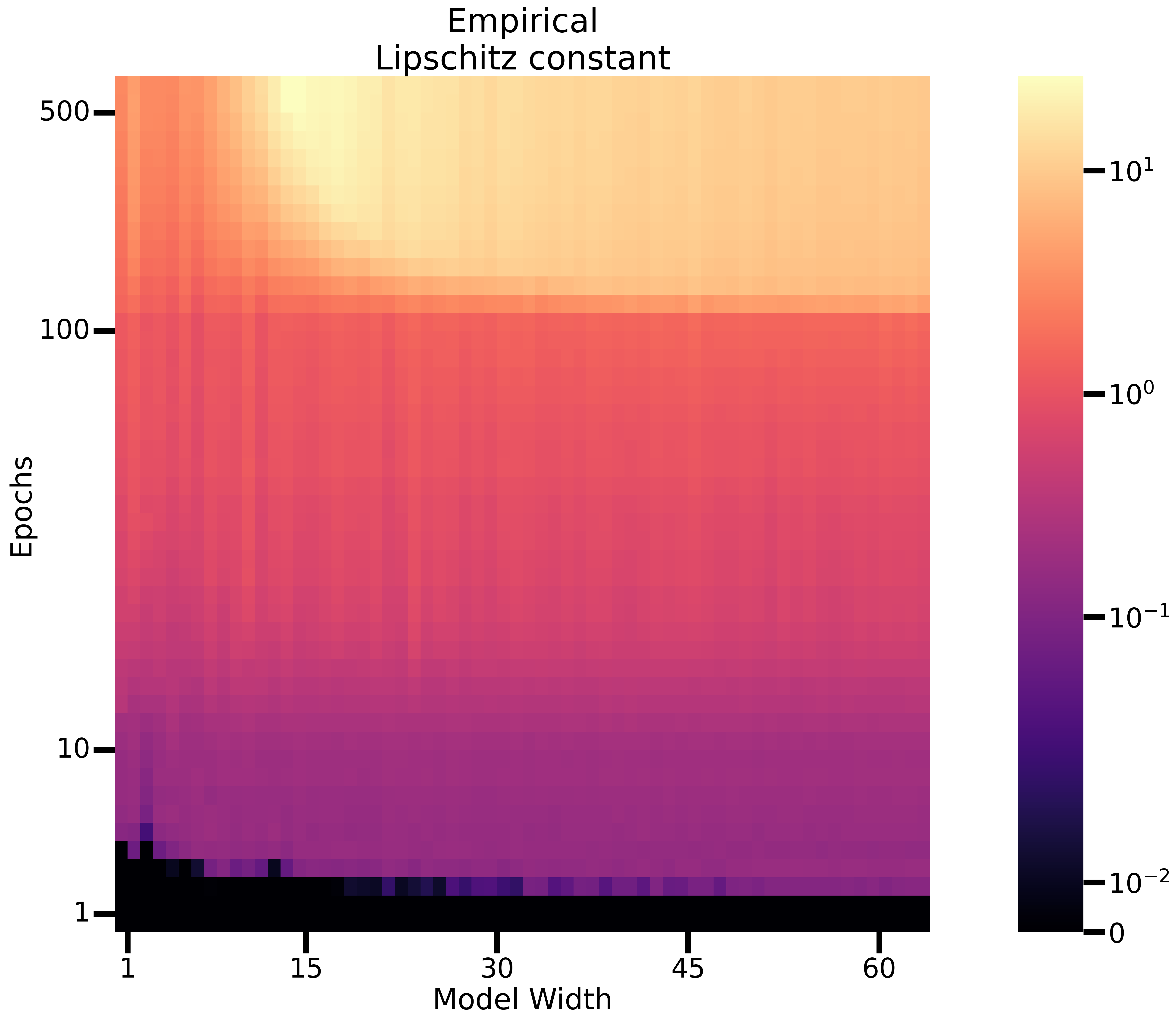}~
    \includegraphics[width=0.3\linewidth]{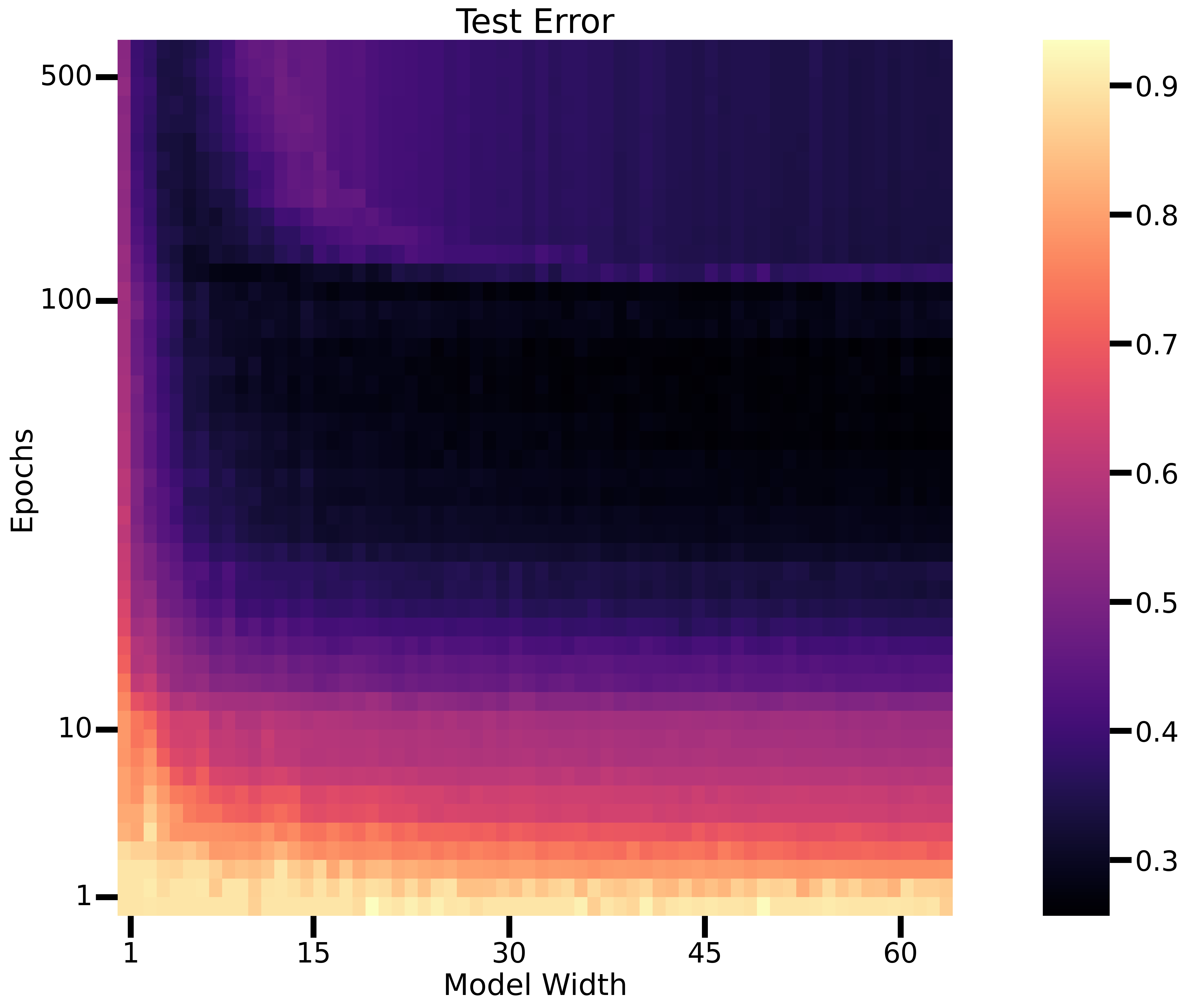} ~
    \includegraphics[width=0.3\linewidth]{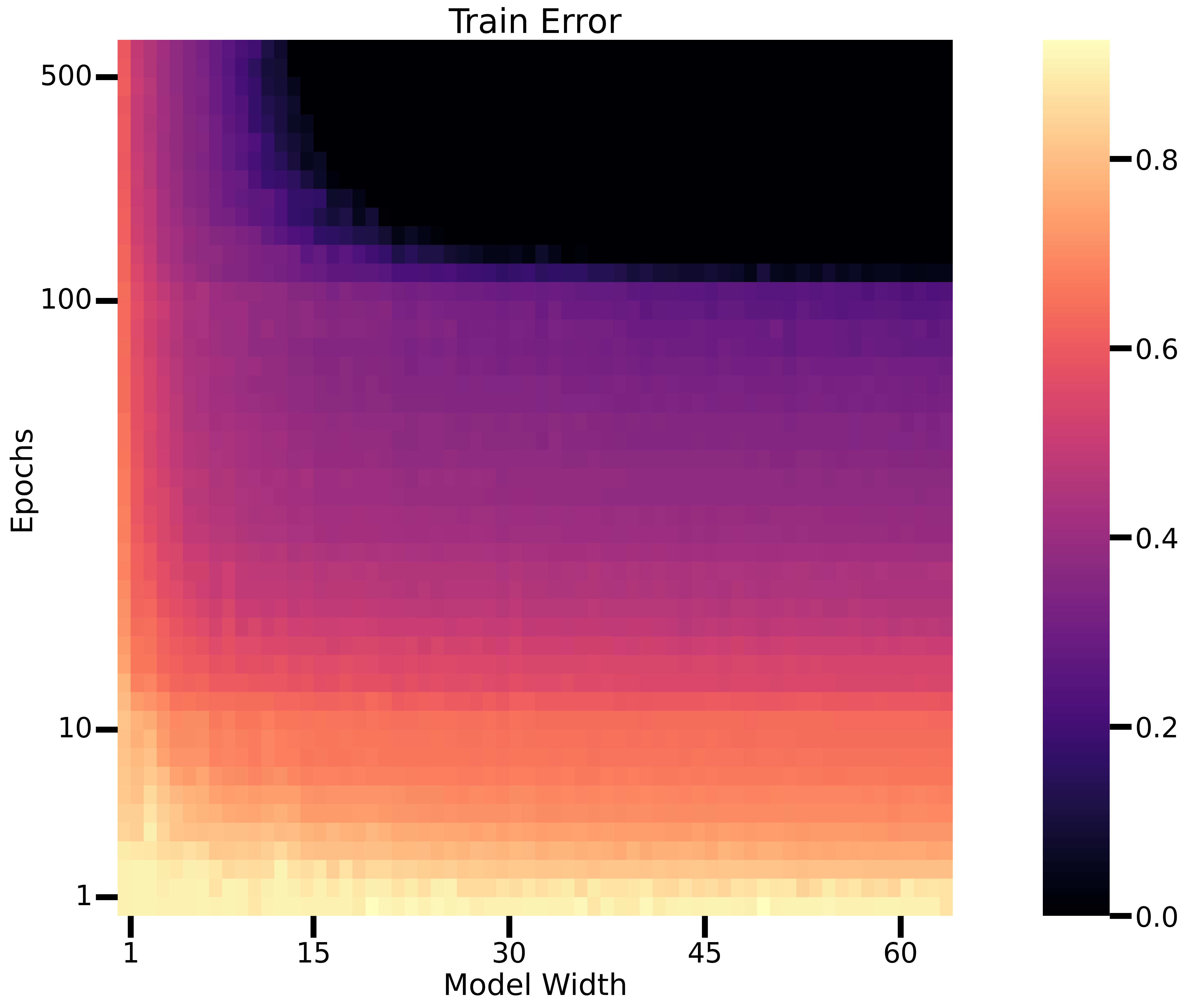} \\
    \includegraphics[width=0.3\linewidth]{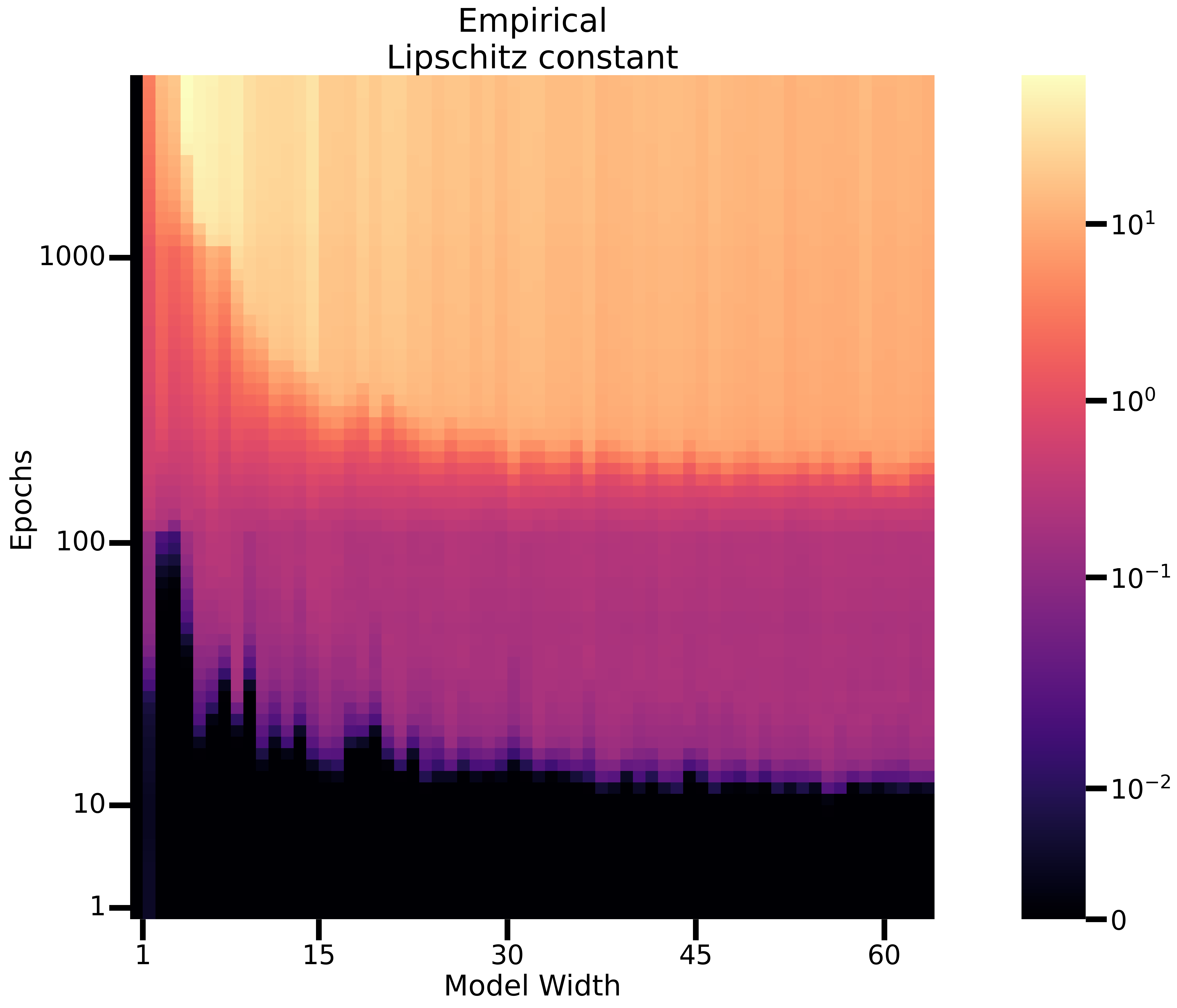}~
    \includegraphics[width=0.3\linewidth]{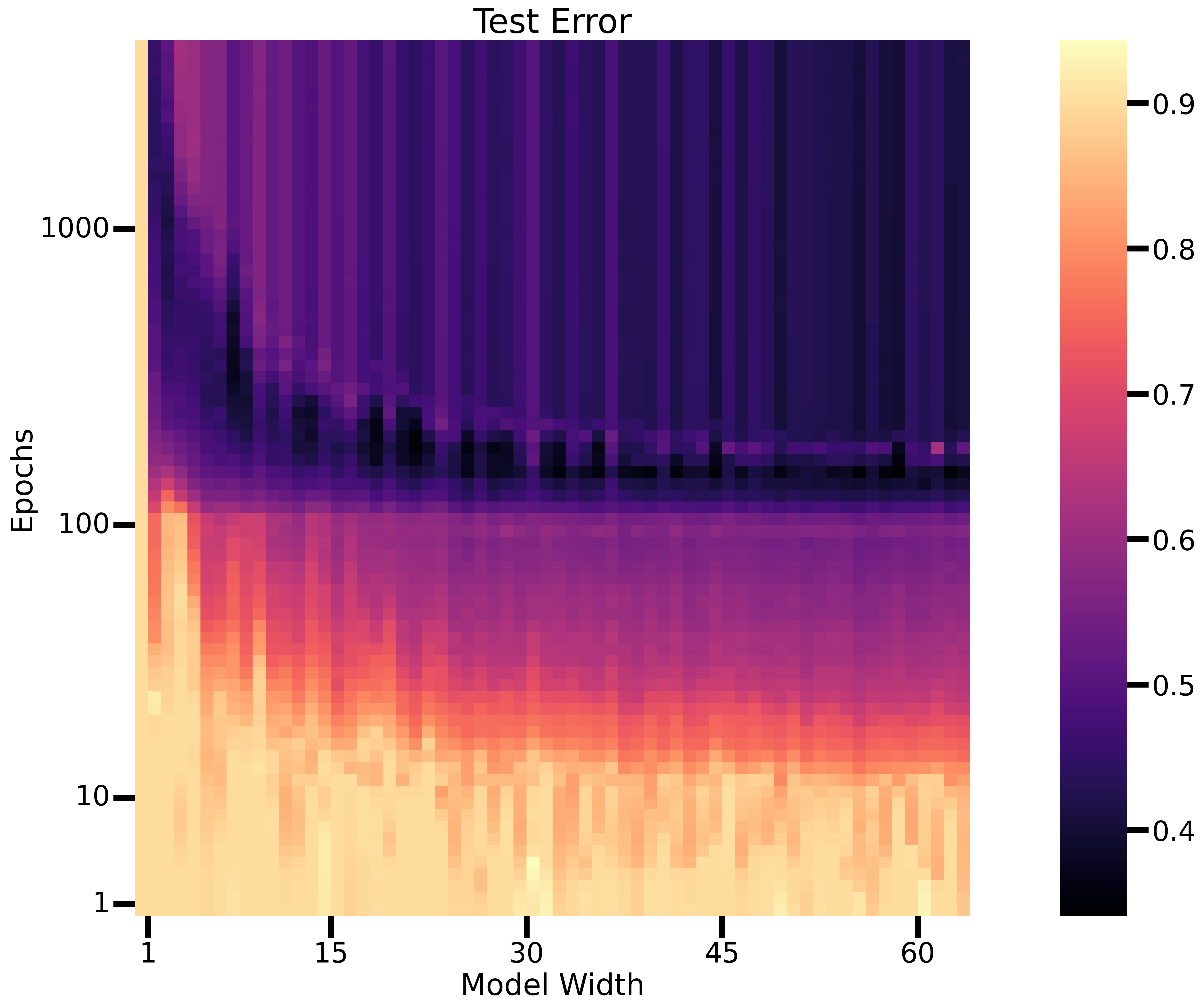} ~
    \includegraphics[width=0.3\linewidth]{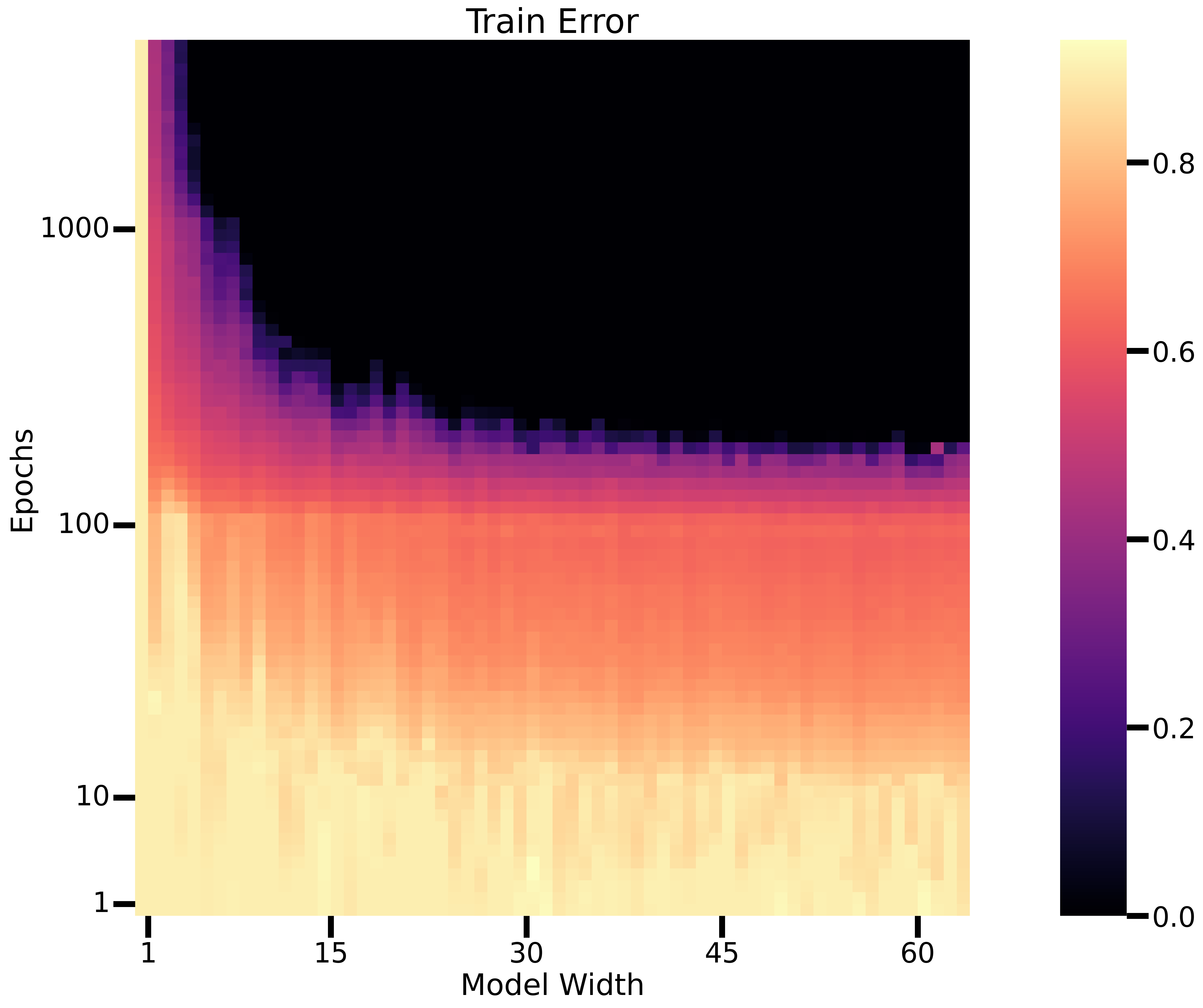}
    \caption{(Top left) \textbf{Empirical Lipschitz constant} (color) \textbf{as a function of training epochs} ($y-axis$) and model size ($x$-axis). (Top middle) \textbf{Test error} for ConvNets on CIFAR-10 with $20\%$ noisy training labels. (Top right) \textbf{Test error} for ConvNets on CIFAR-10 with $20\%$ noisy training labels. (Bottom) Analogous plots for ResNet18s trained on the same dataset.}
    \label{fig:appendix:epochwise}
\end{figure}

At the same time, with reference to the line plots in Figure~\ref{fig:implications:epochwise}, for all models the initial increase in empirical Lipschitz constant -- occurring during ``early'' training (up until epoch $100$ for ConvNets and $400$ for ResNets) -- is matched by a rapid decrease in test error. During mid-training (epoch $e$ $100 < e < 200$ for ConvNets, and $400 < e < 500$ for ResNets) the rate of increase of the Lipschitz constant changes according to model size. Small models plateau in their empirical Lipschitz constant, train and test error, and remain stable thereafter. Models near the interpolation threshold start slowly increasing the empirical Lipschitz constant as they slowly interpolate the training set, with corresponding increase in test error, showcasing the ``malign overfitting'' phenomenon~\citep{bartlett2020benign}, Strikingly, large models quickly interpolate the training set, causing relative increase in the empirical Lipschitz constant, inversely correlating with model size. Throughout this phase of ``accelerated interpolation'' the test error undergoes epoch-wise double descent~\citep{nakkiran2019deep}. Crucially, while for all models the empirical Lipschitz constant is monotonically increasing in epochs, the \textit{rate} at which the empirical Lipschitz constant grows correlates with \textit{epoch-wise} double descent for the test error. This observation suggests that tracking second order information of $\mathbf{f}_{\bm{\theta}}$ in input space may reveal important properties of interpolation. Indeed, input-space Hessian based measures~\citep{moosavi2019robustness,lejeune2019implicit} have been observed to correlate with model performance for fixed-sized models. Our observations suggest that input-space curvature may bear significance for understanding epoch-wise double descent. We leave this exciting direction to future work. 

\subsection{Validation of our Bound in the Interpolating Regime}
\label{sec:appendix:correlation}

\begin{figure*}[t]
    \centering
    \includegraphics[width=0.32\linewidth, trim={0cm 0cm 0cm 0cm}, clip]{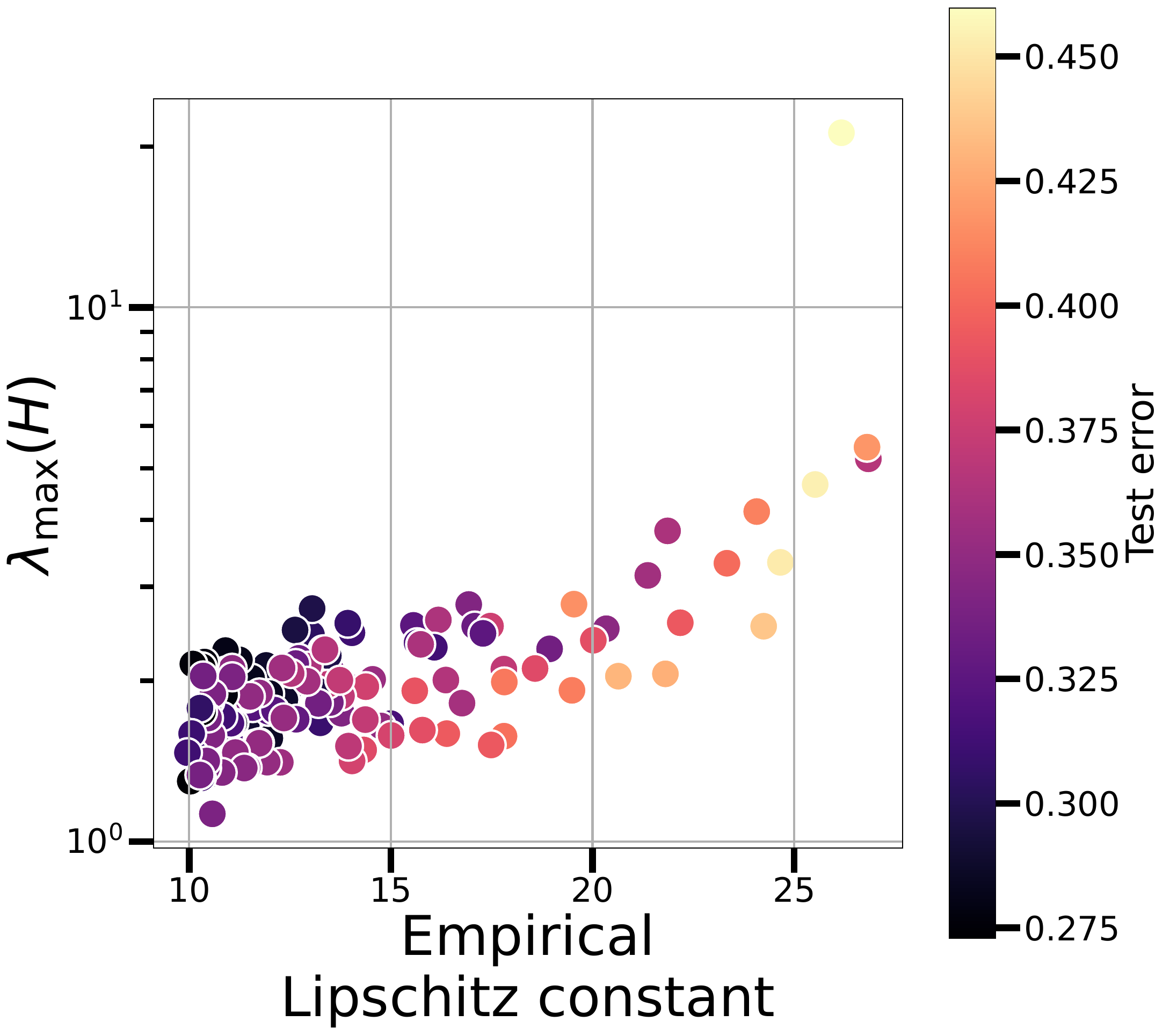}~
    \includegraphics[width=0.32\linewidth, trim={0cm 0cm 0cm 0cm}, clip]{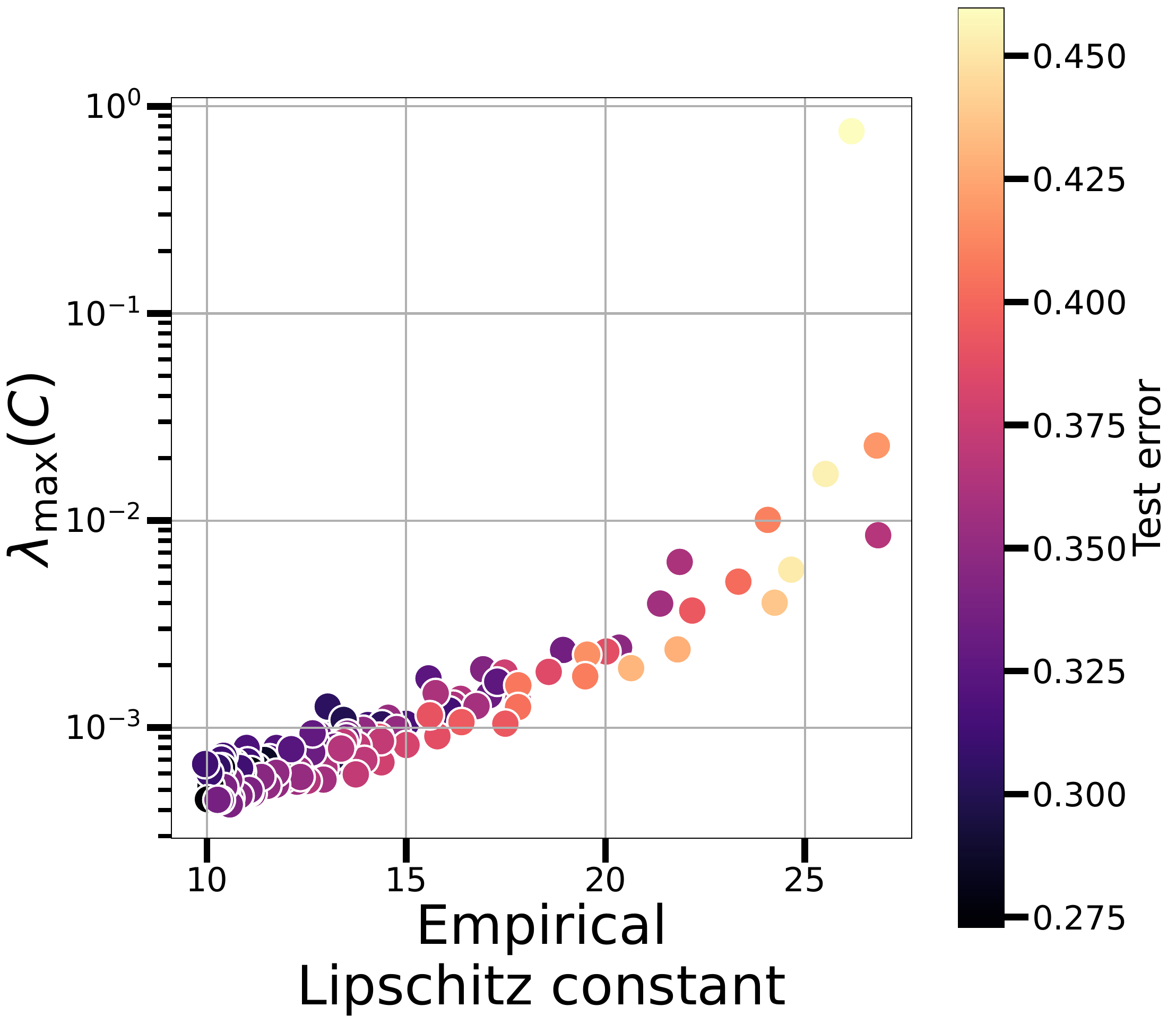}~
    \includegraphics[width=0.32\linewidth, trim={0cm 0cm 0cm 0cm}, clip]{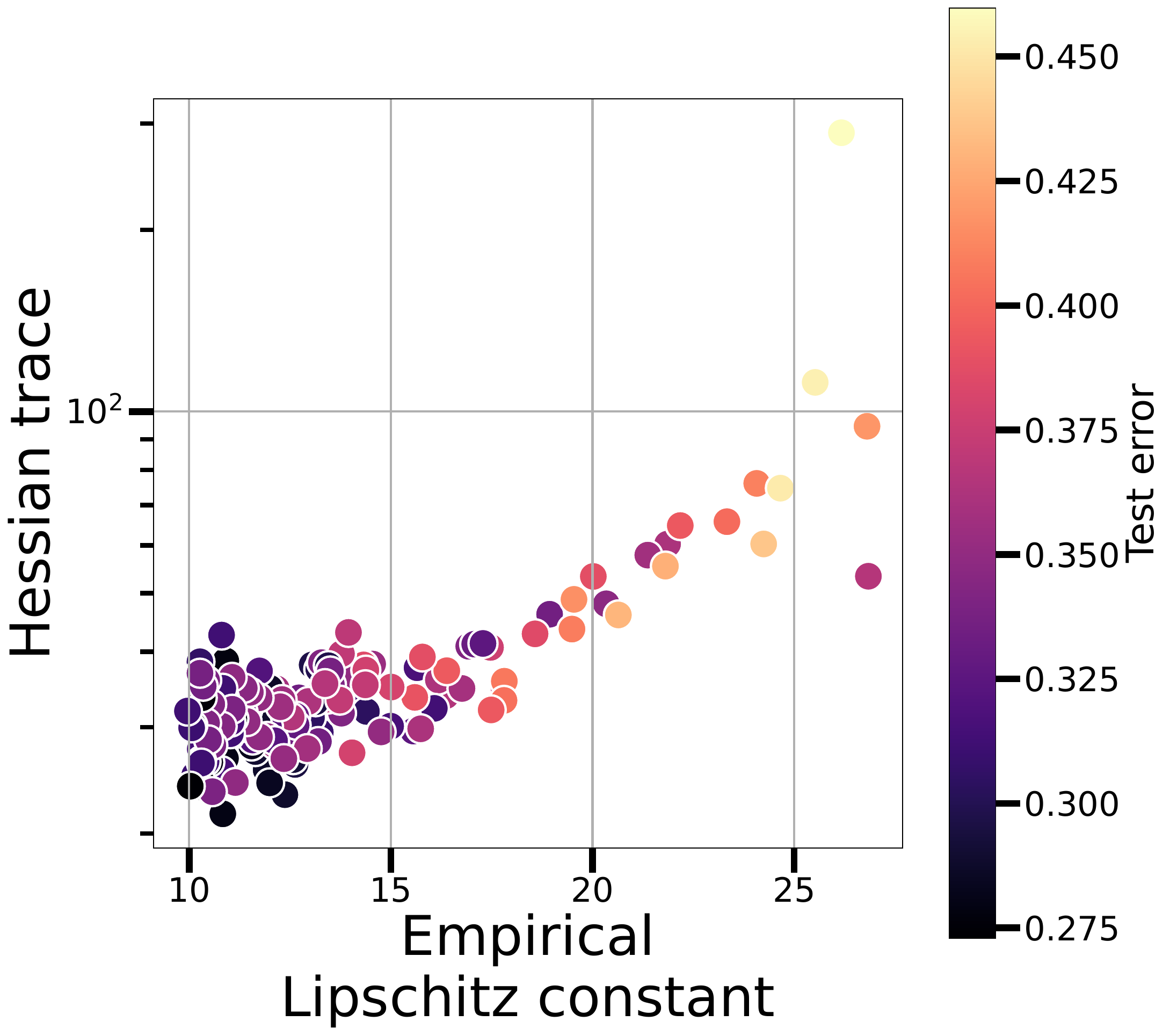}~ \\
    \includegraphics[width=0.32\linewidth, trim={0cm 0cm 0cm 0cm}, clip]{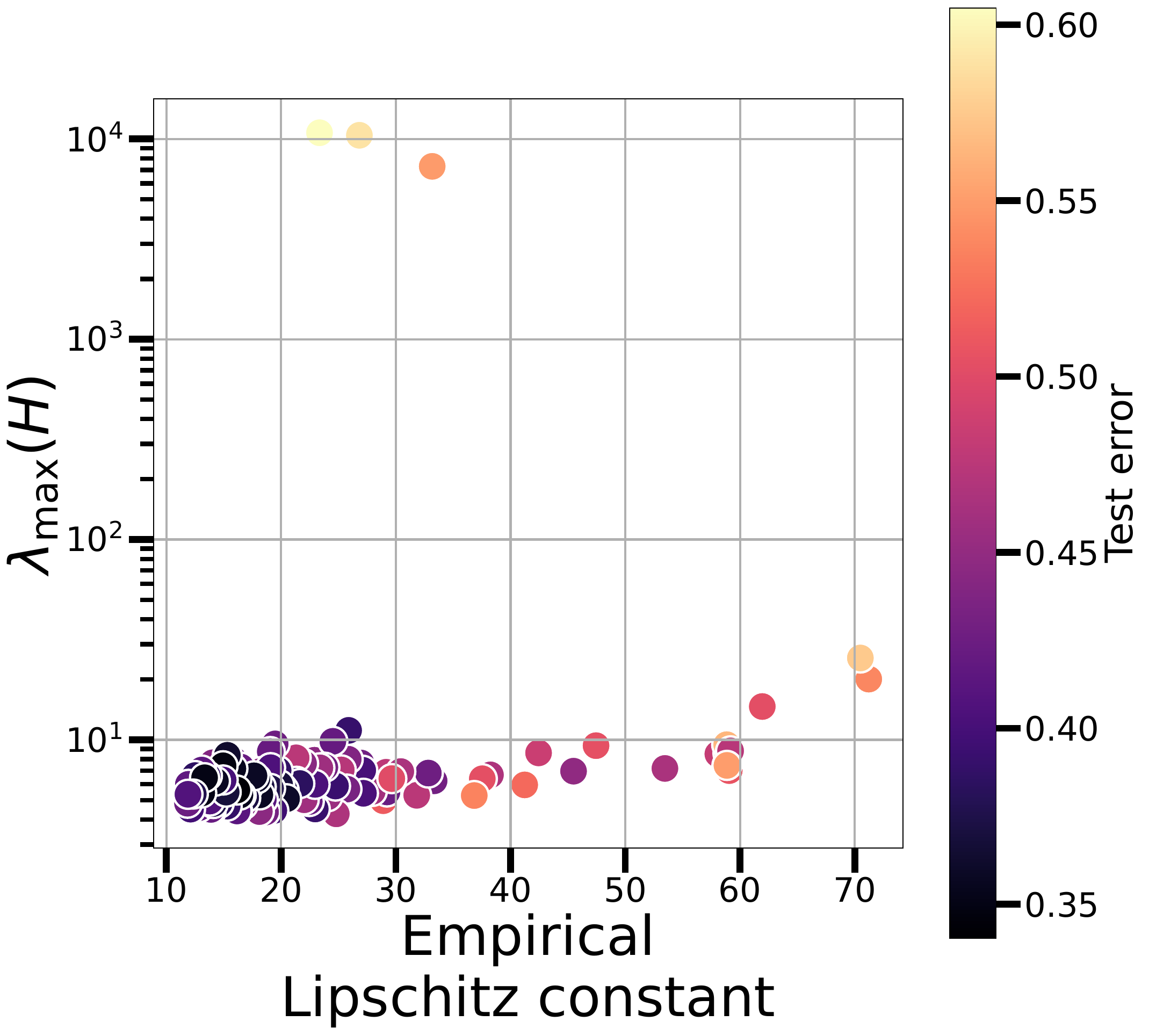}~
    \includegraphics[width=0.32\linewidth, trim={0cm 0cm 0cm 0cm}, clip]{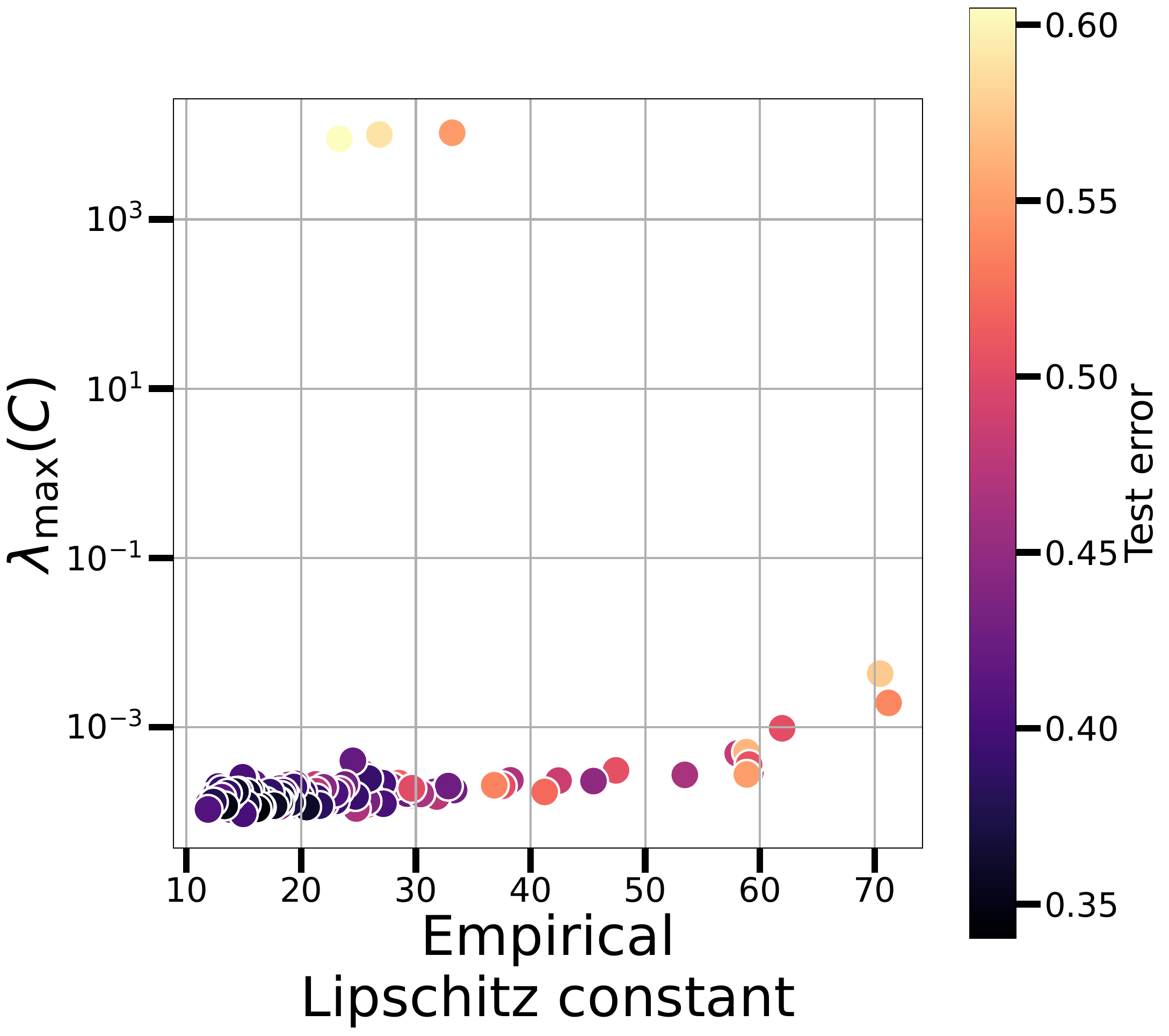}~
    \includegraphics[width=0.32\linewidth, trim={0cm 0cm 0cm 0cm}, clip]{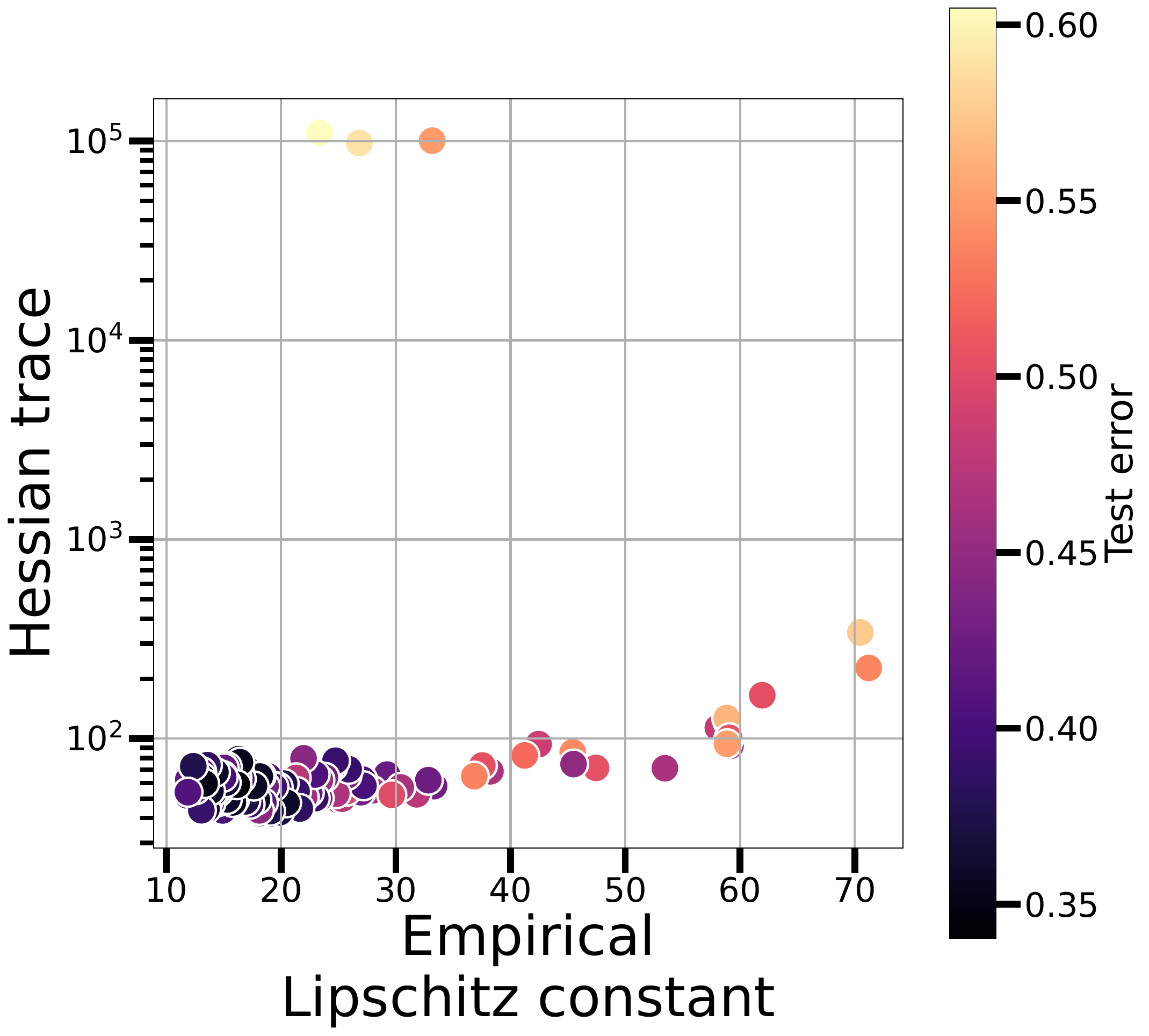}
    \caption{\textbf{Correlation between empirical Lipschitz constant and parameter-space curvature} in the interpolating regime. From left to right: maximum curvature (left), dominant noise-covariance eigenvalue (middle) and mean curvature (right), respectively for ConvNets trained on CIFAR-10 (top), and ResNets trained on CIFAR-10 (bottom). In all settings, mean and maximum parameter-space curvature strongly correlate with the empirical Lipschitz constant in the interpolating regime. Furthemore, models with higher empirical Lipschitz present higher mean and maximum curvatures, and incur in higher test error. All values are reported in $\log$-$y$ scale to better separate models.}
    \label{fig:findings:correlation}
\end{figure*}

Figure~\ref{fig:findings:correlation} summarizes our main findings, showing a strong correlation between the empirical Lipschitz constant and maximum parameter-space curvature of the loss landscape, mean parameter-space curvature, as well as the first principal component of gradient noise, with networks with large empirical Lipschitz constant incurring in high test error.

\section{Generating Random Validation Data}
\label{sec:appendix:random}

To generate random validation data for the experiments reported in Figure~\ref{fig:implications:ablation}, we define several distributions over RGB pixels, and sample each pixel independently. We consider the following distributions:
\begin{itemize}
    \item $\mathbf{x}_n \sim \mathcal{U}\big({[}\bm{\mu}_{\text{CIFAR}} - \bm{\sigma}_{\text{CIFAR}}, \bm{\mu}_{\text{CIFAR}} + \bm{\sigma}_{\text{CIFAR}}{]}\big)$ pixel-wise
    \item $\mathbf{x}_n \sim \mathcal{N}\big({[}\bm{\mu}_{\text{CIFAR}}, \mathcal{I}_3 \bm{\sigma}_{\text{CIFAR}}{]}\big)$ pixel-wise
    \item $\mathbf{x}_n \sim \mathcal{U}\big(S_{d-1}\big)$ (pixel-wise) hypersphere
    \item $\mathbf{x}_n + \bm{\epsilon}_n,$ with $\bm{\epsilon}$ strong random jitter
\end{itemize}
where $\bm{\mu}_{\text{CIFAR}}$ and $\bm{\sigma}_{\text{CIFAR}}$ respectively denote the per-channel mean and standard deviation computed on the CIFAR-10 training set. For each distribution, we generate a validation set of $50$k i.i.d.\ samples, and probe networks trained on the standard CIFAR-10 with $20\%$ corrupted labels.

For reference, we also plot the empirical Lipschitz constant estimated on the CIFAR-10 train and test split. For both out-of-sample and in-sample validation datasets, it can be observed how the empirical Lipschitz constant remains bounded, and closely follows the double descent trend for the test error (c.f.r.\ Figure~\ref{fig:findings:lipschitz}). Remarkably, the empirical Lipschitz constant on random validation data closely matches the one estimated on the training set, supporting the hypothesis of globally bounded function complexity.

\section{Proofs}
\label{sec:appendix:proofs}

\input{sections/proofs}

%% file: sections/proofs.tex
In this section, we provide proofs for the formal statements presented in section~\ref{sec:findings}. We begin by deriving results on boundedness of model function input-space gradients via parameter-space gradients. Then, we prove results of section~\ref{sec:findings:curvature}. 

\subsection{Duality of Linear Layers}

We begin by providing a general form of Theorem~\ref{thm:findings:sobolev}.

\sobolev*
\begin{proof}
We recall that, by duality of inputs and weights in linear transformations, the partial derivatives $\frac{\partial\mathbf{f}}{\partial\mathbf{x}^{\ell-1}}$ and $\frac{\partial\mathbf{f}}{\partial\bm{\theta}^\ell}$ w.r.t.\ any layer of the form $\mathbf{x}^\ell = \phi(\bm{\theta}^\ell\mathbf{x}^{\ell-1})$ are tied by the upstream gradient $\frac{\partial \mathbf{f}}{\partial (\bm{\theta}^\ell\mathbf{x}^{\ell -1})}$. Indeed, by the chain rule
\begin{equation}
\label{eq:appendix:duality}
\begin{cases}
    \frac{\partial \mathbf{f}}{\partial \mathbf{x}^{\ell -1}} &= \frac{\partial \mathbf{f}}{\partial(\bm{\theta}^\ell\mathbf{x}^{\ell -1} + \mathbf{b}^\ell)}\bm{\theta}^\ell \\
    \frac{\partial \mathbf{f}}{\partial \bm{\theta}^\ell} &= \frac{\partial \mathbf{f}^T}{\partial(\bm{\theta}^\ell\mathbf{x}^{\ell -1} + \mathbf{b}^\ell)} {\mathbf{x}^{\ell -1}}^T
\end{cases}
\end{equation}

Let $\mathbf{f}: \mathbb{R}^d \times \mathbb{R}^p \to \mathbb{R}$ be an arbitrary function composing linear layers with (optional) nonlinearities $\phi: \mathbb{R} \to \mathbb{R}$, that are differentiable a.e. Furthermore, let $\mathbb{R}^{d_\ell}$ denote the codomain of layer $\ell$, i.e. $\mathbf{x}^\ell \in \mathbb{R}^{d_\ell}$.

Combining the two conditions in Equation~\ref{eq:appendix:duality} gives 
\begin{equation}
\label{eq:appendix:chain}
\begin{aligned}
    \frac{\partial \mathbf{f}}{\partial \bm{\theta}^\ell}\mathbf{x}^{\ell -1} &= \frac{\partial \mathbf{f}^T}{\partial (\bm{\theta}^\ell\mathbf{x}^{\ell -1} + \mathbf{b}^\ell)}\|\mathbf{x}^{\ell-1}\|_2^2 \\
    \frac{\partial \mathbf{f}}{\partial \mathbf{x}^{\ell-1}} &= \frac{{\mathbf{x}^{\ell -1}}^T}{\|\mathbf{x}^{\ell -1}\|_2^2}\frac{\partial \mathbf{f}^T}{\partial \bm{\theta}^\ell}\bm{\theta}^\ell \\
    \left\|\frac{\partial \mathbf{f}}{\partial \mathbf{x}^{\ell-1}}\right\|_2^2 &\le \left\|\frac{{\mathbf{x}^{\ell -1}}^T}{\|\mathbf{x}^{\ell -1}\|_2^2}\right\|_2^2\left\|\frac{\partial \mathbf{f}^T}{\partial \bm{\theta}^\ell}\right\|_2^2\left\|\bm{\theta}^\ell\right\|_2^2 \\
    \left\|\frac{\partial \mathbf{f}}{\partial \mathbf{x}^{\ell-1}}\right\|_2^2 \frac{\left\|\mathbf{x}^{\ell -1}\right\|_2^2}{\left\|\bm{\theta}^\ell\right\|_2^2} &\le \left\|\frac{\partial \mathbf{f}}{\partial \bm{\theta}^\ell}\right\|_2^2 \\
\end{aligned}
\end{equation}

Given a set of data points $\{\mathbf{x}_1, \ldots, \mathbf{x}_N\} \subset \mathbb{R}^d$, with corresponding activations $\{\mathbf{x}_1^\ell, \ldots, \mathbf{x}_n^\ell \} \subset \mathbb{R}^{d_\ell}$ then:
\vspace{.5em}
\begin{equation}
\label{eq:appendix:chain_finite}
\begin{aligned}
    \left\|\frac{\partial \mathbf{f}}{\partial \mathbf{x}^{\ell-1}}\right\|_2^2 \frac{\min\limits_n\left\|\mathbf{x}^{\ell -1}(\mathbf{x}_n)\right\|_2^2}{\left\|\bm{\theta}^\ell\right\|_2^2} &\le \left\|\frac{\partial \mathbf{f}}{\partial \bm{\theta}^\ell}\right\|_2^2 \\
\end{aligned}
\end{equation}

Particularly, given the training set $\mathcal{D}$, applying Equation~\ref{eq:appendix:chain_finite} to the gradients $\nabla_\mathbf{x}\mathbf{f}$ and $\nabla_{\bm{\theta}}\mathbf{f}$, and taking the expectation over $\mathcal{D}$ on both sides concludes the proof.

\end{proof}

We note that, while a similar bound was observed in \citet{ma2021linear} for the first layer gradients, the authors propose to bound $\nabla_{\bm{\theta}}\mathbf{f}$ via a uniform bound that linearly depends on model size, and thus cannot capture double descent. In this work, we improve upon their bounds,generalizing the result of~\citet{ma2021linear} to any layer beyond the first, and by explicitly studying Equation~\ref{eq:findings:sobolev_loss} in connection to parameter-space dynamics and double descent.

\subsection{Extension to loss functions}
\label{sec:appendix:sobolev_loss}

For losses $\mathcal{L}: \mathbb{R}^p \times \mathbb{R}^d \times \mathcal{Y} \to \mathbb{R}^+$ of the exponential family~\citep{brebisson2015exploration} like mean squared error and cross entropy, the following corollary holds.

\sobolevloss*

\begin{proof}
For each sample $(\mathbf{x}_n, y_n)$, the gradient $\frac{\partial \mathcal{L}}{\partial \mathbf{f}}$ takes the form $\mathbf{p}_n - \mathbf{e}_{y_n}$. For crossentropy, $\mathbf{p}_n$ denotes the softmax normalized logits, and $\mathbf{e}_{y_n}$ the one-hot encoded label $y_n$. For mean squared error, $\mathbf{p}_n = \mathbf{f}_{\bm{\theta}}(\mathbf{x}_n)$. 

When composing the loss with a model $\mathbf{f}$, we have 
\begin{equation}
\label{eq:appendix:residual}
\begin{aligned}
 \nabla_{\bm{\theta}} (\mathcal{L} \circ \mathbf{f}) &= (\mathbf{p}_n - \mathbf{e}_{y_n})\nabla_{\bm{\theta}}\mathbf{f} \\
 \nabla_{\mathbf{x}} (\mathcal{L} \circ \mathbf{f}) &= (\mathbf{p}_n - \mathbf{e}_{y_n})\nabla_{\mathbf{x}}\mathbf{f}
\end{aligned}
\end{equation}
Applying Theorem~\ref{thm:findings:sobolev} trivially concludes the proof.
\end{proof}

Importantly, the term $\|\mathbf{p}_n - \mathbf{e}_{y_n}\|$ is inversely proportional to the model confidence~\citep{brier1950verification} $\sigma = 1 - \frac{1}{N}\sum\limits_{n = 1}^N \|\mathbf{p}_n - \mathbf{e}_{y_n}\|$, which generally saturates for large models, typically yielding high confidence predictions at convergence. In Figure~\ref{fig:appendix:confidence} we empirically study how the quantity is affected by model size, to understand its impact on the bounds presented throughout section~\ref{sec:findings}. Model confidence is observed to monotonically depend on model size. This highlights the fact that the double descent trends observed throughout this paper are to be attributed to the model function, as shown throughout our experiments for the empirical Lipschitz constant.

\begin{figure}[t]
    \centering
    \includegraphics[width=0.3\linewidth, trim={0cm 0cm 13cm 12cm}, clip]{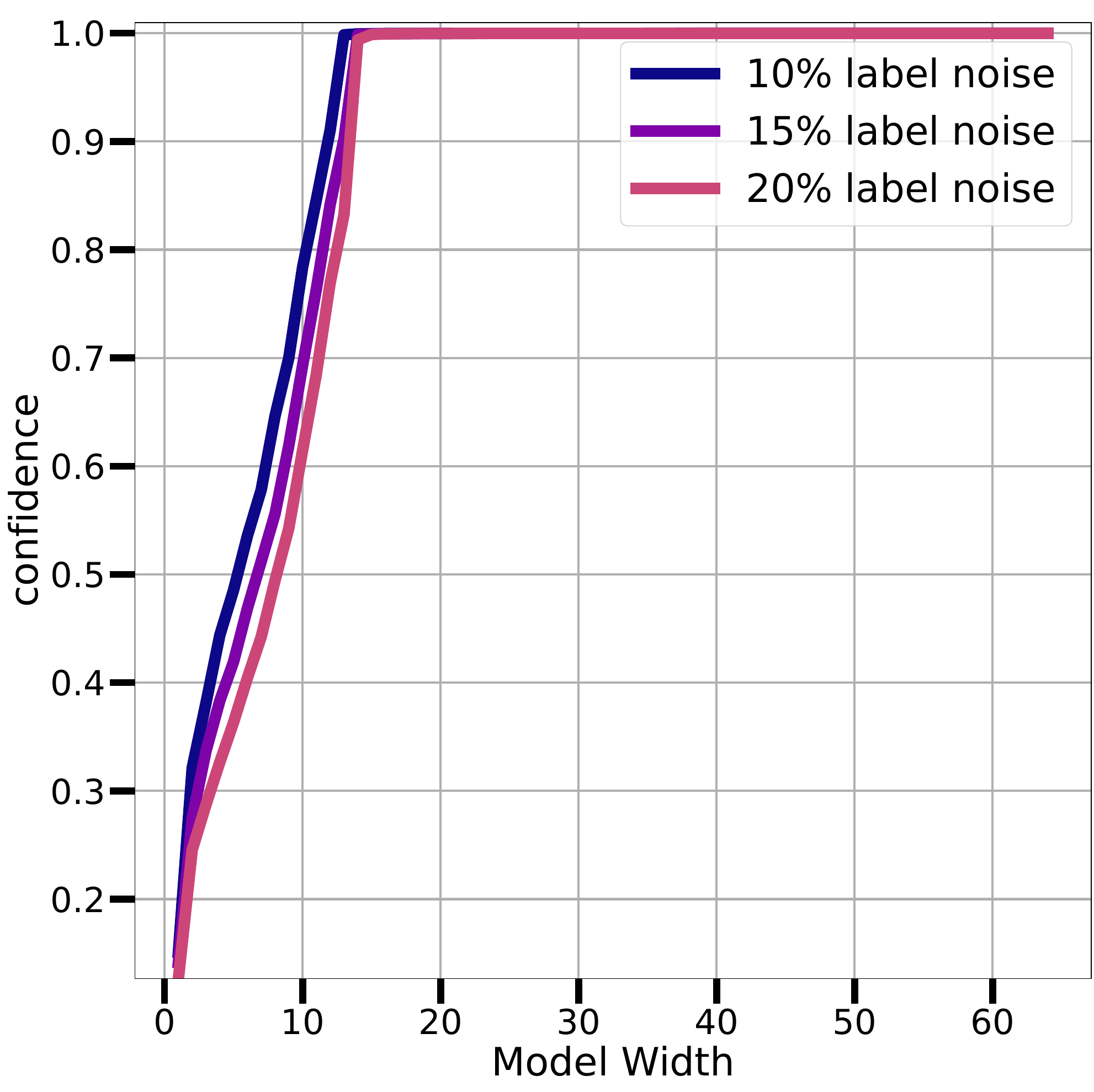}~
    \includegraphics[width=0.3\linewidth, trim={0cm 0cm 13cm 12cm}, clip]{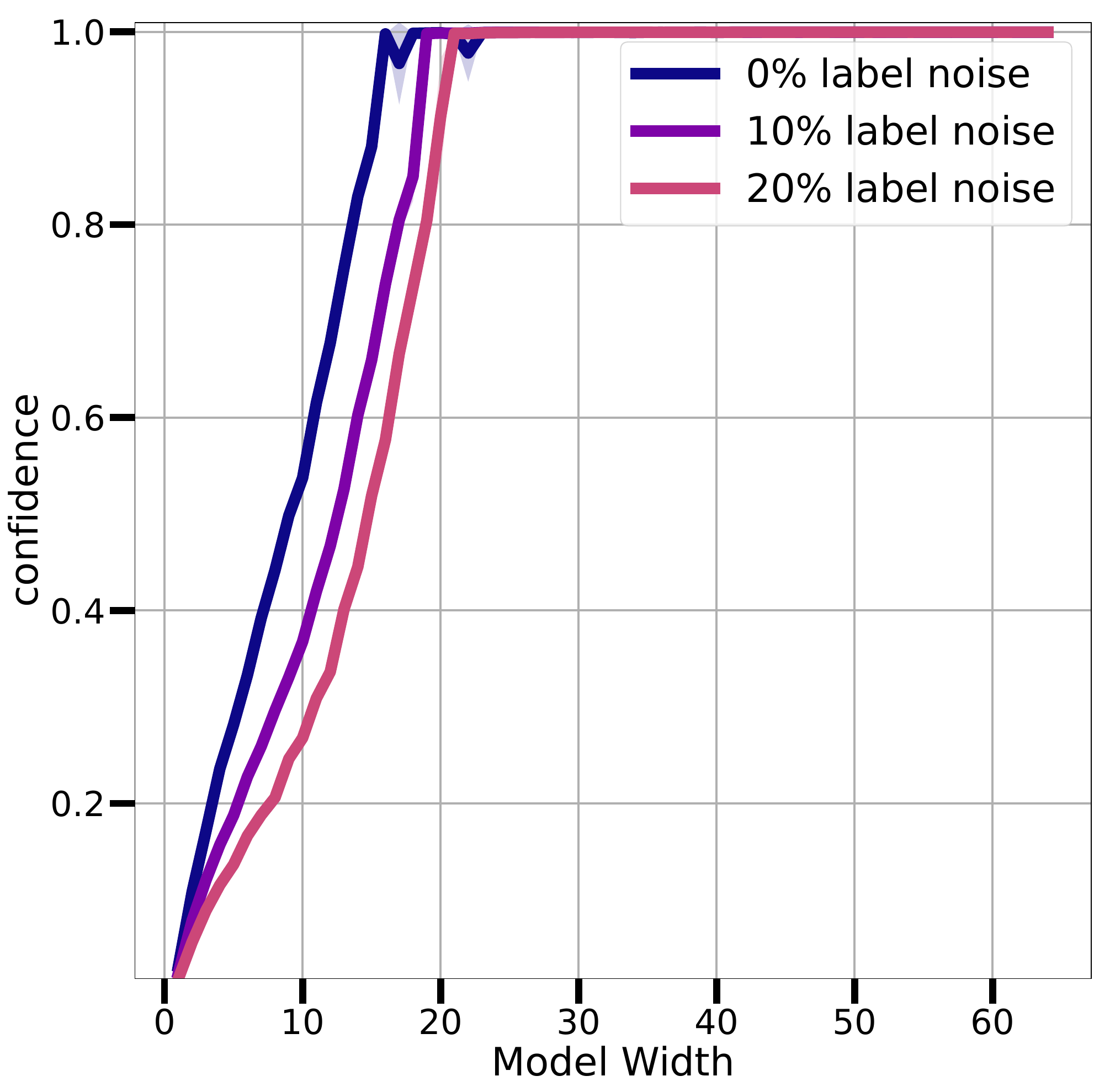}~
    \includegraphics[width=0.3\linewidth, trim={0cm 0cm 13cm 12cm}, clip]{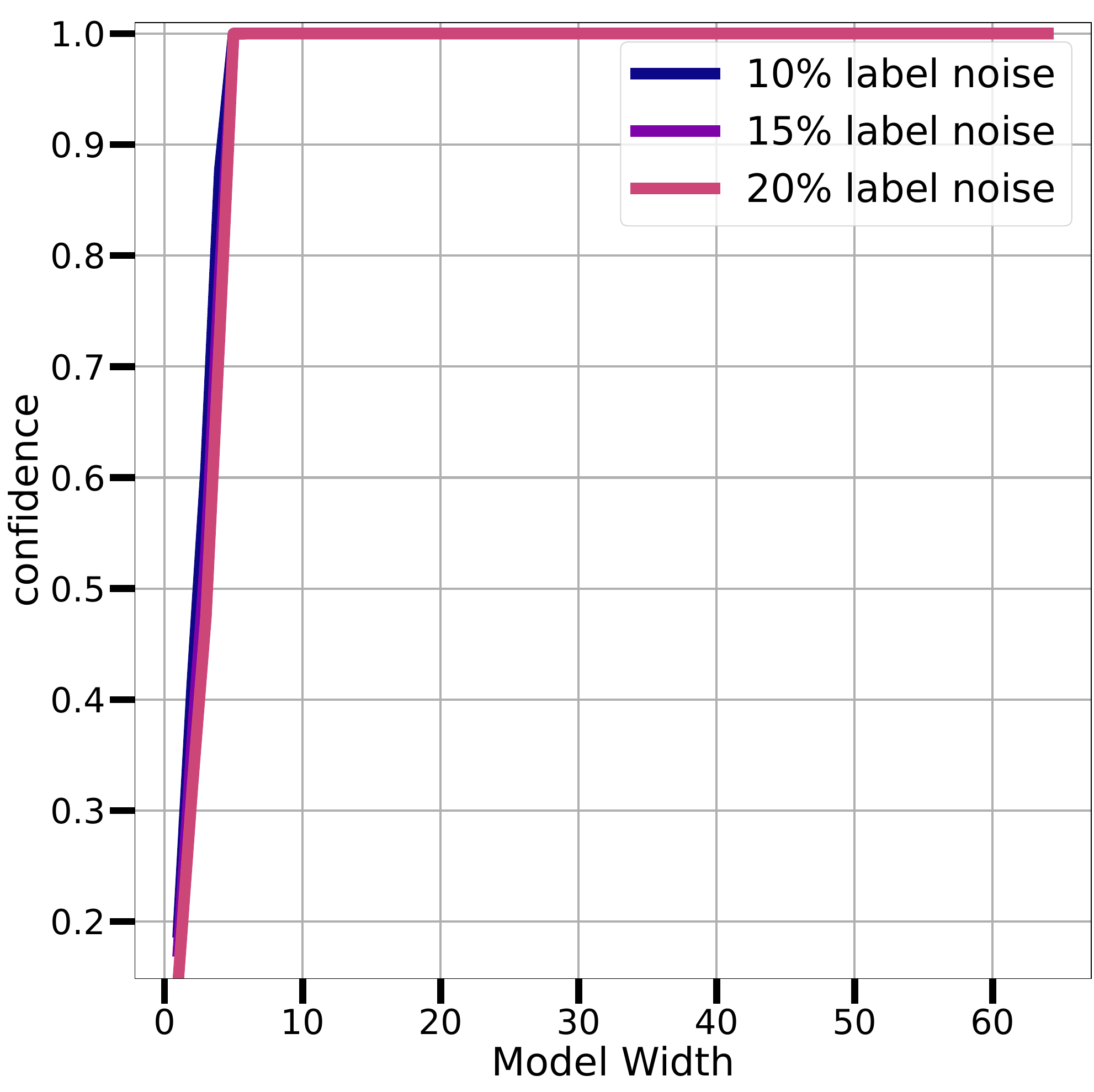}
    \caption{\textbf{Prediction confidence} as a function of model size, for ConvNets trained on CIFAR-10 (left), CIFAR-100 (middle) and ResNet18s trained on CIFAR-10. For all experimental settings, the model confidence monotonically depends on model size. By Equation~\ref{eq:appendix:residual}, this confirms that the non-monotonic trends reported in this work are caused by the model function $\mathbf{f}$.}
    \label{fig:appendix:confidence}
\end{figure}

Next, we provide proofs for section~\ref{sec:findings:curvature}.

\subsection{Connection to Parameter-Space Curvature}
\label{sec:appendix:curvature}

In this section, we prove formal statements connecting Theorem~\ref{thm:findings:sobolev_loss} to the dynamics of SGD in proximity of a critical point $\bm{\theta}^*$. For our proofs, we use the mean square error $\mathbb{E}_\mathcal{D}\mathcal{L} = \frac{1}{2N}\sum\limits_{n=1}^N (f_{\bm{\theta}}(\mathbf{x}_n) - y_n)^2$, and adopt a recent model of stochastic noise proposed by~\citet{liu2021noise}. The crux of the proof of Theorem~\ref{thm:findings:curvature} is bounding $\mathbb{E}_{\mathcal{D}}\|\nabla_{\bm{\theta}}\mathcal{L}\|$ with $\trace{(H)}$, which we can later connect to the noise uncentered covariance $S$.

\curvature*
\begin{proof}

The proof is broken down in two parts. First, we write out explicitly the expected Hessian $H$ of $\mathcal{L}$.

\begin{equation}
\label{eq:appendix:mse_hessian}
    H = \frac{1}{N}\sum\limits_{n=1}^N\frac{\partial^2}{\partial\bm{\theta}\partial{\bm{\theta}}^T} \mathcal{L}_n = \frac{1}{N}\sum\limits_{n=1}^N \mathcal{L}_n''\nabla_{\bm{\theta}}\mathbf{f}_n^T\nabla_{\bm{\theta}}\mathbf{f}_n  + \frac{1}{N}\sum\limits_{n=1}^N \mathcal{L}_n'\frac{\partial^2}{\partial\bm{\theta}\partial{\bm{\theta}}^T}\mathbf{f}_n
\end{equation}
with $\mathbf{f}_n := \mathbf{f}(\mathbf{x}_n,\bm{\theta})$, for $n = 1, \ldots, N$.

By noting that $\mathcal{L}''_n = 1, \forall n$, and that $\mathcal{L}'_n \propto \mathcal{L}_n \to 0$ as $\|\bm{\theta} - \bm{\theta}^*\|^2 \to 0$ for interpolating models, the expected loss Hessian amounts to the cross term
\begin{equation}
\label{eq:appendix:mse_outer}
    H = \frac{1}{N}\sum\limits_{n=1}^N\nabla_{\bm{\theta}}\mathbf{f}_n^T\nabla_{\bm{\theta}}\mathbf{f}_n + \mathcal{O}(\mathcal{L}(\bm{\theta}))
\end{equation}

Next, we connect $\|\nabla_{\bm{\theta}}\mathcal{L}\|_2^2$ to $H$. We note that $\frac{1}{N}\sum\limits_{n=1}^N\nabla_{\bm{\theta}}\mathcal{L} = \mathcal{L}'_n\nabla_{\bm{\theta}}\mathbf{f}_n$. Then, by definition of norm:
\begin{equation}
\label{eq:appendix:norm}
\begin{aligned}
    \mathbb{E}_\mathcal{D}\|\nabla_{\bm{\theta}}\mathcal{L}\|_2^2 &= \mathbb{E}_\mathcal{D}\nabla_{\bm{\theta}}\mathcal{L} \nabla_{\bm{\theta}}\mathcal{L}^T \\
                                            &= \mathbb{E}_\mathcal{D}\trace{(\nabla_{\bm{\theta}}\mathcal{L}^T \nabla_{\bm{\theta}}\mathcal{L})} \\
                                            &= \trace{\big(\frac{1}{N}\sum\limits_{n=1}^N{[}{\mathcal{L}_n'}^2~\nabla_{\bm{\theta}}\mathbf{f}_n^T \nabla_{\bm{\theta}}\mathbf{f}_n{]}  \big)} \\
                                            &\le 2(\max\limits_{1 \le n \le N}{\mathcal{L}_n'}^2) (\trace{(\frac{1}{N}\sum\limits_{n=1}^N\nabla_{\bm{\theta}}\mathbf{f}_n \nabla_{\bm{\theta}}\mathbf{f}_n^T)}) \\
                                            &= 2\mathcal{L}_{\max}(\bm{\theta}) \trace{(H)} \\
                                            &= 2\mathcal{L}_{\max}(\bm{\theta}) \laplace{\mathcal{L}(\bm{\theta})}
\end{aligned}
\end{equation}
\end{proof}
Having built a connection between $\nabla_{\bm{\theta}}\mathcal{L}$ and $H$, we can prove Corollary~\ref{thm:findings:covariance}.

\covariance*
\begin{proof}
    \begin{equation}
    \label{eq:appendix:covariance}
    \begin{aligned}
        \mathbb{E}_\mathcal{D}\|\nabla_{\bm{\theta}}\mathcal{L}\|_2^2 &= \mathbb{E}_\mathcal{D}\nabla_{\bm{\theta}}\mathcal{L} \nabla_{\bm{\theta}}\mathcal{L}^T \\
                                                                      &= \mathbb{E}_\mathcal{D}\trace{(\nabla_{\bm{\theta}}\mathcal{L}^T \nabla_{\bm{\theta}}\mathcal{L})} \\
                                                                      &= \trace{\big(\frac{1}{N}\sum\limits_{n=1}^N{[}{\mathcal{L}_n'}^2~\nabla_{\bm{\theta}}\mathbf{f}_n^T \nabla_{\bm{\theta}}\mathbf{f}_n{]}  \big)} \\
                                                                      &= \trace{(S)}
    \end{aligned}
    \end{equation}
\end{proof}




